\documentclass{article}
\usepackage{authblk}
\usepackage[top=1.in, bottom=1.in, left=1.in, right=1.in]{geometry}
\usepackage{amstext,amsmath,amsthm,amssymb,amsfonts}
\usepackage{bm, yhmath}
\usepackage{bbm, comment}

\usepackage{wrapfig}
\usepackage[utf8]{inputenc}
\usepackage{graphicx} 
\usepackage{subcaption}
\usepackage{xcolor}
\usepackage{lipsum}
\usepackage{enumitem}
\usepackage{url}
\usepackage[ruled]{algorithm2e} 

\usepackage{hyperref}     
\usepackage{cleveref}

\usepackage{tabulary}
\usepackage{booktabs}

\usepackage{natbib}
\setcitestyle{authoryear} 

\usepackage{environ, etoolbox, fancyvrb, ifthen, kvoptions}
\usepackage{apxproof}

\newtheorem*{theorem*}{Theorem}
\newtheorem{theorem}{Theorem}[section]

\newtheorem{corollary}[theorem]{Corollary}

\newtheorem{definition}[theorem]{Definition}
\newtheorem{observation}[theorem]{Observation}

\newtheoremrep{lemma}[theorem]{Lemma}
\newtheoremrep{claim}[theorem]{Claim}
\newtheoremrep{proposition}[theorem]{Proposition}

\usepackage{amsmath,amsfonts,bm}

\def\eqref#1{equation~\ref{#1}}

\def\floor#1{\lfloor #1 \rfloor}
\def\1{\bm{1}}

\def\vs{{\bm{s}}}

\def\vu{{\bm{u}}}

\def\vw{{\bm{w}}}
\def\vx{{\bm{x}}}
\def\vy{{\bm{y}}}

\DeclareMathAlphabet{\mathsfit}{\encodingdefault}{\sfdefault}{m}{sl}
\SetMathAlphabet{\mathsfit}{bold}{\encodingdefault}{\sfdefault}{bx}{n}

\newcommand{\E}{\mathbb{E}}

\newcommand{\R}{\mathbb{R}}

\DeclareMathOperator*{\argmax}{arg\,max}
\DeclareMathOperator*{\argmin}{arg\,min}

\newcommand{\Xdis}[1]{\mathbf{X}_{#1}}

\newcommand{\Thetafin}[0]{\Theta^{\text{fin}}}
\newcommand{\Thetaci}[0]{\Theta^{\text{ci}}}
\newcommand{\Thetabci}[0]{\Theta^{\text{bci}}}

\newcommand{\ThetaN}[0]{\Theta_N}
\newcommand{\ThetaciN}[0]{\Thetaci_N}
\newcommand{\ThetabciN}[0]{\Thetabci_N}

\newcommand{\ThetabciNM}[0]{\Thetabci_{N,M}}

\DeclareMathOperator{\rsupp}{rsupp}

\usepackage{color-edits}
\addauthor[Fang-Yi]{fang}{olive}
\addauthor[Yuqing]{yk}{blue}
\addauthor[Yongkang]{gyk}{red}
\definecolor{hzhorange}{RGB}{219, 48, 122}
\addauthor[Zhihuan]{hzh}{hzhorange}

\newcommand\yk[1]{\ykcomment{#1}}
\newcommand\gyk[1]{\gykcomment{#1}}
\newcommand\hzh[1]{\textcolor[RGB]{202,100,22}{[Zhihuan: #1]}}
\newcommand{\fang}[1]{\fangcomment{#1}}
\renewcommand\yk[1]{}
\renewcommand\gyk[1]{}
\renewcommand\hzh[1]{}
\renewcommand{\fang}[1]{}

\title{Algorithmic Robust Forecast Aggregation}

\author[1]{Yongkang Guo}
\author[2]{Jason D. Hartline}
\author[1]{Zhihuan Huang}
\author[1]{Yuqing Kong}
\author[2]{Anant Shah}
\author[3]{Fang-Yi Yu}

\affil[1]{Peking University}
\affil[2]{Northwestern University}
\affil[3]{George Mason University}

\affil[1]{\texttt{\{yongkang\_guo,zhihuan.huang,yuqing.kong\}@pku.edu.cn}}
\affil[2]{\texttt{hartline@northwestern.edu}}
\affil[2]{\texttt{anantshah2026@u.northwestern.edu}}
\affil[3]{\texttt{fangyiyu@gmu.edu}}
\date{}

\begin{document}

\maketitle
\begin{abstract}
    Forecast aggregation combines the predictions of multiple forecasters to improve accuracy. However, the lack of knowledge about forecasters' information structure hinders optimal aggregation. Given a family of information structures, robust forecast aggregation aims to find the aggregator with minimal worst-case regret compared to the omniscient aggregator. Previous approaches for robust forecast aggregation rely on heuristic observations and parameter tuning. We propose an algorithmic framework for robust forecast aggregation. Our framework provides efficient approximation schemes for general information aggregation with a finite family of possible information structures. In the setting considered by \citet{doi:10.1073/pnas.1813934115} where two agents receive independent signals conditioned on a binary state, our framework also provides efficient approximation schemes by imposing Lipschitz conditions on the aggregator or discrete conditions on agents' reports. Numerical experiments demonstrate the effectiveness of our method by providing a nearly optimal aggregator in the setting considered by \citet{doi:10.1073/pnas.1813934115}.
\end{abstract}

\section{Introduction}
\label{sec:intro}



Forecast aggregation combines the predictions of multiple agents into a more accurate prediction. With forecast aggregation, decision-makers can reduce error, diversify risk and enhance accuracy based on the collective knowledge of agents compared to any single agent, thereby advancing the common good. Forecast aggregation is commonly used in many domains to generate more informed predictions for various variables, such as weather in weather forecasting, the spread of infectious diseases in public health, the outcome of games in sports, fuel prices in energy, and GDP growth in economics.

In practice, one crucial challenge of forecast aggregation is that the aggregator may not have full knowledge of the information structure and the agents. Without this prior knowledge, the aggregator cannot employ Bayes rules to combine the forecasts optimally.\fang{we may introduce omniscient aggregator here} Traditional prior-free aggregation methods, such as simple averaging, are especially bad on some information structures. For example, in weather forecasting, assume the prior probability of raining tomorrow is $30\%$, and there are two agents who will receive a conditionally independent binary signal (Low or High). Agents will report their posterior, which is $10\%$ given the Low signal and $50\%$ given the High signal. When both agents report 50\%, the simple averaging will also output $50\%$. However, the optimal aggregator will output the Bayesian posterior $70\%$ by calculation.

To address this challenge, \citet{doi:10.1073/pnas.1813934115} propose a robust forecast aggregation paradigm, which aims to find an aggregator with the best worst-case performance. In this paradigm, the aggregator\fang{We use an aggregator as a function and also the designer who picks the function.  Do we want to distinguish them?} takes the forecasts of multiple agents as input and outputs an aggregated forecast. The aggregator knows the family of information structures $\Theta$ to which the underlying information structure $\theta$ belongs but does not know the exact $\theta$ that the agents share. Fixing the information structure $\theta$, the regret is defined as the loss of $f$, subtracting the loss of the omniscient aggregator's performance $R(f,\theta)=loss(f)-loss(opt_{\theta})$. The robustness paradigm aims to find the best aggregator $f$ among the family $\mathcal{F}$ with the lowest worst-case regret $\inf_{f\in\mathcal{F}} \sup_{\theta\in\Theta} R(f,\theta)$. The omniscient aggregator $opt_{\theta}$ outputs the Bayesian posterior conditioning on the agents' information and the underlying information structure $\theta$. The worst-case regret refers to the max regret over $\Theta$. In addition to \citet{doi:10.1073/pnas.1813934115}, many recent works including~\citet{neyman2022you, LEVY2022105075,de2021robust} also employ the framework of robust information aggregation framework with different choices of information structure, and formats of reports. 


This paper proposes an algorithmic framework for computing the optimal robust aggregator. Theoretically, when the set of information structures is finite, with mild conditions on the regret function and the family of aggregators, we provide a fully polynomial time approximation scheme (FPTAS) to compute the optimal aggregator. \fang{We may begin by saying that we apply our framework to infinite setting by Arieli et al.}When the set of information structures is continuous\fang{infinite?  Not sure how to define continuous formally}, we put an additional Lipschitz restriction on the aggregators, which ensures that the computed aggregator has a bounded Lipschitz constant. In the robust forecast paradigm, the term ``robustness'' refers to the robustness to the choice of information structure. Lipschitz aggregators add an additional level of robustness: the ability to handle small perturbations in agents' reports. As a direct application, in the setting considered by \citet{doi:10.1073/pnas.1813934115}, our framework provides a fully polynomial time approximation scheme (FPTAS) to compute the optimal Lipschitz aggregator. Numerically, our algorithm obtains an aggregator whose regret estimation almost matches the lower bound in the two conditional independent agents~\citep{doi:10.1073/pnas.1813934115}. We find that when both agents provide near certain forecasts such as $(0.1,0.1)$ or $(0.9,0.9)$, our aggregator with the smallest regret consistently amplifies these forecasts to even more extreme values than previous state-of-the-art aggregators (\Cref{fig:aggregators}).  Our observation aligns with previous empirical work showing that extremizing the average of forecasts often improves the aggregate forecast~\citep{satopaa2014combining,baron2014two,satopaa2015combining}.


Now we present a detailed exposition of our framework and the associated challenges. Our framework views robust forecast aggregation as a zero-sum game between nature and the aggregator. Nature's action space consists of a set of information structures $\Theta$, while the aggregator's action space is represented by a function space $\mathcal{F}$. One notable feature of this problem is that the aggregator's best response can be computed efficiently, even when considering the additional Lipschitz restriction. This enables the application of computational methodology as ellipsoid and online learning techniques to effectively solve the zero-sum game~\citep{HJS-23}.


When the set of information structures $\Theta$ is finite, with mild assumption, our algorithmic approach, based on the online learning framework, allows for the development of an FPTAS for the robust aggregation problem. This result has wide-ranging implications, encompassing various information aggregation scenarios without imposing any restrictions on the format of the agents' reports. Consequently, our algorithmic approach enables the collection of more intricate reports and offers flexibility in modeling agents. For instance, agents can provide not only their forecasts but also higher-order information, such as their expectations for other agents' forecasts~\citep{prelec2017solution,palley2019extracting}. 

Dealing with a continuous set of information structures $\Theta$ is more challenging. To address the complexities of continuous information structures, we employ approaches that involve dimension reduction, discretization, coupling analysis, and a smoothing step for the omniscient aggregator. By tackling these challenges, we can effectively handle the complexities of continuous information structures in the setting considered by \citet{doi:10.1073/pnas.1813934115}. 


In addition, the choice of paradigm $R(f,\theta)$ plays a key role in robust aggregation, as well as other robust optimization problems~\citep{gabrel2014recent, hartline2020benchmark}. Different robustness paradigms can yield distinct results and interpretations. Prior work has often focused on specific robustness paradigms without thoroughly justifying their choices over others. In our work, we mainly consider the additive regret paradigm $loss(f)-loss(opt_\theta)$.  We also empirically compare different paradigms (ratio $loss(f)/loss(opt_\theta)$ and absolute $loss(f)$)and offer a better understanding of the implications and trade-offs associated with each paradigm (\Cref{fig:optimization_goal}).

In summary, our work presents a structured framework and a systematic algorithmic approach for robust aggregation. This enables the automatic design of aggregators. Additionally, our framework offers assistance in selecting suitable robustness paradigms. By empowering decision-makers with better aggregators, our approach enables them to make more informed choices. Our ultimate goal is to advance the field of robust information aggregation, benefit decision-makers, and promote enhanced decision-making practices.

The following sections introduce the problem, settings, results, and our approaches more formally.  

\subsection{Problem, Settings, and Results}\fang{I wonder how many details should be in this section or in problem statement (Section 3).  Right now there are some repetition.}
We state the problem setting and give an overview of the results. There is a state of the world $w \in \Omega$. An agent $i$ receives a private signal $s_{i}$ from a space of signals $\mathcal{S}_{i}$. Let $\mathcal{S}=\mathcal{S}_{1} \times \mathcal{S}_{2} \times \dots \times \mathcal{S}_{n}$ be the space of joint signals. An information structure $\theta$ is a joint distribution over the space of states and signals $(\Omega,\mathcal{S})$. Every agent reports a forecast.\fang{We use forecast and prediction interchangeably} The forecast of agent $i$ depends on their private signal $s_{i}$ and the information structure $\theta$. The space of forecasts of agent $i$ is denoted by $X_{i}$. An aggregator is a function $f: X_{1} \times X_{2} \times \dots \times X_{n} \to Y$, which maps the joint forecasts of agents to the space of aggregations $Y$. The loss that an aggregator suffers depends on the loss function. A loss function $\ell : Y \times \Omega \to \mathbb{R}^{+}$ captures the loss suffered on aggregation $y$ for the state of the world $w$. The loss suffered by an aggregator for an information structure, denoted by $loss(f,\theta)$ is the expected loss suffered by the aggregator over information structure $\theta$ for some forecasting rules used by the agents. Formally, given a loss function $\ell$, a family of aggregators $\mathcal{F}$ and a family of information structures $\Theta$, we aim to solve 
\[\inf_{f\in\mathcal{F}}\sup_{\theta\in\Theta}R(f,\theta)=\inf_{f\in\mathcal{F}}\sup_{\theta\in\Theta} \left(loss(f,\theta)-loss(opt_{\theta},\theta)\right)\] where $opt_{\theta}$ is an omniscient aggregator who knows $\theta$ perfectly. We use the loss of $opt_{\theta}$ as a benchmark, and we set the \textit{loss} as the quadratic loss. In this work, we focus on additive regret, i.e., the difference between the loss suffered by the aggregator and the loss suffered by the omniscient aggregator. Other robustness paradigms would include the ratio paradigm and the absolute paradigm. The ratio paradigm corresponds to the ratio of the loss of the aggregator to the loss of the omniscient aggregator, while the absolute paradigm corresponds solely to the loss of the aggregator. We provide evidence as to why the additive paradigm is most suitable to the problem we study.

\paragraph{Settings} 
We introduce the following settings that consider multiple pairs of nature's action set $\Theta$ and aggregator's action set $\mathcal{F}$. 
\begin{itemize}
\item \textbf{Finite} The set of information structures $\Theta$ is finite. 
\item \textbf{Continuous} The set of information structures is continuous, with mild restrictions put on either $\Theta$ or $\mathcal{F}$.
\begin{itemize}
\item \textbf{Discrete Reports} The agents' reports are discrete such as $10\%,20\%$. Other underlying parameters are continuous. Thus, this does not imply that the set of information structures is discrete. 
\item \textbf{Lipschitz Aggregators} The aggregators have a bounded Lipschitz constant. 
\end{itemize}
\end{itemize}


\paragraph{Results} Theoretically, when the set of information structures $\Theta$ is finite and each $\theta\in\Theta$ has constant size support, the set of aggregators $\mathcal{F}$ is convex and compact with a polynomial time separation oracle, and the loss function is convex continuous and bounded, we provide an FPTAS (\cref{thm:finite} in \cref{sec:finite}). 

For the continuous setting, we focus on the model considered by \cite{doi:10.1073/pnas.1813934115}. Here there are two agents whose private signals are independent conditioning on a binary state. They are asked to report their forecasts for the state. We provide an FPTAS for both the discrete reports (\cref{thm:discrete}) and Lipschitz aggregators (\cref{thm:lipschitz}) settings. 

Our algorithmic framework obtains an aggregator with regret 0.0226 that almost\footnote{The convergence time of our FPTAS depends on the discretization parameters and the Lipschitz constant. To have a reasonable convergence time, we pick relatively small discretization parameters. This is why we only obtain a near-tight aggregator. } matches the existing lower bound $\frac18(5\sqrt{5}-11)\approx 0.0225$, while previous state-of-the-art in \cite{doi:10.1073/pnas.1813934115} has regret 0.0250. 

Our aggregator differs from the previous aggregators mainly when both agents' reports are close to $0$ or $1$. For instance, when both agents report $(0.1, 0.1)$ or $(0.9, 0.9)$, as shown in Figure \ref{fig:aggregators}, our aggregator, which has the best performance, outputs even more extreme forecasts than previous aggregators. 


\begin{figure}[!ht]
    \centering
    \begin{subfigure}[b]{0.225\textwidth}
        \centering
        \includegraphics[height=.96\textwidth]{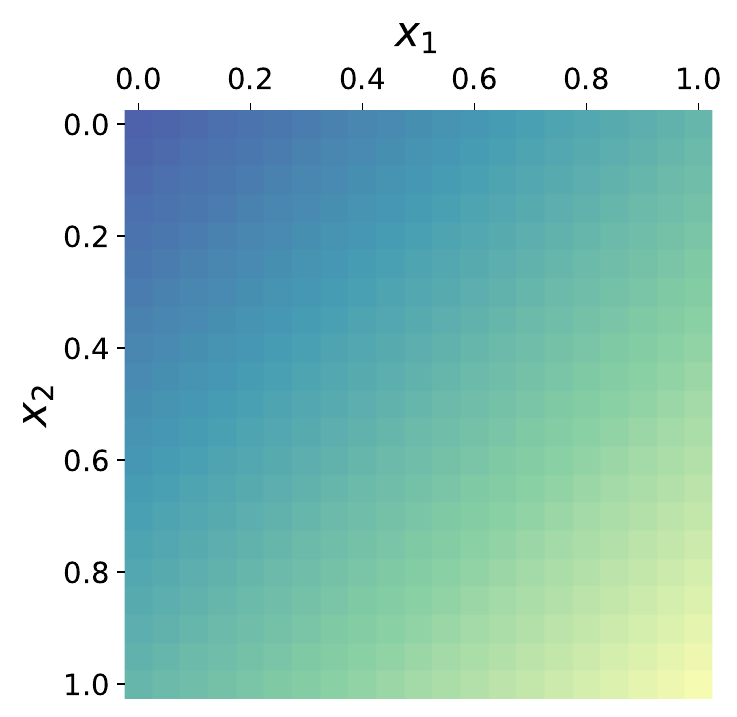}
        \caption{\textbf{Simple averaging}}
    \end{subfigure}
    \begin{subfigure}[b]{0.225\textwidth}
        \centering
        \includegraphics[height=.96\textwidth]{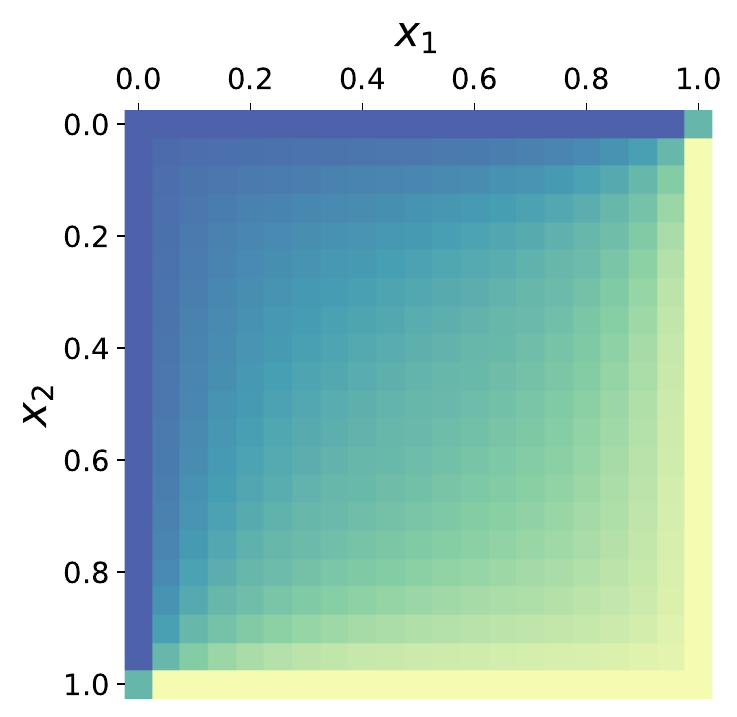}
        \caption{\textbf{State-of-the-art}}
    \end{subfigure}
    \begin{subfigure}[b]{0.27\textwidth}
        \centering
        \includegraphics[height=.8\textwidth]{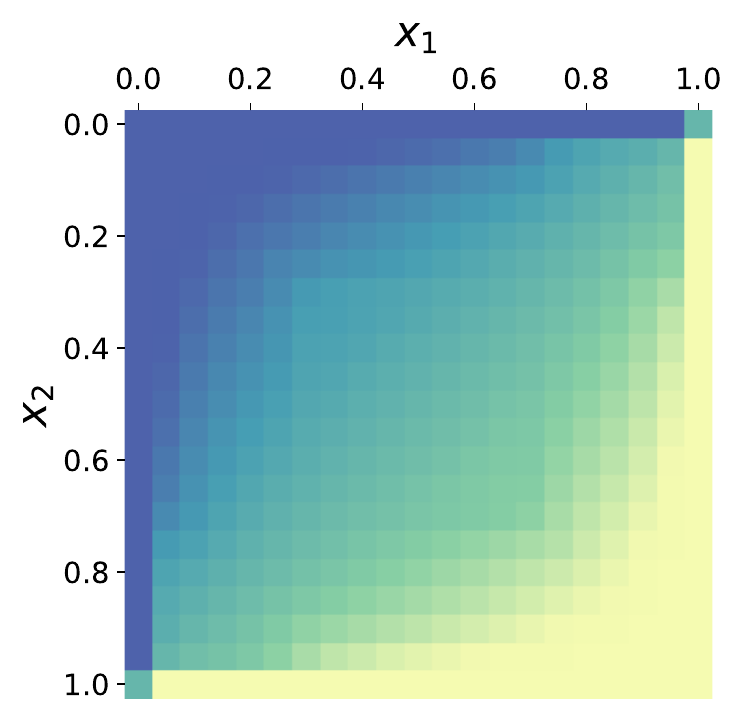}
        \includegraphics[height=.72\textwidth]{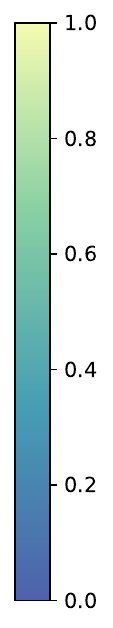}
        \caption{\textbf{Our aggregator}}
    \end{subfigure}
    
    \caption{\textbf{Heatmaps of different aggregators $f(x_1,x_2)$.} The horizontal axis represents the first agent's report $x_1$, and the vertical axis represents the second agent's report $x_2$. Darker to lighter shades represent the range of $f(x_1,x_2)$ from 0 to 1.}
    \label{fig:aggregators}
\end{figure}

Additionally, we experiment and compare our algorithmic framework on different robustness paradigms, additive, ratio, and absolute (see formulas in \Cref{sec:paradigm}). In the context considered in this work, the findings indicate that the additive robustness paradigm outperforms the other two (\Cref{fig:optimization_goal}). The additive robustness paradigm demonstrates superior overall performance in quadratic loss on a wide range of information structures. This is because the additive robustness paradigm pays attention to a broader range of information structures, while the absolute and ratio robustness paradigms exhibit limited attention to only a few information structures. 

\begin{figure}[!ht]
    \centering
    \begin{subfigure}[b]{0.32\textwidth}
        \includegraphics[width=\textwidth]{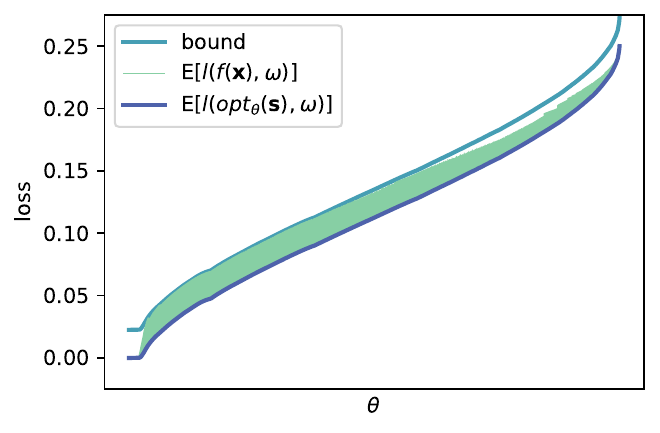}
        \caption{\textbf{Additive}}
        \label{fig:image2}
    \end{subfigure}
    \begin{subfigure}[b]{0.32\textwidth}
        \includegraphics[width=\textwidth]{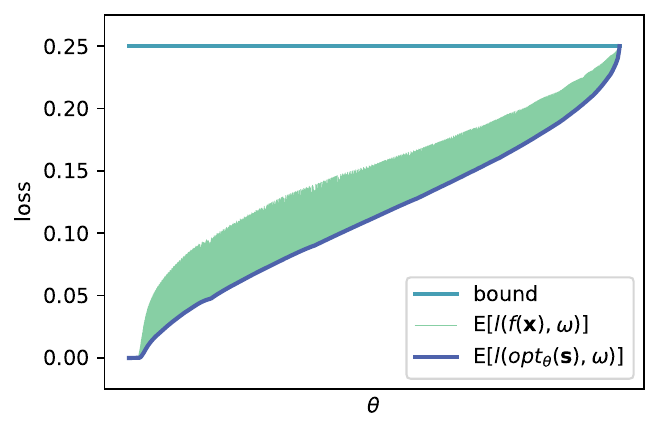}
        \caption{\textbf{Absolute}}
        \label{fig:image1}
    \end{subfigure}
    \begin{subfigure}[b]{0.32\textwidth}
        \includegraphics[width=\textwidth]{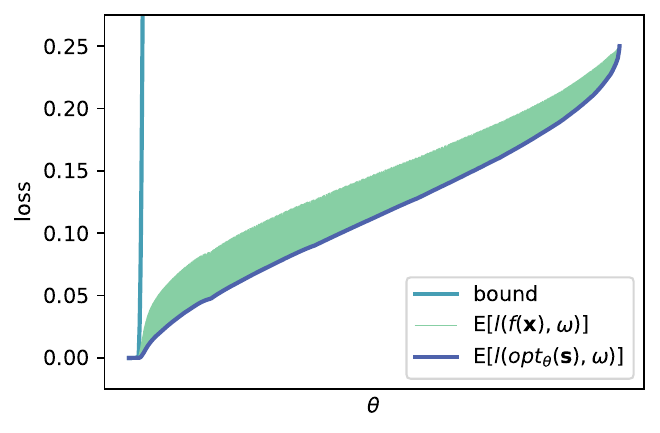}
        \caption{\textbf{Ratio}}
        \label{fig:image2}
    \end{subfigure}
    
    \caption{\textbf{Losses of the optimal aggregator under additive, ratio, and absolute robustness paradigms.} We pick a finite collection of the information structures, and the horizontal axis represents the information structures sorted by their losses under the omniscient aggregator. The vertical axis represents the loss. The bottom curve (navy blue) represents the optimal loss (lower bound), i.e., the loss of the omniscient aggregator. The middle area (green), is the range of loss of the optimal aggregator $f$ obtained by our algorithm for each paradigm. The top curve (cyan) represents the highest loss that can be afforded for each information structure without increasing the maximum regret.\fang{weird}\gyk{It is hard to explain this curve. Maybe put the formula here.}\hzh{The top curve represents the maximum loss that will not increase the regret for each information structure.} The worst case occurs when the top curve touches the middle curve.}
    \label{fig:optimization_goal}
\end{figure}

\subsection{Framework Overview}

At a high level, we view the robust aggregation problem as a zero-sum game between nature who picks the information structure $\theta$ and the aggregator who picks the aggregation function $f$. The challenge is that both players' action spaces are complex and high dimensional. However, a key insight is that despite this complexity, the aggregator's best response can be efficiently computed given a mixed strategy of nature, which is a distribution over information structures. Leveraging this observation, we introduce an algorithmic framework that addresses the challenges of robust aggregation. 

\begin{description}
\item [Finite] When the set of information structures $\Thetafin$ is finite but the aggregator's action set can be continuous, we employ online learning to solve the zero-sum game outlined as the following. 
\begin{description}
\item [Zero-Sum Game] To solve the robust aggregation problem as a zero-sum game between nature picking $\theta$, and an aggregator picking $f$, we prove the existence of a mixed Nash equilibrium in this game. Then we demonstrate that finding an approximate equilibrium allows us to obtain a near-optimal aggregator.

\item [Online Learning] Each information structure $\theta\in\Thetafin$ is an action of nature. For each round $t$, nature will select a probability distribution $\vw_{\Thetafin}^t$, the aggregator will select the best response $f^t$ and each $\theta$ will gain reward $u^t_\theta=R(f^t,\theta)$ under $f^t$. Nature aims to minimize the sum of $T$ rounds online learning regret $\max_{\theta\in \Thetafin}\sum_{t=1}^T u^t_\theta-\sum_{t=1}^T \vw_{\Thetafin}^t\cdot \vu^t
    $. We show that the aggregator's best response is efficient to compute through convex optimization with mild assumptions. Combining the above,  \cref{alg:1} can obtain arbitrarily small online learning regret in polynomial time by properly selecting the parameters.
\end{description}
\item [Continuous] Here both nature and the aggregator's action sets are continuous. In such cases, we will need to reduce dimensions, discretize action sets, and prove insensitivity properties. 
\begin{description}
\item [Dimension Reduction] The dimension is the minimal number of parameters to describe the continuous information structures $\Theta$. In general, because the size of signal space ${\mathcal{S}}$ can be infinite, $\Theta$ can be infinite-dimensional, and hence the weight vector $\vw_{\Theta}^t$ is also infinite-dimensional.  To run the online learning algorithm efficiently, we need to reduce the space of $\Theta$ to a lower dimension space  $\Theta^{sub}\subset\Theta$ so that minimizing the expected loss in the set $\Theta^{sub}$ is equivalent to minimizing the expected loss in the original set $\Theta$. By making this reduction, we can simplify the problem without loss of generality and reduce the time complexity of our algorithm. 
\item [Discretization] After the dimension reduction, the set $\Theta^{sub}$ can be represented by a finite number of parameters. However, the number of information structures in $\Theta^{sub}$ is still infinite because the parameters are continuous. To reduce its complexity, we will discretize $\Theta^{sub}$ and obtain a sketch set $\Theta^{dis}$ of $\Theta^{sub}$. To make sure the aggregator learned in the discrete information structures $\Theta^{dis}$ also performs well in the continuous case $\Theta$, the following properties are desired to control the generalization error. 
\begin{description}
\item [$(\epsilon,d)$-covering] For any information structure $\theta$ in $\Theta^{sub}$, we can find another information structure $\theta'$ in $\Theta^{dis}$ such that $\theta$ and $\theta'$ are close enough, i.e., $d(\theta,\theta')\leq \epsilon$ under some metric $d$.  
\item [Insensitivity] for any aggregator $f\in\mathcal{F}$, any close pair of $\theta$ and $\theta'$ under metric $d$, $R(f,\theta)\approx R(f,\theta')$.

\end{description}
With these properties, we can show that for any aggregator $f$, the maximal regret in space $\Theta^{dis}$ is close to space $\Theta$, i.e., $\max_{\theta\in\Theta^{dis}} R(f,\theta)\approx \max_{\theta\in\Theta }R(f,\theta)$. Thus, we can apply the online learning algorithm in the finite set $\Theta^{dis}$ to obtain a near-optimal aggregator $f$ for $\Theta$.
\end{description}
\end{description}

\paragraph{Implementation Details} To implement the above framework, there are still some obstacles that need to be overcome regarding choosing the proper distribution metric in the space of the information structures, upper bounding the metric by coupling, and delicately dealing with the sensitivity of the additive regret. 
\begin{description}
\item[\textbf{Metric Choice}] Since the family of information structures is a set of probability distributions, we should pick a distribution metric. Total variation distance (TVD) is a natural choice in the discrete reports setting. However, TVD is overly sensitive in the Lipschitz setting. Instead, we employ a weaker distance, the earth's mover distance (EMD) in the Lipschitz setting. 
\item[\textbf{Coupling}] Recall that we need to show for any information structure $\theta$ in $\Theta^{sub}$, we can find another information structure $\theta'$ in $\Theta^{dis}$ such that $\theta$ and $\theta'$ are close enough. In the Lipschitz setting, we will adopt the EMD as the metric which is less sensitive. Using a duality form of EMD, we will construct a proper coupling between $\theta$ and $\theta'$ to upper bound the EMD between $\theta$ and $\theta'$. 
\item [\textbf{Sensitivity}] Neither TVD nor EMD covering guarantee generalization straightforwardly, because the aggregators can be very sensitive as the information structure changes. To bound the sensitivity of the additive regret, we will analyze the $loss(f,\theta)$ and $loss(opt_{\theta},\theta)$ separately. We can bound $loss(f,\theta)$'s Lipschitz constant easily in both the discrete reports or Lipschitz aggregators setting. That is, $loss(f,\theta)\approx loss(f,\theta')$ when $\theta$ and $\theta'$ are close. However, the omniscient aggregator $opt_{\theta}$ is highly sensitive even in the simple conditionally independent model. We will prove that the forecasts are unlikely to happen in those sensitive areas so that $loss(opt_{\theta},\theta)$, which refers to the expected square loss of the omniscient aggregator, is also insensitive regarding $\theta$, i.e. $loss(opt_{\theta},\theta)\approx loss(opt_{\theta'},\theta')$ when $\theta$ and $\theta'$ are close.
\end{description}

\section{Related Work}

\paragraph{Robust Information Aggregation} In this work, we consider the problem of robust forecast aggregation. The question of finding a forecast aggregator which performs well in the worst-case for any information structure was first studied by  \citet{doi:10.1073/pnas.1813934115}. In their work, they propose low-regret aggregators for the two agents, binary state space setting when the agents receive conditionally independent signals. They provide upper and lower bounds on the additive regret, showing that the average prior scheme\footnote{The formula of the average prior aggregator is $\frac{x_1x_2(1-\frac{x_1+x_2}2)}{x_1x_2(1-\frac{x_1+x_2}2)+(1-x_1)(1-x_2)\frac{x_1+x_2}2}$.} provides a worst-case regret of $0.0260$ while a lower bound on the regret is $\frac18(5\sqrt{5}-11)\approx 0.0225$. They improve the regret to $0.0250$ by tuning the parameters in the average prior scheme. We consider their setting of conditionally independent signals and show that under some restrictions on the class of forecast aggregators and the class of information structures, a near-optimal robust aggregator can be computed efficiently. Our algorithm obtains an aggregator which has regret $0.0226$, which almost matches the lower bound. In further work on robust forecast aggregation, \citet{neyman2022you} consider the ratio robustness paradigm for a class of information structures which they term as \textit{projective substitutes}. We show empirically that the ratio robustness paradigm is not a good choice for forecast aggregation. \citet{LEVY2022105075} consider a setting where the aggregator knows the marginal distribution of the forecasters and wants to design an aggregator which works well for any unknown correlation structure among the forecasters. \citet{de2021robust} consider a similar setting to \citet{LEVY2022105075}, but with an ex-ante perspective, in which the aggregator knows the marginal distribution of signals of each expert in isolation but is unaware of any correlation between them. When the set of information structures is finite, we provide an efficient algorithm for a general robust aggregation problem. When the set of information structures is continuous, we mainly focus on the two-agent, conditionally independent information structure setting considered in \citet{doi:10.1073/pnas.1813934115}.


\paragraph{Prior-Independent Optimal Design} 
The robust information aggregation problem we consider falls under the class of prior-independent optimal design problems. The challenge in a prior-independent design problem is to come up with a design that works well for any distributional input while competing against the optimal design which knows the distribution. \citet{dhangwatnotai2010revenue} show that in a single-item two-agent setting when the values are drawn independently and identically from regular distributions, the second-price auction is a 2-approximation in terms of the prior-independent ratio. For this setting, \citet{fu2015randomization} showed that the ratio of $2$ is not tight, while \citet{allouah2020prior} showed that the ratio must lie in the range $[1.80,1.95]$. \citet{hartline2020benchmark} resolved this question by showing a ratio of 1.91, while also proving that it is tight. More recently, \citet{anunrojwong2022} consider the prior-independent mechanism design problem for a single item multi-agent setting, for a variety of valuation distributions including i.i.d distributions, mixtures of i.i.d distributions, exchangeable distributions, and all joint distributions. They show that the prior-independent optimal mechanism across all settings is a second-price auction with a random reserve under the additive regret robustness paradigm. Our work is different from prior-independent mechanism design literature in the sense that we consider a prior-independent information aggregation problem. In prior-independent mechanism design, the space of distributions corresponds to the values of the agents in the mechanism, while in prior-independent information aggregation, they would correspond to all possible information structures, which potentially have a high dimension and complex formats. 

\paragraph{Online Learning \& Zero-Sum Games}
In this work, we consider the reduction from computing a near-optimal
prior-independent algorithm to computing an approximate Nash
equilibrium in a two-player zero-sum game. \citet{Dantzig-63} showed
that a Nash equilibrium in a finite bi-matrix zero-sum game can be computed
by a linear program. Following the result of \citet{Khachiyan-79},
finding an equilibrium in zero-sum games is tractable in the size of
the payoff matrix. There is a lot of work in the literature which looks at
computing an approximate Nash equilibrium in a finite game using online-learning
techniques. If both the players are employing no-regret learning
strategies (e.g., \citealp{FS-97}), then their average play converges
to a Nash equilibrium at rate $O(\frac{1}{\sqrt{T}})$ with a logarithmic dependence on the number of actions of the agents. \citet{DDK-11}
give a no-regret learning algorithm which when used by both players
simultaneously, their average play converges to an approximate Nash
equilibrium at rate $O(\frac{\log T}{T})$ and in the adversarial case
achieves the $O(\frac{1}{\sqrt{T}})$ rate. Recent works have looked at
learning strategies so as to obtain the last iterate convergence to the
minimax optimization problem when both players simultaneously employ a
specific no-regret learning strategy (e.g., \citealp{DP-19}). In our problem, the space of aggregators being optimized over is uncountably large and the space of information structures is uncountably large as well, and thus the above techniques cannot be applied directly to the original strategy spaces.

We utilize the framework proposed by \citet{HJS-23} to compute a near-optimal robust forecast aggregator. In their work, \citet{HJS-23} identify sufficient conditions under which there exists an algorithm to efficiently compute a near-optimal prior-independent online algorithm. At a high level, the three sufficient conditions are an \textit{efficient best response} property, an \textit{efficient utility computation} property, and a \textit{small-cover} property. While they apply the framework to online algorithms by showing an application to the ski-rental problem, we show that the properties are also satisfied in the setting of robust forecast aggregation. As mentioned in Section 1, the robust forecast aggregation problem can be viewed as an infinite zero-sum game. Motivated by computing equilibrium in infinite games from applications to machine learning, \citet{AADDF-23} computes an equilibrium in an infinite two-player zero-sum game with general payoffs assuming that each player has access to a best-response oracle. Their bounds however depend exponentially on the approximation factor desired. Thus, the framework of \citet{HJS-23} is best suited to our problem. While they apply this framework in the context of online algorithms, we show that this framework also can be applied in the context of information aggregation. The main challenge to overcome is to show an analogous \textit{small-cover} property over the space of information structures.
\section{Problem Statement}
In this section, we state our general model and main problem. Then we will discuss different settings derived from the general model.
\paragraph{Information Structure $\theta$, Aggregator $f$} We consider the following information aggregation problems in this paper. Suppose the world has a state $\omega\in\Omega$. There are $n$ agents and each agent $i$ receives a private signal $s_i$ in a signal space $\mathcal{S}_i$. Let $\mathcal{S}=\mathcal{S}_1\times \mathcal{S}_2\times...\times \mathcal{S}_n$. An information structure $\theta\in \Delta_{\Omega\times \mathcal{S}}$ is a joint distribution over the state and signals.
We denote a family of information structures as $\Theta\subseteq \Delta_{\Omega\times \mathcal{S}}$. The aggregator will ask each agent $i$ to report $x_i(s_i, \theta)$ which depends on her private signal and the information structure $\theta$.\footnote{For example, when $x_i(s_i,\theta)$ is the posterior report, $x_i(s_i,\theta)=\Pr[\omega=1|s_i]$ when $\Omega = \{0,1\}$} Let $\vx=(x_1,\cdots,x_n)\in X$, and $\rsupp(\theta)=\{\vx:\Pr_{\theta}[\vx]>0\}$ is the set of $\vx$ where the marginal probability of $\vx$ under $\theta$ is positive. An aggregator is a deterministic function $f\in \mathcal{F}$ which maps $\vx$ to space $Y$. We define a loss function $\ell(y,\omega):Y\times\Omega\mapsto \mathbb{R}^+$, indicating the loss suffered by the aggregator when the real state is $\omega$ and the aggregation result is $y$. To evaluate our aggregator, we define a generally unachievable benchmark, \emph{omniscient aggregator}, who knows the information structure $\theta$ and all signals of agents $\vs$ and reports
\[opt_\theta(\vs)=\argmin_g\E_\theta[\ell(g(\vs),\omega)]\]
to minimize the expected loss.

\paragraph{Min-max Problem} Given a family of information structures $\Theta$, a set of aggregators $\mathcal{F}$, we aim to minimize the expected loss in the worst information structure. Thus, we want to find the optimal function $f^*$ to solve the following min-max problem:
$$\inf_{f\in \mathcal{F}}\sup_{\theta\in\Theta}\E_\theta[\ell(f(\vx),\omega)]-\E_\theta[\ell(opt_\theta(\vs),\omega)].$$

Notice that the distribution of $\vx$ is determined by $\theta$ so all the randomness comes from $\theta$. $R(f,\theta)=\E_\theta[\ell(f(\vx),\omega)]-\E_\theta[\ell(opt_\theta(\vs),\omega)]$ is aggregator $f$'s regret on information structure $\theta$. $R(f,\Theta)=\sup_{\theta\in\Theta} R(f,\theta)$ is aggregator $f$'s maximal regret among the family of information structures $\Theta$. 

\paragraph{Forecast Aggregation} We mainly focus on forecast aggregation. Here each $x_i\in \Delta_{\Omega}$ is agent $i$'s Bayesian posterior for state $\omega$, i.e., $x_i(s_i, \theta)_{\omega}=\Pr_{\theta}[\omega|s_i]$. The aggregator will map all reports to a forecast in $\Delta_{\Omega}$.  We assume that agents’ reports are truthful.  To incentivize the agents to report truthfully, we may introduce additional reward to the agents by a proper scoring rule \citep{winkler1996scoring,gneiting2007strictly}, such as square loss or cross-entropy loss.


\subsection{Settings of Information structures $\Theta$ and Aggregators $\mathcal{F}$}

We will show that solving the min-max problem can be seen as solving the zero-sum game between nature, who selects the information structure, and the aggregator. However, in the basic setting where there is no restriction on $\mathcal{F}$ nor $\Theta$, it is difficult to solve the game as the action set is continuous and has a high dimension. Thus we introduce different variants of the setting which puts finite or Lipschiz restrictions on the set of information structures, or the set of the aggregators.

\paragraph{Finite Setting: Finite Set of Information Structures}
As a warm-up, we first consider the min-max problem with a finite number of information structures with a finite number of signals. That is, we aim to solve the following min-max problem:

$$\inf_{f\in \mathcal{F}}\sup_{\theta\in\Thetafin}R(f,\theta)$$
where $|\Thetafin|$ is finite, and for any $\theta\in \Thetafin$, $\left|\rsupp(\theta)\right|$ is bounded by a constant. Recall that $\rsupp(\theta)$ is the support of report space. In the finite setting, we will show that we can directly run the online learning algorithm to solve the zero-sum game because nature's action space is finite and the aggregator's best response can be computed efficiently. Later we will extend our method to a continuous set of special information structures. The analysis of the continuous set is based on the results in the finite setting but requires a more delicate sensitivity analysis of the objective function $R$.

\paragraph{Discrete Setting: Discrete Reports}
In practice, we usually use sliders to elicit forecasters' reports. Thus, a slightly relaxed setting is the setting where the reports $x_i$ are restricted to be in a finite discrete space, such as the space of discrete percentages. Note that even if the reports are discrete, the set of information structures can still be continuous because other parameters that are unrelated to the reports are continuous. The analysis of this setting is also an intermediate step in the analysis of the generally continuous information structures.

Formally, given a resolution scale $N\in \mathbb{N}$, we define $[1/N]:=\left\{0,\frac1N, \frac2N, \cdots, 1\right\}$ and $[n] = \{1,2,\dots, n\}$. We will focus on the following set of information structures. 

$$\ThetaN =\left\{\theta\in\Delta_{\Omega\times S}: \forall i\in [n], s_i\in \mathcal{S}_i,x_i(s_i, \theta)\in [1/N]\right\}.$$


We aim to solve the following min-max problem:
$$\inf_{f\in \mathcal{F}}\sup_{\theta\in\ThetaN}R(f,\theta).$$

\paragraph{Lipschitz Setting: Lipschitz Aggregators}  We consider a special set of aggregators--the Lipschitz aggregator. The ``robustness'' in the robust forecast paradigm refers to the ability to handle the lack of knowledge regarding the information structure. Lipschitz aggregators provide an additional level of robustness by handling small perturbations in agents' reports. Additionally, this restriction on the set of aggregators will allow nature to have a generally continuous action set. Formally, we define the $L$-Lipschitz function: 

\begin{definition}[$L$-Lipschitz, \citealp{sohrab2003basic}]
  A function $f:X\mapsto \R$ is \textbf{\textit{$L$-Lipschitz}} if 
$$|f(\vx)-f(\vx')|\le L\|\vx-\vx'\|_1$$
for all $\vx, \vx'\in X$. Additionally, $f$ has \textbf{\textit{Lipschitz constant $L$}} if $f$ is $L$-Lipschitz but not $L'$-Lipschitz for all $L'<L$.  We will use $\|f\|_{Lip}$ to denote $f$'s Lipschitz constant, and $\|f\|_\infty = \sup_{\vx}|f(\vx)|$.  
\end{definition}

Let $\mathcal{F}_L=\{f:\|f\|_{Lip}\le L\}$. We aim to solve the following min-max problem:
$$\inf_{f\in \mathcal{F}_L}\sup_{\theta\in\Theta}R(f,\theta).$$

\subsection{Binary State and Two Conditionally Independent Agents}\label{sec:pre_bci}

In the robust information aggregation problem, we consider the setting of two agents with conditionally independent signals. We first reason as to why this is the right model to consider when considering the robust aggregation problem. We then move on to the loss function we consider, stating the functional form of the omniscient aggregator. We can directly solve the finite information structures setting for general models. However, to efficiently solve the discrete reports setting and Lipschitz aggregators setting, we still need to reduce the dimension and find an appropriate discretization. We formalize the discretization on the information structures that is useful to our setting.

\citet{doi:10.1073/pnas.1813934115} propose a model which assumes a binary state $\omega=\{0,1\}$ and two agents with conditionally independent signals. They show that the min-max problem is trivial with a set of general information structures because nature can always find an information structure where no aggregator can be better than the naive one which always picks one agent's report. Within the conditionally independent structure, as the number of agents goes to infinite, the min-max problem becomes trivial again because the benchmark omniscient aggregator will be ``too good''. Therefore, like \citet{doi:10.1073/pnas.1813934115}, we focus on the setting of two conditionally independent agents. 

Formally, we consider all possible conditionally independent information structures $\Thetaci$ so that for each $\theta\in \Thetaci$, for all $s_1\in \mathcal{S}_1,s_2\in \mathcal{S}_2,\omega\in\Omega=\{0,1\}$, $\Pr_\theta[s_1,s_2|\omega]=\Pr_\theta[s_1|\omega]\Pr_\theta[s_2|\omega]$. Agents share $\theta$, as well as the common prior for the state $\mu=\Pr_\theta[\omega=1]$. Each agent $i$ reports her posterior $x_i(s_i,\theta)=\Pr_\theta[\omega=1|s_i]$.

In this binary state setting, both the reports $x_1,x_2$ and the aggregator's output $y$ are in $[0,1]$. We use the square loss function, i.e. $\ell(y,\omega)=(y-\omega)^2$.\fang{Are $loss$ general loss and $\ell$ for square loss?} So we only need to consider the set of aggregators $\mathcal{F}=\{f:f(x_1,x_2)\in[0,1],\forall (x_1,x_2)\in [0,1]^2\}$ that maps $[0,1]^2$ to $[0,1]$. The omniscient aggregator's posterior is $opt_\theta(\vs)=\Pr_\theta[\omega=1|\vs]$. 

This special model has several properties that simplify the robust forecast aggregation problem. 

\paragraph{Properties} First, with the conditionally independent structure, the omniscient aggregator's Bayesian posterior only depends on the reports $\vx$ and prior $\mu=\Pr_\theta[\omega=1]$. Thus, we use $g_{\mu}(\vx)$ to denote the omniscient aggregator's Bayesian posterior.
\begin{lemma}[Bayesian Posterior, \citealp{bordley1982multiplicative}]\label{lem:prediction2posterior}
Given an information structure with prior $\mu$, the omniscient aggregator's Bayesian posterior given forecasts $\vx$ is 
\begin{equation}\label{eq:omniscient}
    g_{\mu}(\vx)=\frac{(1-\mu)x_1 x_2}{(1-\mu)x_1 x_2+\mu(1-x_1)(1-x_2)}.
\end{equation}
 Let $\bar{\mu} = \mu$, $\bar{x}_1 = 1-x_1$, and $\bar{x}_2 = 1-x_2$. The above formula can be simplified as $g_\mu(\vx) = \frac{\bar{\mu}x_1x_2}{\bar{\mu}x_1x_2+\mu\bar{x}_1\bar{x}_2}$. 
\end{lemma}

We use $\Xdis{\theta} \in\Delta_{[0,1]^2}$ to denote the marginal distribution over the reports. The following lemma states its formula. $\Pr_\theta[x_1]$ is the probability that agent 1 report $x_1$ and $\Pr_\theta[x_2]$ is the probability that agent 2 report $x_2$.

\fang{In my opinion, these three seems too trivial to be lemmas.}
\begin{lemma}[Reports Distribution, \citealp{doi:10.1073/pnas.1813934115}]
The probability of reports $(x_1,x_2)$ is
$$\Pr_\theta[x_1,x_2]=\Pr_\theta[x_1]\Pr_\theta[x_2]\left(\frac{(1-x_1)(1-x_2)}{1-\mu}+\frac{x_1x_2}{\mu}\right).$$
\end{lemma}

The expected square distance between the aggregator's output and the omniscient aggregator's output. 

\begin{lemma}[Additive Regret, \citealp{doi:10.1073/pnas.1813934115}] For any aggregator $f\in\mathcal{F}$, 
$$R(f,\theta)=\E_\theta[(f(\vx)-g_{\mu}(\vx))^2].$$
\end{lemma}

\paragraph{Dimension Reduction}


\citet{doi:10.1073/pnas.1813934115} prove that it is sufficient, with conditionally independent information structures, to consider binary signals. Formally, let $\Thetabci$ denote the set of conditionally independent information structures with binary signals where for all $i$, $|\mathcal{S}_i|=2$.  For any aggregator $f$, to maximize the regret, it is sufficient for nature to pick from $\Thetabci$, i.e., $R(f, \Thetaci)=R(f,\Thetabci)$. We extend the dimension reduction results to the discrete reports version of $\Thetaci$ (\Cref{lem:reduction}).

Formally, we write the discrete report versions of $\Thetaci$ and $\Thetabci$ as follows.
\begin{equation}\label{eq:discrete_infostruct}    
\begin{aligned}
    &\ThetaciN = \{\theta\in\Thetaci: \forall i, \forall s_i\in \mathcal{S}_i,x_i(s_i, \theta)\in [1/N]\},\text{ and }\\
&\ThetabciN = \{\theta\in\Thetaci: \forall i, |\mathcal{S}_i| = 2, \forall s_i\in \mathcal{S}_i,x_i(s_i, \theta)\in [1/N]\}
\end{aligned}
\end{equation}
where the reports are discrete with a resolution scale $N$. 
Alternatively, $\ThetaciN = \Thetaci\cap \ThetaN$ and $\ThetabciN = \Thetabci\cap \ThetaN$.



\begin{lemmarep}[Dimension Reduction for Discrete Setting]\label{lem:reduction}For any aggregator $f$, $R(f, \ThetaciN)=R(f,\ThetabciN)$.
\end{lemmarep}

To show the dimension reduction, based on \cite{doi:10.1073/pnas.1813934115}'s idea, we write each information structure as a convex combination of ``basic'' information structures which only have binary supports. When the action space is $\Thetaci$, nature's optimization problem has a multi-linear format. Therefore, it is sufficient for nature to pick ``basic'' information structures with binary supports.  This property holds because the discrete version of $\Thetaci$ still has ``basic'' discrete information structures with binary supports. We defer the proof to \Cref{apx:reduction}. 

\begin{toappendix}
\label{apx:reduction}
\end{toappendix}

\begin{appendixproof}[Proof of \Cref{lem:reduction}]
    Suppose $q_i^x=\Pr_\theta[x_i=x]$ for any $x\in [1/N]$. Then $\mathbf{q}_i$ is a non-negative vector with dimension $N+1$. The regret can be represented by
    \begin{align*}
        R(f,\theta)=&\E_\theta[(f(x_1,x_2)-g_\mu(x_1,x_2))^2]\\
        =&\sum_{x_1,x_2}\Pr_\theta[x_1,x_2](f(x_1,x_2)-g_\mu(x_1,x_2))^2\\
        =&\sum_{x_1,x_2}q_1^{x_1}q_2^{x_2}\left(\frac{(1-x_1)(1-x_2)}{1-\mu}+\frac{x_1x_2}{\mu}\right)\left(f(x_1,x_2)-\frac{(1-\mu)x_1 x_2}{(1-\mu)x_1 x_2+\mu(1-x_1)(1-x_2)}\right)^2\\
    \end{align*}
    which means it is a multilinear function of $\mathbf{q_1},\mathbf{q_2}$. And $\mathbf{q_1},\mathbf{q_2}$ should satisfy the following linear constraints:
    $$\sum_{x}q_i^xx=\mu\text{ and }\sum_x q_i^x=1\text{, for all } i=1,2$$
    We first fix $\mathbf{q}_2$, so that we can view $R(f,\ThetaciN)=\sup_{\theta\in\ThetaciN} R(f,\theta)$ as a linear programming problem with variable $\mathbf{q}_1$. By the fundamental theorem of linear programming, we know that there must exist an optimal basic feasible solution for this problem (Luenberger and David G, 1973). Here the constraint matrix is $2\times (N+1)$, so the basic solution has at most $2$ non-zero entries. Similarly, for fixed $\mathbf{q}_1$, we can find an optimal basic feasible solution for $\mathbf{q}_2$ with $\le 2$ non-zero entries. Thus, it is sufficient to consider information structures in $\ThetabciN$.
\end{appendixproof}

\paragraph{Choice of Coordinates} In $\ThetabciN$ or $\Thetabci$, each $\theta$ can be encoded five parameters. For simplicity, we assume $s_i\in\{0,1\},i=1,2$. Here we present two possible parametrizations used in this paper. First, \emph{prediction parametrization} $(\mu, a_0, a_1, b_0, b_1)$ uses the prior for state $\mu=\Pr_{\theta}[\omega=1]$, and the agents' reports given different private signals 
\begin{equation}\label{eq:para1}
    \begin{cases}
a_0=\Pr_{\theta}[\omega=1|s_1=0]\\
a_1=\Pr_{\theta}[\omega=1|s_1=1]\\
b_0=\Pr_{\theta}[\omega=1|s_2=0]\\
b_1=\Pr_{\theta}[\omega=1|s_2=1]
\end{cases}.
\end{equation}

One the other hand, \emph{probability parametrization} $(\mu, p_0, p_1, q_0, q_1)$ uses $\mu=\Pr_{\theta}[\omega=1]$, and the probability of receiving each signal 
\begin{equation}\label{eq:para2}
\begin{cases}
p_0=\Pr_{\theta}[s_1=1|\omega=0]\\
p_1=\Pr_{\theta}[s_1=1|\omega=1]\\
q_0=\Pr_{\theta}[s_2=1|\omega=0]\\
q_1=\Pr_{\theta}[s_2=1|\omega=1]
\end{cases}
.
\end{equation}

We will pick the parametrization properly based on the specific scenario and the analysis being conducted. For example, in the discrete report setting, it is more convenient to pick the first prediction parametrization as it directly contains the reports. When we want to measure the distance between the information structures in the analysis, it is more convenient to use the second probability parametrization. 

We prove an auxiliary result (\Cref{lem:construct_info}) that there is a bijection between these two coordinates. We defer the proof to \Cref{apx:reduction}.
\begin{lemmarep}\label{lem:construct_info}
    For any $a_0, a_1, b_0, b_1$ and $\mu$ in $[0,1]$ with $a_0<\mu<a_1$ and $b_0<\mu<b_1$, there exists a unique conditional independent information structure with binary signals $\theta = (\mu, p_0, p_1, q_0, q_1)\in \Thetabci$ where 
    \begin{equation}\label{eq:construct_info}
        p_1 = \frac{a_0(a_1-\mu)}{\mu (a_1-a_0)}, p_0 = \frac{(1-a_0)(a_1-\mu )}{(1-\mu )(a_1-a_0)}, q_1 = \frac{b_0(b_1-\mu )}{\mu (b_1-b_0)}\text{, and }q_0 = \frac{(1-b_0)(b_1-\mu )}{(1-\mu )(b_1-b_0)}
    \end{equation}
    so that $\rsupp(\theta) = \{(a_0, b_0), (a_1, b_0), (a_0, b_1), (a_1, b_1)\}$.
\end{lemmarep}

\begin{appendixproof}
Note that given $a_0, a_1, b_0, b_1$ and $\mu$, the condition $\rsupp(\theta) = \{(a_0, b_0), (a_1, b_0), (a_0, b_1), (a_1, b_1)\}$ induces a system of linear equations.
$$\left\{\begin{aligned}
    &a_0(\mu p_1+(1-\mu) p_0) = \mu p_1\\
    &a_1(\mu(1-p_1)+(1-\mu) (1-p_0)) = \mu(1-p_1)\\
    &b_0(\mu q_1+(1-\mu) q_0) = \mu q_1\\
    &b_1(\mu(1-q_1)+(1-\mu) (1-q_0)) = \mu(1-q_1)
\end{aligned}\right.$$
By direct computation, we can show the system of linear equations is full rank and \eqref{eq:construct_info} is the unique solution.  

On the other hand, we need to show $p_0, p_1, q_0, q_1\in [0,1]$.  Because $a_0<\mu<a_1$, $p_1, p_0\ge 0$.  Additionally, because $1-p_1 = \frac{(\mu-a_0)a_1}{\mu (a_1-a_0)}\ge0$, and $1-p_0 = \frac{(1-a_1)(\mu-a_0)}{(1-\mu)(a_1-a_0)}\ge 0$, so $p_0, p_1\le 1$.  We have $p_0, p_1\in [0,1]$ and $q_0, q_1\in [0,1]$ by symmetry.  Therefore, $\theta$ is a valid conditional independent information structure with binary signals.
\end{appendixproof}

\section{Warm-up: Finite Set of Information Structures}\label{sec:finite}
In this section, we provide an algorithm which computes an approximate optimal robust aggregator in the finite setting when we can access an efficient $\epsilon$-best response oracle that outputs an $\epsilon$-approximate optimal aggregator for any distribution over information structures. With an $\epsilon$-best response oracle, we compute an approximate optimal robust aggregator using online learning techniques where the aggregator acts as an adversary to the distribution over information structures. We first provide a definition for the $\epsilon$-best response oracle, following which we give a FPTAS to compute a near optimal robust aggregator. 

\begin{definition}[$\epsilon$-Best Response]
Given $\epsilon\ge 0$, sets $\Thetafin$ and $\mathcal{F}$ with regret $R$, an \emph{$\epsilon$-best response oracle} inputs a distribution $\vw_{\Thetafin}\in\Delta_{\Thetafin}$, and outputs an aggregator $f$ such that for any $f'\in \mathcal{F}$, 
    $\E_{\theta\sim \vw_{\Thetafin}^t}[R(f,\theta)]\le \E_{\theta\sim \vw_{\Thetafin}^t}[R(f',\theta)]+\epsilon$.
\end{definition}
We show that when the set of information structures is finite, with convexity, continuity, and compactness conditions on regret $R$, there exists an FPTAS to solve the robust aggregation problem described in \Cref{alg:1}. \Cref{thm:finite} states our main theorem. 

\begin{theorem}\label{thm:finite}
    Suppose $|\Thetafin|=n$ and the size of support $\rsupp(\theta)$ is a constant for every $\theta\in\Thetafin$. When $\mathcal{F}$ is compact, the loss function $\ell$ is convex and continuous regarding $f$, lies in $[0,1]$, and there exists a  polynomial time $poly(n,1/\epsilon)$ oracle for the $\epsilon$-best response for any $0<\epsilon<1$, \Cref{alg:1} is an FPTAS which finds an $\epsilon$-optimal aggregator $\hat{f}$ over information structures $\Thetafin$ so that $\sup_{\theta\in\Thetafin} R(\hat{f},\theta)\le\inf_{f\in\mathcal{F}}\sup_{\theta\in\Thetafin} R(f,\theta)+\epsilon$.
\end{theorem}

\begin{algorithm}[!ht]
\caption{Online Learning For Finite $\Theta$}
\KwIn{Approximation parameter $\epsilon>0$, set of information structures $\Thetafin$, and a class of aggregators $\mathcal{F}$}
\KwOut{Optimal aggregator $f^*\in\mathcal{F}$}
Initialize the policy $\vw_{\Thetafin}^1\in\Delta_{\Thetafin}$ uniformly.\\
Set $T=\lceil25\epsilon^{-2}{\ln n}\rceil$ where $n=|\Thetafin|$ and $\eta=1/(1+\sqrt{2\ln n/T})$\\
\For{$t=1\ \mathbf{to}\ T$} {
Calculate the $\epsilon/5-$ best response $f^t\in\mathcal{F}$ to $\vw_{\Thetafin}^t$ such that for any $f\in\mathcal{F}$\\
\[\E_{\theta\sim \vw_{\Thetafin}^t}[R(f,\theta)]\le \E_{\theta\sim \vw_{\Thetafin}^t}[R(f,\theta)]+\epsilon/5.\]
Calculate the reward $\vu^t$ when the aggregator is $f^t$\\
\[
u^t_{\theta}=R(f^t,\theta)\text{ for all }\theta\in \Thetafin.
\]
Update the weights $\vw_{\Thetafin}^{t+1}\in\Delta_{\Thetafin}$
\[
w^{t+1}_{\theta}=w^{t}_{\theta} \exp(-\eta u_{\theta}^t)/Z_t\text{ for all }\theta\in \Thetafin
\]
where $Z_t$ is a normalization factor.
}
$f^*=\frac{1}{T}\sum_{t=1}^T f^t$
\label{alg:1}
\end{algorithm}

Before proving \Cref{thm:finite}, we discuss applications and the assumptions of the theorem. Notice that our theorem does not put any requirement for the format of reports, and thus applicable to a much wider range of scenarios. For example, we may ask agents for higher-order reports, which is the prediction for other agents' reports. As for the best response oracle, in many settings such as the forecast aggregation, the best response is often the posterior, which has an explicit formula. Besides, we show that when we have a polynomial time separation oracle and $\mathcal{F}$ is convex, we can give the best response oracle by the convex optimizer (\Cref{lem:separation}). 

\begin{lemma}\label{lem:separation}
    If $\ell$ is convex regarding $f$, $\mathcal{F}$ is compact and convex and there exists a polynomial time separation oracle, then there exists a polynomial time $\epsilon$-best response oracle.
\end{lemma}
\begin{proof}
If $\ell$ is convex regarding $f$, then $R(f,\theta)=\E_\theta[\ell(f(\vx),\omega)-\ell(opt_\theta(\vs),\omega)]$ is also convex regarding $f$. Then for any convex combination of $\Thetafin$, their expected regret $\E_{\theta\sim \vw_{\Thetafin}^t}[R(f,\theta)]$ is also convex. When $\mathcal{F}$ is compact and bounded, then calculating the $\epsilon$-best response is a convex optimization problem that can be solved by the ellipsoid method~\citep{boyd2004convex}. 

Since the separation oracle costs polynomial time, the optimization also costs polynomial time, which gives a polynomial time $\epsilon$-best response oracle.
\end{proof}

This directly leads to the following results. 

\begin{corollary}
Suppose $|\Thetafin|=n$ and the size of support $\rsupp(\theta)$ is a constant for every $\theta\in\Thetafin$. When $\mathcal{F}$ is compact and convex, the loss function $\ell$ is convex and continuous regarding $f$ and upper bounded by 1, and there exists a  polynomial time separation oracle, \Cref{alg:1} is an FPTAS which finds an $\epsilon$-optimal aggregator over information structures $\Thetafin$.
\end{corollary}

\paragraph{Application in Forecast Aggregation} The above results hold for the general robust information aggregation problem. In particular, for the forecast aggregation problem with no restriction on the set of aggregators, the best response is the Bayesian posterior which has a closed-form expression without the help of the polynomial time separation oracle. 
\begin{observation}[Efficient Best Response]\label{ob:response}
Given any distribution $\vw_{\Theta}$ over the information structures, the optimal aggregator will be
\begin{align*}
f_{\vw_{\Theta}}(\vx)=&\Pr[\omega=1|\vx,\vw_{\Theta}] =  \frac{\E_{\theta\sim \vw_{\Theta}}\Pr_\theta[\omega=1|\vx]\Pr_\theta[\vx]}{\E_{\theta\sim \vw_{\Theta}}\Pr_\theta[\vx]}
\end{align*}
\end{observation}

With \Cref{thm:finite} and \Cref{ob:response} we obtain the corollary:
\begin{corollary}
    Let $r=\max_{\theta\in\Thetafin} |\rsupp(\theta)|$ and $n = |\Thetafin|$. The forecast aggregation problem for finite setting can be solved with running time $O\left(r\frac{n \ln n}{\epsilon^2}\right)$.
\end{corollary}

\paragraph{Two Conditionally Independent Agents} We can apply the above results to the special setting of two conditionally independent agents with binary signals. We consider a natural discretization of the information structure, $\ThetabciNM$ where $N$ is the resolution scale of agents' reports, and $M$ is the resolution scale of the prior for the binary state. Using the prediction parameterization in ~\cref{eq:para1}, we define

\begin{equation}\label{eq:finite_infostruct0}
    \ThetabciNM:=\{\theta = (\mu,a_0,a_1,b_0,b_1)\in\ThetabciN:\mu\in[1/M]],a_i,b_j\in[1/N]\}. 
\end{equation}

Because $|\ThetabciNM|\le (N+1)^4(M+1) = O(N^4M)$, \Cref{thm:finite} implies the following corollary.
\begin{corollary}\label{cor:finite_time}
For all $\epsilon>0$, and $M, N\in \mathbb{N}$, \Cref{alg:1} finds an $\epsilon$-optimal aggregator over information structure $\ThetabciNM$ defined in \cref{eq:finite_infostruct0} with running time $O\left(\frac{N^4M \ln(N^4M)}{\epsilon^2}\right)$.
\end{corollary}

We will use $\ThetabciNM$ as a discretization of $\ThetabciN$ defined in \cref{eq:discrete_infostruct}. By picking sufficiently large $M$, we will use the above results to solve the robust forecast aggregation problem in the discrete reports setting with two conditionally independent agents. 

We now prove \Cref{thm:finite}. In \Cref{ss:zero-sum-redn}, we first show that we can obtain a near-optimal aggregator from the approximate equilibrium in a zero-sum game between nature and the aggregator, with action space $\Theta$ and $\mathcal{F}$ correspondingly. We then use the online learning algorithm to find an approximate equilibrium, the details of which are described in \Cref{ss:apxeqlbonlinelearn}. Combining these results, we analyze the guarantee of \Cref{alg:1} in \Cref{ss:pfthmfinite}.

\subsection{Approximate Equilibrium Implies Near-Optimal Aggregator}
\label{ss:zero-sum-redn}
To solve the minimax problem, we consider a zero-sum game between nature, who picks $\theta$, and aggregator who picks $f$. The game is described as:
\[
\inf_{\vw_\mathcal{F}\in\Delta_\mathcal{F}}\sup_{\vw_{\Thetafin}\in\Delta_{\Thetafin}}\E_{f\sim \vw_\mathcal{F},\theta\sim \vw_{\Thetafin}}[R(f,\theta)].
\]
First we prove that, when the function $R(f,\theta)$ is a convex function for $f$, the aggregator only needs pure strategy.

\begin{lemmarep}\label{lem:convex2pure}
If $R(f,\theta)$ is a convex function for $f$, for any set of information structures $\Theta$ and any set of aggregators $\mathcal{F}$,
\[ 
\inf_{f\in\mathcal{F}}\sup_{\vw_{\Theta}\in\Delta_{\Theta}}\E_{\theta\sim \vw_{\Theta}}[R(f,\theta)]=\inf_{\vw_\mathcal{F}\in\Delta_\mathcal{F}}\sup_{\vw_{\Theta}\in\Delta_{\Theta}}\E_{f\sim \vw_\mathcal{F},\theta\sim \vw_{\Theta}}[R(f,\theta)]
\]
\end{lemmarep}
\begin{proof}[Proof of \Cref{lem:convex2pure}]
Since we can pick $\vw_\mathcal{F}$ as a pure strategy, we have 
\[\inf_{f\in\mathcal{F}}\sup_{\vw_{\Theta}\in\Delta_{\Theta}}\E_{\theta\sim \vw_{\Theta}}[R(f,\theta)]\geq\inf_{\vw_\mathcal{F}\in\Delta_\mathcal{F}}\sup_{\vw_{\Theta}\in\Delta_{\Theta}}\E_{f\sim \vw_\mathcal{F},\theta\sim \vw_{\Theta}}[R(f,\theta)].\]
On the other hand, for any $\vw_\mathcal{F}\in \Delta_{\mathcal{F}}$, define $f_{\vw_\mathcal{F}}=\E_{f\sim \vw_\mathcal{F}}[f]\in \mathcal{F}$ since $\mathcal{F}$ is convex,
\begin{align*}
\inf_{f\in\mathcal{F}}\sup_{\vw_{\Theta}\in\Delta_{\Theta}}\E_{\theta\sim \vw_{\Theta}}[R(f,\theta)]
\le & \inf_{\vw_\mathcal{F}\in\Delta_\mathcal{F}}\sup_{\vw_{\Theta}\in\Delta_{\Theta}}\E_{\theta\sim \vw_{\Theta}}[R(f_{\vw_\mathcal{F}},\theta)]\le &\inf_{\vw_\mathcal{F}\in\Delta_\mathcal{F}}\sup_{\vw_{\Theta}\in\Delta_{\Theta}}\E_{f\sim \vw_\mathcal{F},\theta\sim \vw_{\Theta}}[R(f,\theta)].
\end{align*}
The last inequality holds because $R$ is convex in the first argument.
\end{proof}

Next, we prove that if we find an approximate equilibrium, then we can obtain a near-optimal aggregator. We define the approximate equilibrium as follows. 


\begin{definition}[$\epsilon$-Equilibrium, \citealp{roughgarden2010algorithmic}]\label{def:eps}
    In a two-player zero-sum game with pure strategy space $X,Y$ and outcome $g(x,y)$, a strategy profile $(\vw_{X}', \vw_{Y}')$ is an $\epsilon$-equilibrium if 
    $$\E_{x\sim \vw_{X}'}[\E_{y\sim \vw_{Y}'}[g(x,y)]]\ge \sup_{y\in Y}\E_{x\sim \vw_{X}'}[g(x,y)]-\epsilon
    \text{ and }
        \E_{y\sim \vw_{Y}'}[\E_{x\sim \vw_{X}'}[g(x,y)]]\le \inf_{x\in X}\E_{y\sim \vw_{Y}'}[g(x,y)]+\epsilon.
    $$
\end{definition}

\begin{lemmarep}\label{lem:2eps}
    When the outcome function $R(f,\theta)$ is convex for $f$, if a strategy profile $(\vw_\mathcal{F}', \vw_{\Thetafin}')$ is an $\epsilon$-equilibrium, let $f^*=\E_{f\sim \vw_\mathcal{F}^*}[f]$, then we have
    \begin{align*}
    R(f^*,\Thetafin)\le\inf_{f\in\mathcal{F}}R(f,\Thetafin)+2\epsilon
    \end{align*}
\end{lemmarep}
\begin{toappendix}
  \label{prf:2eps}
\end{toappendix}
\begin{proof}
Since $\E_{\theta\sim \vw_{\Thetafin}}[R(f,\theta)]$ is continuous regarding the strategies $f\in \mathcal{R}$ and $\vw_{\Thetafin}\in \Delta_{\Thetafin}$.\fang{$\vw$ or $w$?  Multiple places have this issue.} The following lemma shows that the minimax theorem holds when the aggregator chooses pure strategy and nature chooses mixed strategy.

\begin{lemma}[Glicksberg's theorem (Glicksberg, 1952)]\label{lem:minimax}
    If $\E_{\theta\sim \vw_{\Thetafin}}[R(f,\theta)]$ is continuous regarding $f$ and $\vw_{\Thetafin}$, $\mathcal{F}$ and $\Delta_{\Thetafin}$ is compact, then
    \[    \inf_{f\in\mathcal{F}}\sup_{\vw_{\Thetafin}\in\Delta_{\Thetafin}}\E_{\theta\sim \vw_{\Thetafin}}[R(f,\theta)]=
    \sup_{\vw_{\Theta}\in\Delta_{\Thetafin}}\inf_{f\in\mathcal{F}}\E_{\theta\sim \vw_{\Thetafin}}[R(f,\theta)].
    \]
\end{lemma}

By \Cref{lem:minimax} and \Cref{lem:convex2pure} we can infer that there exists a Nash equilibrium for the zero-sum game. 
\begin{lemma}[von Neumann, 1928]\label{lem:von1928} 
    If the minimax theorem holds for a two-player zero-sum game with pure strategy space $\mathcal{F},\Theta$,
    then 
    \[
    \inf_{\vw_{\mathcal{F}}\in\Delta_\mathcal{F}}\sup_{\theta\in\Theta}\E_{f\sim \vw_{\mathcal{F}}}[R(f,\theta)]=
    \sup_{\vw_{\Theta}\in\Delta_{\Theta}}\inf_{f\in\mathcal{F}}\E_{\theta\sim \vw_{\Theta}}[R(f,\theta)].
    \]
\end{lemma}

Now we can proof our main lemma.
\begin{align*}
    \sup_{\theta\in \Thetafin}R(f^*,\theta)&\le\sup_{\theta\in \Thetafin}\E_{f\sim \vw_\mathcal{F}'}[R(f,\theta)]\tag{convexity of $R$ in the first argument}\\
    &\le \E_{f\sim \vw_\mathcal{F}'}[\E_{\theta\sim \vw_{\Thetafin}'}[R(f,\theta)]]+\epsilon \tag{$(\vw_\mathcal{F}', \vw_{\Thetafin}')$ is an $\epsilon$-equilibrium Definition~\ref{def:eps}}\\
    &\le \inf_{f\in \mathcal{F}}\E_{\theta\sim \vw_{\Thetafin}'}[R(f,\theta)]+2\epsilon \tag{\Cref{def:eps}}\\
    &\le \sup_{\vw_{\Thetafin}\in\Delta_{\Thetafin}}\inf_{f\in\mathcal{F}}\E_{\theta\sim \vw_{\Thetafin}}[R(f,\theta)]+2\epsilon\\
    &=\inf_{\vw_{\mathcal{F}}\in\Delta_\mathcal{F}}\sup_{\theta\in\Thetafin}\E_{f\sim \vw_{\mathcal{F}}}[R(f,\theta)]+2\epsilon\tag{\Cref{lem:von1928}}\\
    &=\inf_{f\in\mathcal{F}}\sup_{\theta\in\Thetafin}R(f,\theta)+2\epsilon\tag{\Cref{lem:convex2pure}}
\end{align*}
\end{proof}

We defer both proofs of \Cref{lem:convex2pure} and \Cref{lem:2eps} to \Cref{prf:2eps}.

\subsection{Finding an Approximate Equilibrium by Online Learning Algorithm}
\label{ss:apxeqlbonlinelearn}
Thus we only need to find an $\epsilon$-equilibrium in the pure aggregator space and discrete information structure space. We show that \Cref{alg:1} will find an $\epsilon$-equilibrium.

In the online learning setting, nature only knows the action set $\Thetafin$ with $|\Thetafin|=n$.\fang{Do we need this?} For each round $t\in [T]$, nature will choose a probability distribution $\vw_{\Thetafin}^{t}\in \Delta_{\Thetafin}$ over the set of information structures and observes a vector of reward $\vu^t\in [0,1]^n$ where each coordinate of $\vu^t$ is the regrets of each information structures under the aggregator's best response $f^t$ at time $t$, that is, $u_{\theta}^t=R(f^t,\theta)$ and $f^t=\argmin_{f\in\mathcal{F}} \E_{\theta\sim \vw_{\Thetafin}^t}[R(f,\theta)]$. Nature needs to decide the distribution $\vw_{\Thetafin}^{t}$ given the historical online learning reward vectors $(\vu^{t'})_{t'=1}^{t-1}$ and minimize the online learning regret.

\begin{definition}[Online Learning Regret]
    The difference between the maximal reward and the expected reward.
    \[
    \mathcal{R}(T)=\max_{\theta\in \Thetafin}\sum_{t=1}^T u_{\theta}^t-\sum_{t=1}^T \vw_{\Thetafin}^t\cdot \vu^t
    \]
\end{definition}

Notice that we use \textit{max} instead of \textit{sup} here since $|\Thetafin|<\infty$. In particular, we use the multiplicative weights algorithm such that 
\[
w^{t+1}_{\theta}=w^{t}_{\theta} \exp(-\eta u_{\theta}^t)/Z_t
\]
where $Z_t$ is a normalization factor and $\eta$ is the learning rate.

Then we show that the online learning method can obtain an approximate equilibrium. It is a direct application of \citet{freund1999adaptive}. \yk{check it}
The proof is deferred to \Cref{prf:regret2eps}.
\begin{lemmarep}\label{lem:regret2eps}
Suppose $\vw_{\Thetafin}^1$ is a uniform distribution. Let $\eta=\frac{1}{1+\sqrt{2\ln n/T}}$. Then for any loss function $\ell$ whose value lies in $[0,1]$ and for any $f^t$, $\left(\frac1T\sum_{t=1}^T f^{t},\frac1T\sum_{t=1}^T \vw_{\Thetafin}^{t}\right)$ is a $\left(\frac{\sqrt{2T\ln n}+\ln n}T+\frac{\epsilon}{5}\right)$-equilibrium.
\end{lemmarep}
\begin{toappendix}
    \label{prf:regret2eps}
\end{toappendix}
\begin{proof}
    We have the following lemma for the regret bound.
    \begin{lemma}[Freund and Schapire, 1999]
        Suppose $\vw_{\Thetafin}^1$ is a uniform distribution. Let $\beta=\frac{1}{1+\sqrt{2\ln n/T}}$. Then for any online learning loss function $c$ whose value lies in $[0,1]$ and for any $f^t$, the following bound holds:
        $$\mathcal{R}(T)\le \sqrt{2T\ln n}+\ln n.$$
    
        Notice that when $T\to\infty$, the average regret $\frac{\mathcal{R}(T)}T$ can be arbitrarily small.
    \end{lemma}

    Let $\bar{\vw}_{\Thetafin}=\frac1T\sum_{t=1}^T \vw_{\Thetafin}^{t},\bar{f}=\frac1T\sum_{t=1}^T f^{t}$\footnote{$\bar{f}$ is a pure strategy here.}. Then we have the following chain of inequalities
    
    \begin{align*}
        \E_{\theta\sim \bar{\vw}_{\Thetafin}}[R(\bar{f},\theta)]\ge &\inf_{f\in\mathcal{F}}\E_{\theta\sim \bar{\vw}_{\Thetafin}}[R(f,\theta)]\tag{1}\\
        =&\inf_{f\in\mathcal{F}}\frac1T \sum_{t=1}^T\E_{\theta\sim \vw^t_{\Thetafin}}[R(f,\theta)]\\
        \ge &\frac1T \sum_{t=1}^T\inf_{f\in\mathcal{F}}\E_{\theta\sim \vw^t_{\Thetafin}}[R(f,\theta)]\\
        = &\frac1T \sum_{t=1}^T\E_{\theta\sim w^t_{\Thetafin}}[R(f^t,\theta)]-\epsilon/5 \tag{$f^t$ is the $\epsilon/5$ best response}\\
        \ge &\frac1T \sup_{\theta\in \Thetafin}\sum_{t=1}^T[R(f^t,\theta)]-\epsilon/5-\frac{\mathcal{R}(T)}T\tag{definition of regret $\mathcal{R}(T)$}\\
        \ge&\sup_{\theta\in \Thetafin}[R(\bar{f},\theta)]-\epsilon/5-\frac{\sqrt{2T\ln n}+\ln n}T\tag{2}\\
        \ge&\E_{\theta\sim \bar{w}_{\Thetafin}}[R(\bar{f},\theta)]-\epsilon/5-\frac{\sqrt{2T\ln n}+\ln n}T\tag{3}
    \end{align*}

    According to (1)(2) and (3)(1), we have
    \begin{align*}
    \sup_{\theta\in \Thetafin}[R(\bar{f},\theta)]-\epsilon/5-\frac{\mathcal{R}(T)}T&\le \E_{\theta\sim \bar{w}_{\Thetafin}}[R(\bar{f},\theta)]\le \inf_{f\in\mathcal{F}}\E_{\theta\sim \bar{w}_{\Thetafin}}[R(f,\theta)]+\epsilon/5+\frac{\mathcal{R}(T)}T\\
    \sup_{\theta\in \Thetafin}[R(\bar{f},\theta)]-\epsilon/5-\frac{\sqrt{2T\ln n}+\ln n}T&\le \E_{\theta\sim \bar{w}_{\Thetafin}}[R(\bar{f},\theta)]\le \inf_{f\in\mathcal{F}}\E_{\theta\sim \bar{w}_{\Thetafin}}[R(f,\theta)]+\epsilon/5+\frac{\sqrt{2T\ln n}+\ln n}T\tag{Insert the value of $\mathcal{R}(T)$}
    \end{align*}
    Thus we complete the proof.
\end{proof}

\subsection{Proof of~\Cref{thm:finite}}
\label{ss:pfthmfinite}
First we state that when calculating the best response is efficient, then \Cref{alg:1} is an FPTAS.

\begin{lemma}\label{lem:response}
    If calculating the $\epsilon$-best response $f^t$ costs polynomial time $poly(n,1/\epsilon)$, \Cref{alg:1} is an FPTAS with running time $O\left(\frac{(poly(n,1/\epsilon)+n)\ln n}{\epsilon^2}\right)$.
\end{lemma}

\begin{proof}
 For each information structure $\theta$, since the support of $\theta$ is a constant, we only need to enumerate the support to calculate $u_{\theta}^t$ which costs $O(1)$. So the loss step costs $O(n)$. Then updating the weights costs $O(n)$ to enumerate the information structures. Finally, the best response costs $poly(n,1/\epsilon)$. Thus for $T$ rounds, the total time complexity is $O(\frac{poly(n,1/\epsilon)+n}{\epsilon^2\ln n})$.
\end{proof}

The following corollary of \Cref{lem:2eps} and \Cref{lem:regret2eps} shows the near-optimal property of \Cref{alg:1}.

\begin{corollary}
\label{lem:finite_regret}
For any $0<\epsilon<1$, let $n=|\Thetafin|$, $T=\lceil\frac{\ln n}{\epsilon^2}\rceil$. The average output of \Cref{alg:1} $\bar{f}=\frac{1}{T}\sum_{t=1}^T f^t$ satisfies
$$R(\bar{f},\Thetafin)\le \inf_f R(f,\Thetafin)+5\epsilon.$$
\end{corollary}

By \Cref{lem:finite_regret} and \Cref{lem:response} we obtain our main theorem \ref{thm:finite}.

\section{Two Conditionally Independent Agents: Discrete Reports}
\label{sec:discrete}

We have shown how to use our general framework \cref{thm:finite} in the two conditionally independent agents setting with a finite set of information structures as \cref{cor:finite_time}. Now we will consider continuous sets of information structures with general signal. We first solve a partially continuous setting, the discrete report setting, of the robust forecast aggregation problem $$\inf_{f\in \mathcal{F}}\sup_{\theta\in\ThetaciN}R(f,\theta)$$
given that we have two conditionally independent agents. Recall that in \cref{eq:discrete_infostruct}
\[
\ThetaciN = \{\theta\in\Thetaci: \forall i, \forall s_i\in \mathcal{S}_i,x_i(s_i,\theta)\in [1/N]\}.
\]

\begin{theorem}[Discrete Setting]\label{thm:discrete}Given $N\in \mathbb{N}$, and $\epsilon>0$, there exists an algorithm that outputs an $\epsilon$-optimal aggregator over information structures $\ThetaciN$ in $\tilde{O}\left(\frac{N^5}{\epsilon^{13/2}}\right)$. \end{theorem}


\paragraph{Proof Sketch} With the dimension reduction results in \cref{lem:reduction}, it is sufficient to consider the set of conditionally independent information structures with binary signals and discrete reports, $\ThetabciN$. We then use a natural discretization of $\ThetabciN$, $\ThetabciNM$ defined in \cref{eq:discrete_infostruct}. We show that when $M = O(N\epsilon^{-9/2})$, $\ThetabciNM$ is a sufficiently good representation of $\ThetabciN$.  Then we run \Cref{alg:1} with input $\ThetabciNM$ and output an $\epsilon$-optimal aggregator $f^*$. 

More formally, $\ThetabciNM$ is a sufficiently good representation of $\ThetabciN$ if for all $\theta\in \ThetabciN$ \fang{for all $f$ as well?} there exists $\theta'\in \ThetabciNM$ such that $R(f, \theta)\approx R(f, \theta')$. This implies $R(f, \ThetabciNM)\approx R(f, \ThetabciN)$ for any $f$. Thus $\inf_f R(f, \ThetabciNM)\approx \inf_f R(f, \ThetabciN)$. Then by our reduction result, $\inf_{f} R(f,\ThetabciN) = \inf_f R(f,\ThetaciN)$. Hence we will have 
\[R(f^*, \ThetabciNM)\approx \inf_{f} R(f,\ThetabciNM)\approx \inf_{f} R(f,\ThetabciN) = \inf_f R(f,\ThetaciN).\] The first approximate equality follows from the property of \Cref{alg:1}. The second follows from the fact that $\ThetabciNM$ is a sufficiently good representation of $\ThetabciN$, and the last equality follows from the dimension reduction results. 

The main technical part of the proof is showing the second approximate equality. We will use the following concept to show it.

\begin{definition}[$(\epsilon, d)$-Covering]
Given a metric $d$ over space $\Phi$, a set $A\subset \Phi$ is an \textbf{\textit{$(\epsilon, d)$-covering}} (or a $d$-covering with $\epsilon$) of a set $B\subset \Phi$ if for all $x\in B$, there exists $y\in A$ such that $x$ and $y$ are $\epsilon$-close under the metric, i.e., $d(x,y)\le \epsilon$.    
\end{definition}

To this end, we will pick a proper metric on information structures and find a good covering for $\ThetabciN$.  Note that in our setting an aggregator $f\in \mathcal{F}$ and omniscient aggregator takes the predictions $\Xdis{\theta}$ as input.  Thus, we will show that bounding the distance between $\Xdis{\theta}$ and $\Xdis{\theta'}$ is sufficient for our approximation argument for any pair of information structures $\theta$ and $\theta'$. 
Specifically, in the discrete reports setting, for each $\theta\in\ThetabciN$, the marginal distribution over reports $\Xdis{\theta}$ has a discrete support $\{0,\frac1N,\cdots,1\}^2 = [1/N]^2$. We use total variation distance (TVD) to measure the distance between $\Xdis{\theta}$ and $\Xdis{\theta'}$. In later sections where the support is continuous, TVD can be too restricted. Thus, later we will use the earth mover's distance (EMD).


\begin{definition}[Total Variation Distance, \citealp{Chatterjee2008}]\label{def:tvd}
    Given two distributions $P$ and $Q$ on $[0,1]^2$, we introduce total variation distance,
\[d_{TV}(P,Q) = \frac{1}{2}\sup_{h:\|h\|_\infty\le 1}\E_P[h]-\E_Q[h].\] Moreover, the optimal $h(\vx) = \begin{cases}1 & P(\vx)\geq Q(\vx)\\
-1& P(\vx)<Q(\vx)\end{cases}$ and $d_{TV}(P,Q) = \frac{1}{2}\sum|P(\vx)-Q(\vx)|$.
\end{definition}

We induce TVD-covering from the total variation distance.
\begin{definition}[$(\epsilon, d_{TV})$-Covering]
A set $A\subset \Theta$ is an \textbf{\textit{$(\epsilon, d_{TV})$-covering}} of a set $B\subset \Theta$ if for all $\theta\in B$, there exists $\theta'\in A$ so that $d_{TV}(\Xdis{\theta},\Xdis{\theta'})\le \epsilon$. By abusing the notation a little bit, we set $d_{TV}(\theta,\theta')=d_{TV}(\Xdis{\theta},\Xdis{\theta'})$ and call the corresponding $(\epsilon, d_{TV})$-covering the TVD-covering.  
\end{definition}

With the above concept, we will use two steps to show $\ThetabciNM$ is a sufficiently good representation of $\ThetabciN$. First, we bound the change of regret by the TVD between predictions (\Cref{sec:dtv_smooth}). Formally, from \Cref{lem:dtv_smooth}, we induce that $R(f,\Theta')\approx R(f,\Theta'')$ for any pair of families of information structures $\Theta',\Theta''$ if $\Theta'$ is a good $d_{TV}$-coverings for $\Theta''$. Second, we prove in \Cref{lem:dtv_covernum} that $\ThetabciNM$ is a good $d_{TV}$-covering of $\ThetabciN$ (\Cref{sec:tvdcovering}), which completes the proof of our main theorem. 

\subsection{Regret is Insensitive with Respect to TVD}\label{sec:dtv_smooth}

We first show that when the information structure is changed from $\theta$ to $\theta'$ the regret function $R$ does not change much if $\Xdis{\theta}$ and $\Xdis{\theta'}$ are close in total variation distance. Recall that $\Xdis{\theta}$ is the marginal distribution over reports as we defined in \Cref{sec:pre_bci}.

\begin{proposition}\label{lem:dtv_smooth}
For any aggregator $f:[0,1]^2\to [0,1]$, and conditional independent information structures $\theta, \theta'\in \Thetaci$
\begin{equation}\label{eq:dtv_covering1}
    R(f,\theta)\le R(f, \theta')+105d_{TV}(\Xdis{\theta}, \Xdis{\theta'})^{2/9}.
\end{equation}
\end{proposition}

To prove \Cref{lem:dtv_smooth}, we observe that the regret $R(f,\theta) = \E_\theta[\ell(f(\vx),\omega)]-\E_\theta[\ell(g_\mu(\vx),\omega)]$ can be affected by $\theta$ (or $\theta'$) in two ways: 1) the distribution over the aggregator $f$'s input, the reports, and 2) the omniscient aggregator $g_\mu(\vx)$.  The first term satisfies $\E_\theta[\ell(f(\vx),\omega)]\approx \E_{\theta'}[\ell(f(\vx),\omega)]$ as the total variation distance between $\Xdis{\theta}$ and $\Xdis{\theta'}$ is small.  However, the second term can be sensitive to $\mu$. For instance, \Cref{fig:sensitive prior} demonstrates examples that show that fixing $\vx$, $g_\mu(\vx)$ can be very sensitive to $\mu$ especially when $\mu$ is close to 0 or 1. 

To handle the second term, we first observe that if $\mu$ is close to zero, $\vx$ should also be close to zero with large probability (\Cref{lem:ext1}). The case when $\mu$ is close to one is analogous. Thus, we can still bound the difference in expectation $\E_{\theta}[|g_\mu(\vx)-g_{\mu'}(\vx)|]$ as shown in \Cref{lem:prior_shift}. 

\begin{figure}
    \centering
    \includegraphics[width = 0.4\textwidth]{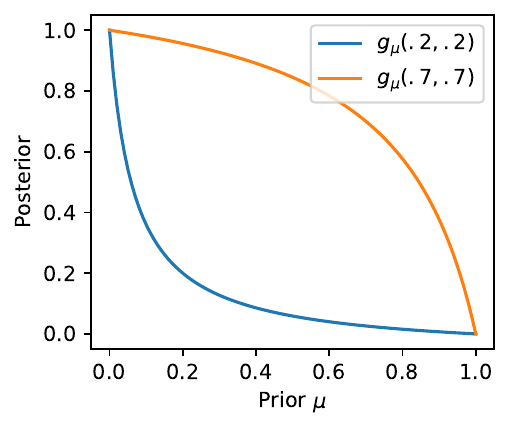}
    \caption{\textbf{Sensitivity of the omniscient aggregator $g_\mu(\vx)$ regarding $\mu$}. We fix agents' reports $\vx=(.2, .2)$ and $\vx=(.7, .7)$, and plot the Bayesian posteriors $g_\mu(\vx)$ as function of prior $\mu$. The figures show that 1) surprisingly, as the prior increases, the posterior decreases. Intuitively, fixing the agents' reports, as the prior probability of rain increases, the reports will increasingly resemble negative evidence for rain. Consequently, the true posterior probability should decrease.  2) $g_\mu(\vx)$ is highly sensitive when the prior $\mu$ is near certain and closely aligned with the reports.}
    \label{fig:sensitive prior}
\end{figure}


\begin{lemmarep}\label{lem:ext1}
For any $\theta\in \Thetaci$ and $\epsilon>0$, the predictions is concentrated
\begin{equation}
\Pr_\theta\left[x_1\le \epsilon\right]\le \frac{1-\mu}{1-\epsilon}\text{ and }\Pr_\theta\left[x_1\ge 1-\epsilon\right]\le \frac{\mu}{1-\epsilon}.\label{eq:ext_1}
\end{equation}
\end{lemmarep}
In particular, \Cref{lem:ext1} ensures $\Pr_\theta\left[x_1\le \epsilon\right]\approx 0$ if $\mu\approx 1$, and $\Pr_\theta\left[x_1\ge 1-\epsilon\right]\approx0$ is $\mu\approx 0$. 

\begin{toappendix}
    \label{prf:ext1}
\end{toappendix}
\begin{proof}[Proof of \Cref{lem:ext1}]
Suppose the prior $\mu \in (0,1)$, or otherwise the inequality trivially holds.  
Let $E_\epsilon = \{s_1: x_1(s_1)\le \epsilon\}$.  By the law of total probability, $1-\mu = \E_\theta[1-x_1] \ge \sum_{s_1\in E_\epsilon} (1-x_1(s_1))\Pr_\theta[s_1]\ge (1-\epsilon)\Pr_\theta[E_\epsilon]$.  Therefore,
$$\Pr_\theta[x_1\le \epsilon]\le \frac{1-\mu}{1-\epsilon}.$$
Similarly, 
$\mu = \E_\theta[x_1] \ge \sum_{s_1:x_1(s_1)\ge 1-\epsilon}x_1\Pr_\theta[s_1]\ge (1-\epsilon)\Pr_\theta[x_1\ge 1-\epsilon]$.

\end{proof}


We then show that the distance between the prior $\mu$ and $\mu'$ is bounded by the distance between $\Xdis{\theta}$ and $\Xdis{\theta'}$. The results hold for both total variation distance and earth mover's distance which we will define in the next section.

\begin{lemmarep}\label{lem:prior_bound}
    Given $\theta$ and $\theta'\in \Thetaci$, the difference of prior $|\mu-\mu'|$ is less than both $d_{EM}(\Xdis{\theta}, \Xdis{\theta'})$ and $d_{TV}(\Xdis{\theta}, \Xdis{\theta'})$.
\end{lemmarep}
\begin{proof}[Proof of \Cref{lem:prior_bound}]
    Because $\E_\theta[x_1] = \mu$ and $\E_{\theta'}[x_1] = \mu'$, 
    $$|\mu-\mu'| = |\E_\theta[x_1]-\E_{\theta'}[x_1]|$$
    First, the function $(x_1, x_2)\mapsto x_1$ is a $1$-Lipschitz function, we prove the first part.  The second part follows because the function $(x_1, x_2)\mapsto x_1$ is bounded between zero and one.
\end{proof}

Finally, we show that the distance between the prior $\mu$ and $\mu'$ can upper bound the expected distance between the omniscient aggregator's posteriors.  

\begin{claimrep}\label{lem:prior_shift}
     If $\epsilon<1/8$ , for all $\theta$ and $\theta'\in \Thetaci$ with $|\mu-\mu'| = \epsilon$ and $|1/2-\mu|>|1/2-\mu'|$, 
     $$\E_{\theta}[|g_\mu(\vx)-{g}_{\mu'}(\vx)|]\le 52\epsilon^{2/9}.$$
\end{claimrep} 
\begin{proof}[Proof of \Cref{lem:prior_shift}]  
By symmetry, we can assume $\mu<1/2$.  Let $\sigma = \frac{\mu}{1-\mu}$, and $\sigma' = \frac{\mu'}{1-\mu'}$.   For any $\vx= (x_1, x_2)\neq (0,1), (1,0)$, let $\xi = \frac{x_1x_2}{(1-x_1)(1-x_2)}\in \R_{\ge 0}\cup \{+\infty\}$.  By the definition of omniscient aggregator in \cref{eq:omniscient}, $g_\mu(\vx) = \frac{\xi}{\xi+\sigma}$, and 
$$|g_\mu(\vx)-g_{\mu'}(\vx)| = \frac{|\sigma'-\sigma|\xi}{\xi^2+(\sigma+\sigma')\xi+\sigma\sigma'}.$$
Because $1\ge (1-\mu)\ge 1/2$, and $1\ge (1-\mu')\ge 1-\mu-\epsilon\ge 1/3$, $\sigma$ and $\sigma'$ satisfies $\mu\le \sigma\le 2\mu$, $\mu'\le \sigma'\le 3\mu'$, and 
\begin{equation}\label{eq:prior_shift1}
    |\sigma'-\sigma| = \frac{1}{(1-\mu)(1-\mu')}|\mu'-\mu|\le 6\epsilon.
\end{equation}

To prove the inequality, we consider two cases: $\mu\ge \epsilon^{1-\alpha}$ or $\mu<\epsilon^{1-\alpha}$ for any $\alpha\le 1/3$.  

\paragraph{Case 1: $\mu\ge \epsilon^{1-\alpha}$} We will show the difference is always bounded by $\epsilon^{\alpha}$.  
For any $\vx = (x_1, x_2)\neq (0,1), (1,0)$,
\begin{align*}
|g_\mu(\vx)-g_{\mu'}(\vx)| =&\frac{|\sigma'-\sigma|\xi}{\xi^2+(\sigma+\sigma')\xi+\sigma\sigma'}\\
\le& \frac{|\sigma'-\sigma|}{\sigma+\sigma'}\tag{$\xi, \sigma, \sigma'$ are nonnegative}\\
\le& \frac{2}{3}\epsilon^{\alpha-1} |\sigma'-\sigma|\tag{$\sigma\ge \mu\ge \epsilon^{1-\alpha}$, and $\sigma'\ge \mu'\ge \mu-\epsilon\ge 1/2\epsilon^{1-\alpha}$.}\\
\le& \frac{2}{3}\epsilon^{\alpha-1}\cdot{6\epsilon}= 4\epsilon^{\alpha}\tag{by \cref{eq:prior_shift1}}
\end{align*}
Because $\vx$ equals $(0,1)$ or $(1,0)$ with probability zero, we have
\begin{equation}\label{eq:prior_shift2}
    \E_\theta[|g_\mu(\vx)-g_{\mu'}(\vx)|]\le 4\epsilon^{\alpha}.
\end{equation}

\paragraph{Case 2: $\mu\le \epsilon^{1-\alpha}$}  Because $\epsilon^{1-\alpha}\le \epsilon^{2/3}\le 1/4$ and $|1/2-\mu|>|1/2-\mu'|$, $\mu$ is less than $\mu'$ and $\mu' = \epsilon+\mu>\epsilon$.  If $\vx = (x_1, x_2)$ satisfies $x_1, x_2\le \mu^{1-\alpha}\le 1/2$, 
\begin{align*}
    |g_\mu(\vx)-g_{\mu'}(\vx)|\le& \frac{|\sigma'-\sigma|\xi}{\sigma\sigma'}\tag{$\xi, \sigma, \sigma'$ are nonnegative}\\
    \le& 6\frac{\epsilon\xi}{\sigma\sigma'}\tag{by \cref{eq:prior_shift1}}\\
    \le& 24\frac{\epsilon\sigma^{2-2\alpha}}{\sigma\sigma'}\tag{because $\xi = \frac{x_1x_2}{(1-x_1)(1-x_2)}\le 4\mu^{2-2\alpha}\le 4\sigma^{2-2\alpha}$}\\
    \le& 24\sigma^{1-2\alpha}\tag{$\sigma'>\mu'>\epsilon$}
\end{align*}
On the other hand, $\Pr_\theta[x_1\ge \mu^{1-\alpha}]\le \frac{\mu}{\mu^{1-\alpha}}\le \mu^{\alpha}\le \sigma^{\alpha}$ and $\Pr_\theta[x_2\ge \mu^{1-\alpha}]\le \sigma^{\alpha}$ by \Cref{lem:ext1}.  Therefore, by union bound and $\alpha\le 1/3$, we have
\begin{equation}\label{eq:prior_shift3}
   \E_\theta |g_\mu(\vx)-g_{\mu'}(\vx)|\le 2\sigma^{\alpha}\cdot \|g_\mu(\vx)-g_{\mu'}(\vx)\|_\infty+24\sigma^{1-2\alpha} \le  26\sigma^{\alpha}\le 26\cdot 2^\alpha \mu^\alpha\le 52 \epsilon^{\alpha(1-\alpha)}.
\end{equation}
Combining \cref{eq:prior_shift2,eq:prior_shift3} completes the proof by taking $\alpha = 1/3$.
\end{proof}

The proofs of \Cref{lem:ext1}, \Cref{lem:ext2} and \Cref{lem:prior_shift} are deferred to \Cref{{prf:ext1}}. With the above auxiliary lemmas, we start to prove \Cref{lem:dtv_smooth}. 
\begin{proof}[Proof of \Cref{lem:dtv_smooth}]
Given $\theta$ and $\theta'$ let $\epsilon_{TV} = d_{TV}(\Xdis{\theta}, \Xdis{\theta'})$, and $\epsilon = |\mu-\mu'|$ which is less than $\epsilon_{TV}$ by \Cref{lem:prior_bound}.  We want to bound the difference of losses on $\theta$ and $\theta'$.  When $|1/2-\mu|>|1/2-\mu'|$,
\begin{align*}
    R(f', \theta)-R(f', \theta') =& \E_\theta[\left(g_\mu(\vx)-f(\vx)\right)^2]-\E_{\theta'}[\left(g_{\mu'}(\vx)-f(\vx)\right)^2]\\
    =&\left(\E_{\theta}[\left(g_{\mu'}(\vx)-f(\vx)\right)^2]-\E_{\theta'}[\left(g_{\mu'}(\vx)-f(\vx)\right)^2]\right)\\
    &+\left(\E_\theta[\left(g_\mu(\vx)-f(\vx)\right)^2]-\E_{\theta}[\left(g_{\mu'}(\vx)-f(\vx)\right)^2]\right)\\
\end{align*}
The first term is the difference of one function on two distributions, $\Xdis{\theta}$ and $\Xdis{\theta'}$, and the second term is the expected difference of two functions on the distribution $\Xdis{\theta}$.  

For the first term, because $0\le \left(g_{\mu}(\vx)-f(\vx)\right)^2\le 1$ for all $\vx$, 
$ \E_\theta[\left(g_{\mu}(\vx)-f(\vx)\right)^2]-\E_{\theta'}[\left(g_{\mu}(\vx)-f(\vx)\right)^2]\le d_{TV}(P_\theta, P_{\theta'}) = \epsilon_{TV}$.
 For the second term, 
\begin{align*}
    &\E_{\theta}[\left(g_\mu(\vx)-f(\vx)\right)^2]-\E_{\theta}[\left(g_{\mu'}(\vx)-f(\vx)\right)^2]\\
    =&  \E_{\theta}[\left(g_\mu(\vx)+g_{\mu'}(\vx)-2f(\vx)\right)\left(g_\mu(\vx)-g_{\mu'}(\vx)\right)]\\
    \le& 2\E_{\theta}[|g_\mu(\vx)-g_{\mu'}(\vx)|]\tag{$\|g_\mu(\vx)+g_{\mu'}(\vx)-2f(\vx)\|_\infty\le 2$}\\
    \le& 104 |\mu-\mu'|^{2/9} = 104\epsilon^{2/9}\tag{by \Cref{lem:prior_shift}}
\end{align*}
Combining these two gets $R(f', \theta)\le R(f', \theta')+104\epsilon^{2/9}+\epsilon_{TV}\le R(f', \theta')+105\epsilon_{TV}^{2/9}$.  

For the other two case, if $|1/2-\mu'|<|1/2-\mu|$, we can write 
\begin{align*}
R(f', \theta)-R(f', \theta') = &\left(\E_\theta[\left(g_\mu(\vx)-f(\vx)\right)^2]-\E_{\theta'}[\left(g_{\mu}(\vx)-f(\vx)\right)^2]\right)\\
    &+\left(\E_{\theta'}[\left(g_{\mu}(\vx)-f(\vx)\right)^2]-\E_{\theta'}[\left(g_{\mu'}(\vx)-f(\vx)\right)^2]\right).
\end{align*}
and the same argument applies.  Finally, if $|1/2-\mu'| = |1/2-\mu|$, the error can be further reduced to $\epsilon_{TV}$.  
\end{proof}

\subsection{$\ThetabciNM$ is a Small TVD Covering of $\ThetabciN$}\label{sec:tvdcovering}\fang{This is not that exciting and we can move to appendix.}

To complete the proof of \Cref{thm:discrete}, we show that $\ThetabciNM$ is a small total variation covering of $\ThetabciN$.
\begin{lemmarep}\label{lem:dtv_covernum}  For all $N$ and $M$ in $\mathbb{N}$ with $M>N$, $\ThetabciNM$ is a $(\frac{6N}{M}, d_{TV})$-covering of $\ThetabciN$.
\end{lemmarep}
\paragraph{Proof Sketch}
To prove \Cref{lem:dtv_covernum}, given any $\theta' = (\mu', a_0', a_1', b_0', b_1')\in \ThetabciN$, we should find \\$\theta = (\mu, a_0, a_1, b_0, b_1)\in \ThetabciNM$, such that $d_{TV}(\Xdis{\theta},\Xdis{\theta'})\leq \frac{6N}{M}$.

We start by setting $\mu$ close to $\mu'$ and constructing $\theta = (\mu, a_0, a_1, b_0, b_1)\in \ThetabciNM$ using \Cref{lem:construct_info} such that $\theta$ has the same reports support as $\theta'$. Next, we upper bound the difference of the marginal distribution over reports, $\left|\Pr_\theta[(a_i, b_j)]-\Pr_{\theta'}[(a_i, b_j)]\right|$ for all $i,j = 0,1$. Based on the transformations between the marginal distributions over reports and bijection formulas introduced in \Cref{lem:construct_info}, we will show that when $a_1-a_0$ and $b_1-b_0$ are lower bounded by $1/N$, changing $\mu$ will not significantly affect  $\Pr_\theta[(a_i, b_j)]$. When $a_1-a_0=0$ or $b_1-b_0=0$, we have $\mu'=a_1$ or $b_1$, which indicates $\theta'\in \ThetabciNM$. So we can directly set $\theta=\theta'$, which completes the proof.
We defer the proof of \Cref{lem:dtv_covernum} to \Cref{prf:dtv_covernum}

\begin{toappendix}
  \label{prf:dtv_covernum}
\end{toappendix}
\begin{proof}[Proof for \Cref{lem:dtv_covernum}]
To prove \Cref{lem:dtv_covernum}, we need to show that for all ${\theta}'\in \ThetabciN$, there exists $\theta\in \ThetabciNM$ so that 
$$d_{TV}(P_\theta(x_1,x_2), P_{\theta'}(x_1,x_2))\le  \frac{6N}{M}$$
Given $\theta'\in \ThetaciN$ and $\theta' = (\mu', p_0', p_1', q_0', q_1')$, there exist $0\le a_0, a_1, b_0, b_1\le 1$ so that  $Na_0, Na_1, Nb_0, Nb_1\in \mathbb{Z}$ and the predictions of two agents are
\begin{align*}
    &{a_0} = \frac{\mu' p_1'}{\mu' p_1'+(1-\mu')p_0'}, a_1 = \frac{\mu' (1-p_1')}{\mu' (1-p_1')+(1-\mu')(1-p_0')},\\
    &{b_0} = \frac{\mu' q_1'}{\mu' q_1'+(1-\mu')q_0'}\text{ and }{b_1} = \frac{\mu' (1-q_1')}{\mu' (1-q_1')+(1-\mu')(1-q_0')}
\end{align*}

When $a_0=a_1$, then $a_0=a_1=\mu$. So $\theta'\in \ThetabciNM$, which is trivial. It is similar when $b_0=b_1$. So we assume $a_0< a_1$ and $b_0<b_1$.

We may always switch $p_0', p_1'$ (and $q_0', q_1'$) so that $a_0< a_1$ and $b_0< b_1$. Since the prior $\mu'$ is a convex combination of the posteriors, ${a_0}< \mu'< {a_1}$ and ${b_0}< \mu'< {b_1}$.

Since $M\ge N$, we construct $\theta = (\mu, p_0, p_1, q_0, q_1)\in \ThetabciNM$ as the following: Set $\mu$ such that ${a_0}< \mu< {a_1}$, ${b_0}< \mu< {b_1}$, and $|\mu-\mu'|\le 1/M$.  Given $a_0, a_1, b_0, b_1$ and above $\mu$, we use \Cref{lem:construct_info} to set $p_1, p_0, q_1, q_0$ so that $\theta$ has the same predictions as $\theta'$, and hence $\theta\in \ThetabciNM$.

Finally, we show the total variation distance between $\theta$ and $\theta'$ is small.  \begin{align*}
    &\Pr_{\theta'}[(x_1, x_2) = (a_0, b_0)] = \mu'p_1'q_1'+(1-\mu')p_0'q_0'\\
    =& \mu'\frac{a_0(a_1-\mu')}{\mu'(a_1-a_0)}\frac{b_0(b_1-\mu')}{\mu'(b_1-b_0)}+(1-\mu')\frac{(1-a_0)(a_1-\mu')}{(1-\mu')(a_1-a_0)}\frac{(1-b_0)(b_1-\mu')}{(1-\mu')(b_1-b_0)}\tag{by \Cref{lem:construct_info}}\\
    =& \frac{(a_1-\mu')(b_1-\mu')}{\mu'}\frac{a_0b_0}{(a_1-a_0)(b_1-b_0)}+\frac{(a_1-\mu')(b_1-\mu')}{(1-\mu')}\frac{(1-a_0)(1-b_0)}{(a_1-a_0)(b_1-b_0)}
\end{align*}
and 
$\Pr_{\theta}[(x_1, x_2) = (a_0, b_0)] = \frac{(a_1-\mu)(b_1-\mu)}{\mu}\frac{a_0b_0}{(a_1-a_0)(b_1-b_0)}+\frac{(a_1-\mu)(b_1-\mu)}{(1-\mu)}\frac{(1-a_0)(1-b_0)}{(a_1-a_0)(b_1-b_0)}$.  
By Taylor expansion there exists $\mu''$ between $\mu$ and $\mu'$ so that  
\begin{equation}\label{eq:construct_info1}
    \frac{(a_1-\mu')(b_1-\mu')}{\mu'}-\frac{(a_1-\mu)(b_1-\mu)}{\mu} = \left(\frac{-(a_1-\mu'')(b_1-\mu'')}{(\mu'')^2}-\frac{b_1-\mu''}{\mu''}-\frac{a_1-\mu''}{\mu''}\right)(\mu'-\mu)
\end{equation}
and $\mu'''$ between $\mu$ and $\mu'$ so that   
\begin{equation}\label{eq:construct_info2}
    \frac{(a_1-\mu')(b_1-\mu')}{(1-\mu')}-\frac{(a_1-\mu)(b_1-\mu)}{(1-\mu)} = \left(\frac{(a_1-\mu''')(b_1-\mu''')}{(1-\mu''')^2}-\frac{b_1-\mu'''}{1-\mu'''}-\frac{a_1-\mu'''}{1-\mu'''}\right)(\mu'-\mu)
\end{equation}

\begin{align*}
    &|\Pr_{\theta'}[(x_1, x_2) = (a_0, b_0)]-\Pr_{\theta}[(x_1, x_2) = (a_0, b_0)]|\\
    \le&\left| \frac{(a_1-\mu')(b_1-\mu')}{\mu'}-\frac{(a_1-\mu)(b_1-\mu)}{\mu}\right|\frac{a_0b_0}{(a_1-a_0)(b_1-b_0)}\\
    &+\left|\frac{(a_1-\mu')(b_1-\mu')}{(1-\mu')}-\frac{(a_1-\mu)(b_1-\mu)}{(1-\mu)}\right|\frac{(1-a_0)(1-b_0)}{(a_1-a_0)(b_1-b_0)}\\
    \le&\left(\frac{(a_1-\mu'')(b_1-\mu'')}{(\mu'')^2}+\frac{b_1-\mu''}{\mu''}+\frac{a_1-\mu''}{\mu''}\right)\frac{a_0b_0}{(a_1-a_0)(b_1-b_0)}(\mu'-\mu)\\
    &+\left(\frac{(a_1-\mu''')(b_1-\mu''')}{(1-\mu''')^2}+\frac{b_1-\mu'''}{1-\mu'''}+\frac{a_1-\mu'''}{1-\mu'''}\right)\frac{(1-a_0)(1-b_0)}{(a_1-a_0)(b_1-b_0)}(\mu'-\mu)\tag{by \cref{eq:construct_info1,eq:construct_info2}}\\
    \le&\left(\frac{(a_1-a_0)(b_1-b_0)}{a_0b_0}+\frac{b_1-b_0}{b_0}+\frac{a_1-a_0}{a_0}\right)\frac{a_0b_0}{(a_1-a_0)(b_1-b_0)}(\mu'-\mu)\\
    &+\left(\frac{(a_1-a_0)(b_1-b_0)}{(1-a_0)(1-b_0)}+\frac{b_1-b_0}{1-b_0}+\frac{a_1-a_0}{1-a_0}\right)\frac{(1-a_0)(1-b_0)}{(a_1-a_0)(b_1-b_0)}(\mu'-\mu)\tag{because $\mu'', \mu'''$ are in $[a_0, a_1]$ and $[b_0, b_1]$}\\
    =&\left(1+\frac{a_0}{a_1-a_0}+\frac{b_0}{b_1-b_0}+1+\frac{1-a_0}{a_1-a_0}+\frac{1-b_0}{b_1-b_0}\right)(\mu'-\mu)\\
    \le& (2+\frac{1}{a_1-a_0}+\frac{1}{b_1-b_0})\frac{1}{M}\tag{$|\mu'-\mu|\le 1/M$}\\
    \le& \frac{3N}{M}\tag{$a_1-a_0, b_1-b_0\ge 1/N$}
\end{align*}
By similar argument, we have $|\Pr_{\theta'}[(x_1, x_2) = (a_1, b_0)]-\Pr_{\theta}[(x_1, x_2) = (a_1, b_0)]|$, $|\Pr_{\theta'}[(x_1, x_2) = (a_0, b_1)]-\Pr_{\theta}[(x_1, x_2) = (a_0, b_1)]|$, and $|\Pr_{\theta'}[(x_1, x_2) = (a_1, b_1)]-\Pr_{\theta}[(x_1, x_2) = (a_1, b_1)]|$ all bounded by $\frac{3N}{M}$.  Therefore, 
$$d_{TV}(P_\theta(x_1,x_2), P_{\theta'}(x_1,x_2)) = \frac{1}{2}\sum_{a,b} |\Pr_{\theta'}[(x_1, x_2) = (a, b)]-\Pr_{\theta}[(x_1, x_2) = (a, b)]| \le \frac{6N}{M}$$
which completes the proof.
\end{proof}

\subsection{Proof of \Cref{thm:discrete}}

After showing that the regret function is insensitive regarding TVD and $\ThetabciNM$ is a small TVD Covering of $\ThetabciN$, we can prove \Cref{thm:discrete}.

\begin{toappendix}
    \begin{claim}\label{clm:compactness}
        $\mathcal{F}$ is compact.
    \end{claim}
    \begin{proof}
    Recall that $$\lvert|f\rvert|_{\infty} = \sup_{\mathbf{x} \in [0,1]^{2}}|f(\mathbf{x})|.$$ We first show that $\left(\mathcal{F},\lvert|\cdot\rvert|_{\infty}\right)$ is complete.\fang{It seem a standard result.  We can defer the proof to appendix as a claim.  (Here the input of $f$ is restricted in $[1/N]^2$, so the function space equivalent to vector space $\R^{(N+1)^2}$} Consider a Cauchy sequence $\left(f_{n}\right)$ where $f_{n} \in \mathcal{F}$. For any $\mathbf{x} \in [0,1]^{2}$ and by the sup-norm, it follows that $(f_{n}(\mathbf{x}))$ is a Cauchy sequence on $\mathbb{R}$ and thus converges in $\mathbb{R}$. Define $f(\mathbf{x}) := \lim_{n \to \infty}f_{n}(\mathbf{x})$. We first show that the function $f(\cdot) \in \mathcal{F}$. Note that for any $\mathbf{x} \in [0,1]^{2}$, each element in the sequence $\left(f_{n}(\mathbf{x})\right)$ is upper bounded by $1$ and lower bounded by $0$. Since $\left(f_{n}(\mathbf{x})\right)$ converges to $f(\mathbf{x})$, we have that $f(\mathbf{x}) \in [0,1]$, for any $\mathbf{x} \in [0,1]^{2}$ and thus $f \in \mathcal{F}$. To show that $(f_{n})$ converges to $f$, for any $\epsilon > 0$, there exists $N$ such that $\sup_{\mathbf{x} \in [0,1]^2}|f_{n}(\mathbf{x})-f_{m}(\mathbf{x})| \leq \epsilon$, for all $n,m \geq N$. Thus $\sup_{\mathbf{x}}|f_{n}(\mathbf{x})-f(\mathbf{x})| \leq \epsilon$. But this holds for all $\epsilon > 0$ and thus $(f_{n})$ converges to $f$. It follows that $\left(\mathcal{F}_{L},\lvert|\cdot\rvert|_{\infty}\right)$ is complete. Since $\mathcal{F}$ is totally bounded, it follows that $\mathcal{F}$ is compact and thus the statement follows.
    \end{proof}
\end{toappendix}

\begin{proof}[Proof of \Cref{thm:discrete}]

First, we can see that the space of bounded function $\mathcal{F}$ is convex and compact (the proof is deferred to \Cref{clm:compactness}) and follow from Glicksberg's theorem an equilibrium exists~\citep{glicksberg1952further}
$$\inf_{\vw_{\mathcal{F}} \in \Delta_{\mathcal{F}}}\sup_{\theta \in \ThetabciNM}\mathbf{E}_{f \sim \vw_{\mathcal{F}}}[R(f,\theta)] = \sup_{\vw_{\Theta} \in \Delta_{\ThetabciNM}}\inf_{f \in \mathcal{F}}\mathbf{E}_{\theta \sim \vw_{\Theta}}[R(f,\theta)].$$
Then it is easy to verify that the discrete setting satisfies the condition in \Cref{thm:finite}.

Let $M = \epsilon^{-9/2}NQ$ where $Q\in \mathbb{N}$ will be specified later. By \Cref{cor:finite_time} and \Cref{thm:finite}, \Cref{alg:1} on $\ThetabciNM$ can output a $\epsilon/2$-optimal aggregator $f^*$ in $\tilde{O}(\frac{N^5Q}{\epsilon^{2+9/2}}) = \tilde{O}(\frac{N^5Q}{\epsilon^{13/2}})$, so that
    $$ R(f^*, \ThetabciNM)\le \inf_{f\in \mathcal{F}} R(f, \ThetabciNM)+\epsilon/2$$
    On the other hand, by \Cref{lem:dtv_smooth,lem:dtv_covernum} for any $f$
    $$R(f, \ThetabciNM)\le R(f, \ThetabciN)+105\cdot \left(\frac{6N}{M}\right)^{2/9} \le R(f, \ThetabciN)+157\epsilon{Q}^{-2/9}$$
    Thus, combing the above two, because $\ThetabciNM\subset \ThetabciN$ we have
    $$ R(f^*, \ThetabciNM)\le \inf_{f\in \mathcal{F}} R(f,  \ThetabciN)+\epsilon/2+157\epsilon{Q}^{-2/9}\le \inf_{f\in \mathcal{F}} R(f,  \ThetabciN)+\epsilon=\inf_{f\in \mathcal{F}} R(f,  \ThetaciN)+\epsilon$$
    if $Q$ is large enough (e.g., larger than $(2\cdot 157)^{4.5}\approx 172259592827$).
\end{proof}






  

\section{Two Conditionally Independent Agents: Lipschitz Aggregators}
In this section, we will solve the Lipschitz aggregators setting of the robust forecast aggregation problem with a collection of Lipschitz aggregators $\mathcal{F}_L$,\[\inf_{f\in \mathcal{F}_L}\sup_{\theta\in\Thetaci}R(f,\theta).\]The main challenge to apply the online learning framework to the Lipschitz aggregators setting is to show the covering property. We show that the information structures with discrete reports is a good earth-moving distance covering for the class of continuous information structures. \fang{Why do we recall condition Independence here}The following theorem shows our results in this setting.
\begin{theorem}[Lipschitz Aggregators]\label{thm:lipschitz}
For any $0 < \epsilon < 1$ and $L > 0$, \Cref{alg:2} is an FPTAS with running time $O\left(\frac{L^{77}}{\epsilon^{79}}\log \left(\frac{L^{13}}{\epsilon^{14}}\right)\right)$ by taking $N=,M=$\fang{by taking $N$ $M = ...$} that finds an $\epsilon$-optimal $L$-Lipschitz aggregator over information structures $\Thetaci$.
\end{theorem}



\begin{algorithm}[!ht]
\caption{Online Learning For Lipschitz Aggregators}
\KwIn{approximation parameter $\epsilon$, discretization parameters $N,M$\fang{Should they depend on $\epsilon$ and $L$?}, and a class of Lipschitz aggregators $\mathcal{F}_L$ with $L>0$}
\KwOut{Optimal Lipschitz aggregator $f^*\in\mathcal{F}_L$}
Initialize the binary signal information structure set $\ThetabciNM$.\\ 
Initialize the uniform policy $\vw_{\ThetabciNM}^1\in\Delta_{\ThetabciNM}$.\\
Set $T=\lceil\frac{25\ln n}{\epsilon^2}\rceil$ where $n=|\ThetabciNM|$.\\
\For{$t=1\ \mathbf{to}\ T$} {
Calculate the $\epsilon/5$-best $L$-Lipschitz response $f^t\in\mathcal{F}_L$ to $\vw_{\ThetabciNM}^t$ such that for any $f\in \mathcal{F}_L$\\
\[\E_{\theta\sim \vw_{\ThetabciNM}^t}[R(f^t,\theta)]\le \E_{\theta\sim \vw_{\ThetabciNM}^t}[R(f,\theta)]+\epsilon/5.\]
Calculate the reward $\vu^t$ when the aggregator is $f^t$\\
\[
\forall \theta\in \ThetabciNM, u^t_{\theta}=R(f^t,\theta).
\]
Use weight update algorithm to calculate the policy $\vw_{\ThetabciNM}^{t+1}\in\Delta_{\ThetabciNM}$
\[
w^{t+1}_{\theta}=w^{t}_{\theta} \exp(-\eta u_{\theta}^t)/Z_t
\]
where $Z_t$ is a normalization factor.
}
$f^*=\frac{1}{T}\sum_{t=1}^T f^t$
\label{alg:2}
\end{algorithm}

\paragraph{Proof Sketch} With the dimension reduction results (similar to \cref{lem:reduction} and already proved in \citet{doi:10.1073/pnas.1813934115}), it is sufficient to consider the set of conditionally independent information structures with binary signals, $\Thetabci$. Similar to \cref{thm:discrete}, We show that the discretization $\ThetabciNM$ is a sufficiently good representation of $\Thetabci$ under a weaker metric than TVD with appropriate choice of $N$ and $M$ specified below.  Then we run \Cref{alg:2} with input $\ThetabciNM$ and output a $L$-Lipschitz $\epsilon$-optimal aggregator $f^*$. 

More formally, we use $\ThetabciN$ as a bridge. As we show in the discrete setting, $\ThetabciNM$ is a sufficiently good representation of $\ThetabciN$, which implies $\inf_f R(f, \ThetabciNM)\approx \inf_f R(f, \ThetabciN)$ for any $f$. We then show that  $\ThetabciN$ is a good representation of $\Thetabci$, so that $\inf_f R(f, \ThetabciN)\approx \inf_f R(f, \Thetabci)$. Then by our reduction result, $\inf_{f} R(f,\Thetabci) = \inf_f R(f,\Thetaci)$. Hence we will have 
\[R(f^*, \ThetabciNM)\approx \inf_{f} R(f,\ThetabciNM)\approx \inf_{f} R(f,\ThetabciN) \approx \inf_{f} R(f,\Thetabci) = \inf_f R(f,\Thetaci).\]

Moreover, we show that we can output a near optimal $L$-Lipschitz aggregator by solving a quadratic optimization problem with linear constraints. So our algorithm has a polynomial running time. 

The main technical part is showing the second approximate equality. Here are multiple new challenges we need to deal with compared to the discrete reports setting.

\paragraph{New Challenges and Techniques} To address the challenges associated with the fully continuous $\Theta_{\text{ci}}$, we employ multiple new techniques.

\begin{description}
\item [TVD $\rightarrow$ EMD] Recall that the difference between $\theta$ and $\theta'$ is bounded by the difference between the marginal distributions over reports $\Xdis{\theta}$, $\Xdis{\theta'}$. In the discrete reports setting, we can construct $\theta$ and $\theta'$ with the same support of the pair reports, thus it is sufficient to use total variation distance to measure the difference between $\Xdis{\theta}$ and $\Xdis{\theta'}$. However, in fully continuous $\Thetaci$, $\theta$ and $\theta'$ can have different supports of the reports. TVD would be too restricted. In fact, no finite collection of information structures $\Theta'$ can be a good $(\epsilon, d_{TV})$-coverings of $\Thetabci$ with $\epsilon<1$. This is because we can always find some reports that do not appear in $\Theta'$ and if we pick an information structure $\theta\in\Thetabci$ which only contains those not appeared reports, the TVD between $\theta$ and information structures in $\Theta'$ is always 1. In such a case, we will use a relaxed distribution metric, the earth mover's distance $d_{EM}$. We can show that $\ThetabciNM$ is a good $d_{EM}$-covering by constructing a coupling to upper bound $d_{EM}$ in \Cref{sec:dem_smallcover}. 

\begin{definition}[Earth Mover's Distance \citep{Chatterjee2008}]
    Given two distributions $P$ and $Q$ on $[0,1]^2$, we introduce the earth mover's distance,
$$d_{EM}(P,Q) = \sup_{h:\|h\|_{Lip}\le 1}\E_P[h]-\E_Q[h].$$
\end{definition} By abusing the notation a little bit, we set $d_{EM}(\theta,\theta')=d_{EM}(\Xdis{\theta},\Xdis{\theta'})$ and call the corresponding $(\epsilon, d_{EM})$-covering the EMD-covering.
\item [Sensitivity of $g_{\mu}(\vx)$ Regarding Both $\mu$ and $\vx$] We then show the bounded Lipschitz constant of the regret function. Though we only consider Lipschitz aggregators, to show the bounded Lipschitz constant of the regret function, we still need to handle the sensitivity of $g_{\mu}(\vx)$. In contrast to the previous discrete report setting where $\vx$ is fixed, here we need to handle the sensitivity of $g_{\mu}(\vx)$ regarding both $\mu$ and $\vx$. Based on a delicate case-by-case analysis, we will carefully smooth the $g_{\mu}(\vx)$ first by trimming its sensitive parts and extending it to ensure it has bounded Lipschitz constant. We will show that in expectation, the modified version of $g_{\mu}(\vx)$ not only has bounded Lipschitz constant but also has a similar expected loss as $g_{\mu}(\vx)$. Therefore, we can replace $g_{\mu}(\vx)$ with the modified version. 
\item [EMD-covering] To show that $\ThetabciNM$ is a good $d_{EM}$-covering of $\Thetabci$, we first prove that $\ThetabciNM$ is a good $d_{EM}$-covering of $\ThetaciN$. Due to the triangle inequality, it is left to show that $\ThetabciN$ is a good $d_{EM}$-covering of $\Thetabci$. For each $\theta\in\Thetaci$, we will construct a proper $\theta'\in \ThetabciN$ such that $d_{EM}(\Xdis{\theta},\Xdis{\theta'})$ is small. Finally, we will use the dual form of EMD and construct a coupling between $\Xdis{\theta},\Xdis{\theta'}$ to upper bound $d_{EM}(\Xdis{\theta},\Xdis{\theta'})$.

\item [Best Response Aggregator $\rightarrow$ Best Response Lipschitz Aggregator] Without any restriction, the best response aggregator has an explicit formula and thus is efficient to compute. We show that it is still efficient to compute the optimal Lipschitz aggregator by proving that it is a convex optimization problem with an efficient separation oracle. 
\end{description}

\subsection{Regret is Insensitive Regarding EMD}\label{sec:dem_smooth}

We analyze the sensitivity of the regret function regarding the earth mover's distance. 
Recall that $R(f,\theta) = \E_\theta[\ell(f(\vx),\omega)]-\E_\theta[\ell(g_{\mu}(\vx),\omega)]$. We will analyze the sensitivity of the first term $\E_\theta[\ell(f(\vx),\omega)]$, and then analyze the sensitivity of the second term $\E_\theta[\ell(g_{\mu}(\vx)]$. Regarding $\E_\theta[\ell(f(\vx),\omega)]$, as $f$ is Lipschitz, we can show when $d_{EM}(\Xdis{\theta}, \Xdis{\theta'})$ is small, $\E_\theta[\ell(f(\vx),\omega)]\approx \E_{\theta'}[\ell(f(\vx),\omega)]$. 

Regarding $\E_\theta[\ell(g_{\mu}(\vx),\omega)]$, we require a more delicate analysis than the previous discrete reports setting. In the previous setting, close $\theta$ and $\theta'$ have the same support of the reports, thus we only need to carefully handle the sensitivity of $g_{\mu}(\vx)$ regarding $\mu$. Here in the general $\Thetabci$, the reports are in a continuous set. Close $\theta$ and $\theta'$ can have different supports of the pair reports. Therefore, in this setting, we need to additionally handle the sensitivity of $g_{\mu}(\vx)$ regarding $\vx$ as well. We illustrate the sensitivity in \Cref{fig:sensitive_prediction}. 


\paragraph{Smoothing $g_{\mu}(\vx)$} We use the following steps to smooth $g_{\mu}(\vx)$ and obtain $\tilde{g}_{\mu}(\vx)$.
\begin{description}
\item [Trimming] We initially trim the sensitive parts of $g_{\mu}(\vx)$. \Cref{lem:ext2} demonstrates that the occurrence of very disagreeing reports is highly unlikely within conditionally independent information structures. Considering $\epsilon_1$ and $\epsilon_2$ as values in the range of $(0, 0.5]$, \Cref{lem:lip_omni} indicates that $g_{\mu}(\vx)$ are nearly Lipschitz with a high probability unless the reports are in disagreement (event $B(\epsilon_1)$), or if the reports are far from the prior (event $C_\mu(\epsilon_2)$):
\begin{equation}\label{eq:bad_area}
    \begin{aligned}B(\epsilon_1):=& \left\{\vx\in [0,1]^2: |x_1-x_2|> 1-\epsilon_1\right\}\\
    C_{\mu}(\epsilon_2):=&\begin{cases}
        &\left\{\vx\in [0,1]^2: x_1, x_2> \frac{\mu}{\epsilon_2}
\right\}\text{ if } \mu\le 1/2\text{, or }\\
&\left\{\vx\in [0,1]^2: x_1, x_2> 1-\frac{1-\mu}{\epsilon_2}\right\}\text{ if } \mu> 1/2.
    \end{cases}\end{aligned}
\end{equation}
and the remaining area is $A_\mu(\epsilon_1, \epsilon_2) = [0,1]^2\setminus \left(B(\epsilon_1)\cup C_\mu(\epsilon_2)\right)$.
\item [Extending] We extend the remaining parts of $g_{\mu}(\vx)$ to $\tilde{g}_{\mu}(\vx)$ to ensure $\tilde{g}_{\mu}(\vx)$ has bounded Lipschitz constant (\Cref{lem:lip_extension}), and guarantee that $\E_{\theta}[|g_\mu(\vx)-\tilde{g}_{\mu'}(\vx)|]$ is small (\Cref{lem:prior_shift_lip}).
\end{description}

Finally, we combine the above results and show that when $d_{EM}(\Xdis{\theta}, \Xdis{\theta'})\approx 0$, $\E_\theta[\ell(g_{\mu}(\vx),\omega)]\approx \E_{\theta'}[\ell(g_{\mu'}(\vx),\omega)]$. Combined with the result $\E_\theta[\ell(f(\vx),\omega)]\approx \E_{\theta'}[\ell(f(\vx),\omega)]$, we can show \Cref{lem:dem_smooth}. The proofs are deferred to \Cref{prf:dem_smooth}.

\begin{figure}
\centering
\begin{subfigure}[h]{0.40\textwidth}
\centering
\includegraphics[height=.8\textwidth]{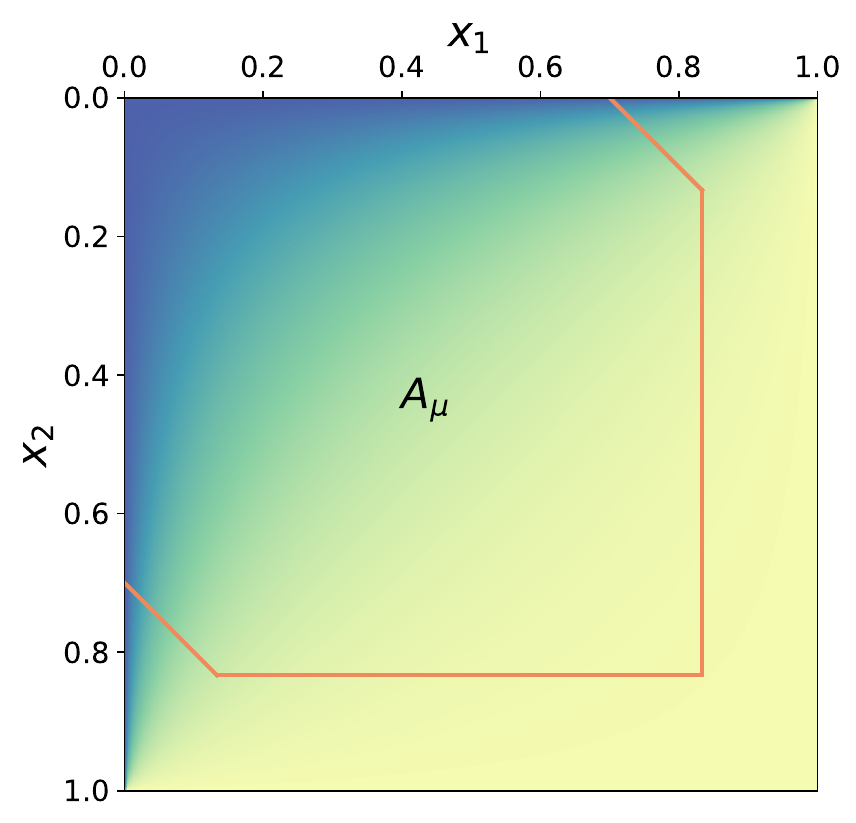}
\caption{Heatmap of $g_\mu(\vx)$, $\mu = 0.1$\fang{Can we have a smaller value, e.g., $\mu = 0.001$.}}
\end{subfigure}
\hfill
\begin{subfigure}[h]{0.50\textwidth}
\centering
\includegraphics[height=.64\textwidth]{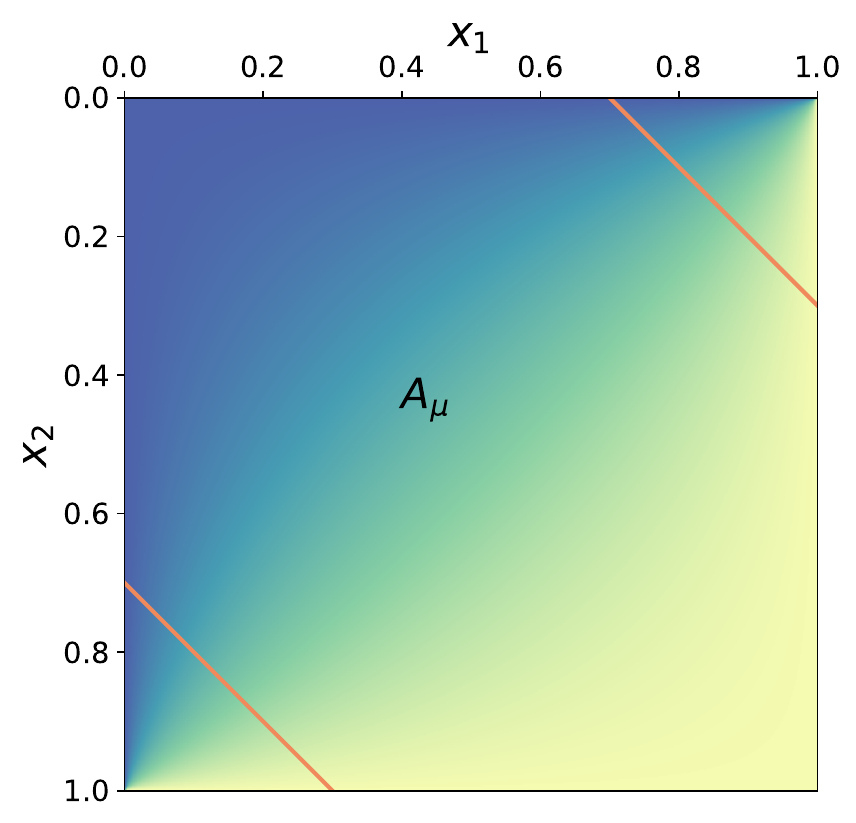}
\includegraphics[height=.576\textwidth]{figures/cbar1.pdf}
\caption{Heatmap of $g_\mu(\vx)$, $\mu = 0.4$}
\end{subfigure}%
\caption{\textbf{Sensitivity of the omniscient aggregator regarding $\vx$: $g_\mu(\vx)$ vs. $\vx$.} We trim the sensitive area \fang{shaded or dotted?} and extend the insensitive area $A_\mu$.}\label{fig:sensitive_prediction}
\end{figure}


\begin{propositionrep}\label{lem:dem_smooth}
For any $L>0$ $\epsilon_L>0$, aggregator $f:[0,1]^2\to [0,1]$, and $\theta, \theta'\in \Thetaci$ with $\epsilon_{EM} = d_{EM}\left(\Xdis{\theta}, \Xdis{\theta'}\right)<2^{-7}$, if $f$ is $L$-Lipschitz on $[0,1]^2\setminus B(\epsilon_L)$, then 
$$R(f, \theta)\le R(f, \theta')+134\epsilon_{EM}^{1/7}+2L\epsilon_{EM}+4\epsilon_L.$$
\end{propositionrep}
\fang{we may only keep the proposition and figure and move the rest into appendix.}
\begin{toappendix}
    \label{prf:dem_smooth}
\end{toappendix}
\begin{proof}[Proof of \Cref{lem:dem_smooth}]
Given $\theta$ and $\theta'$, let $\epsilon_{EM} = d_{EM}(\Xdis{\theta}, \Xdis{\theta'})$, and $\epsilon = |\mu-\mu'|$ which is less than $\epsilon_{EM}$ by \Cref{lem:prior_bound}.  The proof of \Cref{lem:dem_smooth} bears a similar structure as \Cref{lem:dtv_smooth}.  However, to use the bounded $d_{EM}$ condition, we need to modify $g_{\mu'}(\vx)$ to a Lipschitz one.  
Formally, given $\epsilon$ and $\epsilon_{EM}$, we define $\beta, \gamma$ so that $\epsilon_1 = \epsilon^\beta = \epsilon_{EM}^{1/7}$ and $\epsilon_2 = \epsilon^\gamma = \epsilon_{EM}^{4/7}$. By \Cref{lem:lip_extension}, we can construct $\tilde{g}_{\mu'}(\vx)$ that agree with $g_{\mu'}(\vx)$ on $A_{\mu'}(\epsilon_1, \epsilon_2)$ and is $\frac{8}{\epsilon_1^2\epsilon_2}$-Lipschitz.  Additionally, because

We first consider $|1/2-\mu'|>|1/2-\mu|$.
\begin{align*}
    &R(f, \theta)-R(f, \theta')\\
    =& \E_\theta[\left(g_\mu(\vx)-f(\vx)\right)^2]-\E_{\theta'}[\left(g_{\mu'}(\vx)-f(\vx)\right)^2]\\
    \le& \E_\theta[\left(g_\mu(\vx)-f(\vx)\right)^2]-\E_{\theta'}[\left(\tilde{g}_{\mu'}-f(\vx)\right)^2]+\Pr_{\theta'}[\neg A_{\mu'}(\epsilon_1, \epsilon_2)]\tag{since $\|\left(g_\mu(\vx)-f(\vx)\right)^2-\left(\tilde{g}_{\mu'}-f(\vx)\right)^2\|_\infty\le 1$}\\
    \le& \E_\theta[\left(g_\mu(\vx)-f(\vx)\right)^2]-\E_{\theta'}[\left(\tilde{g}_{\mu'}-f(\vx)\right)^2]+4\epsilon_1+2\epsilon_2\tag{by \Cref{lem:lip_omni}}\\
    =& \left(\E_\theta[\left(g_\mu(\vx)-f(\vx)\right)^2]-\E_\theta[\left(\tilde{g}_{\mu'}-f(\vx)\right)^2]\right)+\left(\E_\theta[\left(\tilde{g}_{\mu'}-f(\vx)\right)^2]-\E_{\theta'}[\left(\tilde{g}_{\mu'}-f(\vx)\right)^2]\right)+4\epsilon_1+2\epsilon_2.
\end{align*}
For the first term, because $\|g_\mu(\vx)+\tilde{g}_{\mu'}-2f\|_\infty\le 2$, 
\[
\E_\theta[\left(g_\mu(\vx)-f(\vx)\right)^2]-\E_\theta[\left(\tilde{g}_{\mu'}-f(\vx)\right)^2] =  \E_\theta[\left(g_\mu(\vx)+\tilde{g}_{\mu'}-2f(\vx)\right)\left(g_\mu(\vx)-\tilde{g}_{\mu'}\right)]\le 2\E_\theta[|g_\mu(\vx)-\tilde{g}_{\mu'}|].
\]  Since $\epsilon\le \epsilon_{EM}<2^{-7}< 2^{-32/9}$, $0<\beta, \gamma\le 4/7<2/3$, and $\epsilon^\beta, \epsilon^\gamma<1/2$, we can apply \Cref{lem:prior_shift_lip} and get
\begin{equation}
\E_\theta[\left(g_\mu(\vx)-f(\vx)\right)^2]-\E_\theta[\left(\tilde{g}_{\mu'}-f(\vx)\right)^2]\le 8\epsilon_1+104\epsilon_2^{1/4} = 112\epsilon_{EM}^{1/7}.
\end{equation}
For the second term, because for all $\vx$ and $\vx'$
\begin{align*}
    &|(\tilde{g}_{\mu'}(\vx)-f(\vx))^2-(\tilde{g}_{\mu'}(\vx')-f(\vx'))^2|\\
    =& |(\tilde{g}_{\mu'}(\vx)-f(\vx))+(\tilde{g}_{\mu'}(\vx')-f(\vx'))|\cdot |(\tilde{g}_{\mu'}(\vx)-f(\vx))-(\tilde{g}_{\mu'}(\vx')-f(\vx'))|\\
    \le& 2|(\tilde{g}_{\mu'}(\vx)-f(\vx))-(\tilde{g}_{\mu'}(\vx')-f(\vx'))|\\
    \le& 2|\tilde{g}_{\mu'}(\vx)-\tilde{g}_{\mu'}(\vx')|+2|f(\vx)-f(\vx')|
\end{align*}
By the definition of earth mover's distance $d_{EM}(\Xdis{\theta}, \Xdis{\theta'})\le \epsilon_{EM}$ and $\|\tilde{g}_{\mu'}\|_{Lip}\le \frac{8}{\epsilon_{EM}^{6/7}}$, we have
$\E_\theta[2|\tilde{g}_{\mu'}(\vx)-\tilde{g}_{\mu'}(\vx')|]\le 16\epsilon_{EM}^{1/7}.$
Additionally $f$ is $L$-Lipschitz $[0,1]^2\setminus B(\epsilon_L)$, by \Cref{lem:ext2}
$\E_\theta[2|f(\vx)-f(\vx')|]\le 2L\epsilon_{EM}+2\Pr_\theta[B(\epsilon_L)]\le 2L\epsilon_{EM}+4\epsilon_L.$
Therefore, 
$$\E_\theta[\left(\tilde{g}_{\mu'}(\vx)-f(\vx)\right)^2]-\E_{\theta'}[\left(\tilde{g}_{\mu'}(\vx)-f(\vx)\right)^2]\le 16\epsilon_{EM}^{1/7}+2L\epsilon_{EM}+4\epsilon_L.$$
Combining these two gets $$R(f, \theta)-R(f, \theta')\le 128\epsilon_{EM}^{1/7}+2L\epsilon_{EM}+4\epsilon_1+2\epsilon_2+4\epsilon_L\le 134\epsilon_{EM}^{1/7}+2L\epsilon_{EM}+4\epsilon_L$$

\end{proof}

\begin{lemmarep}\label{lem:ext2}
For any $\theta\in \Thetaci$ and $\epsilon>0$, 
\begin{align}
    &\Pr_\theta[|x_1-x_2|\ge 1-\epsilon]\le \frac{2\epsilon}{1-\epsilon}\label{eq:ext_2}
\end{align}
\end{lemmarep}
\begin{proof}
We can let $\mu \in (0,1)$, or the inequality trivially holds.  
We first prove \cref{eq:ext_2}.  Suppose that a pair of signals $(s_1, s_2)$ induces disagreeing predictions $(x_1, x_2)$ with $x_1-x_2>1-\epsilon$.  Let $\alpha_1 = p_1(s_1), \alpha_0 = p_0(s_1), \beta_1 = q_1(s_2)$, and $\beta_0 = q_0(s_2)$.  Because $x_1-x_2>1-\epsilon$ and $x_2\in [0,1]$, we have $1-\epsilon<x_1 = \frac{\mu \alpha_1}{\mu \alpha_1+(1-\mu)\alpha_0}$, and with some arrangement
    \begin{equation}\label{eq:ext1}
        \alpha_0\le \frac{\mu}{1-\mu}\frac{\epsilon}{1-\epsilon}\alpha_1.
    \end{equation}
    Similarly 
    \begin{equation}\label{eq:ext2}
        \beta_1\le \frac{\epsilon}{1-\epsilon}\frac{1-\mu}{\mu}\beta_0.
    \end{equation}
    The probability of signal pairs $(s_1, s_2)$ is 
    \begin{align*}
        \Pr_\theta[s_1, s_2] =& \mu \alpha_1\beta_1+(1-\mu)\alpha_0\beta_0\\
        \le& \mu \alpha_1\frac{\epsilon}{1-\epsilon}\frac{1-\mu}{\mu}\beta_0+(1-\mu)\frac{\mu}{1-\mu}\frac{\epsilon}{1-\epsilon}\alpha_1\beta_0\tag{ by \cref{eq:ext1,eq:ext2}}\\
        =& \frac{\epsilon}{1-\epsilon}\alpha_1\beta_0.
    \end{align*}
    Therefore, summing all possible pair of signals that induces disagreeing predictions, we have
    $$\Pr_\theta[x_1-x_2>1-\epsilon]\le \frac{\epsilon}{1-\epsilon}$$
    which completes the proof by symmetry.
\end{proof}
The following lemma shows that $g_\mu(\vx)$ is mostly Lipschitz, and then we replace it with a Lipschitz function $\tilde{g}_\mu(\vx)$.
\begin{lemmarep}\label{lem:lip_omni}
Given $\epsilon_1, \epsilon_2\in (0,1/2]$ and $\mu\in [0,1]$, the function $g_\mu(\vx)(x_1, x_2) := \frac{(1-\mu)x_1x_2}{(1-\mu)x_1x_2+\mu(1-x_1)(1-x_2)}$ is $\frac{4}{\epsilon_1^2\epsilon_2}$-Lipschitz on $A_\mu(\epsilon_1, \epsilon_2)$ defined in \cref{eq:bad_area}.  Moreover,     $\Pr_\theta[A_{\mu}(\epsilon_1, \epsilon_2)]\ge 1-4\epsilon_1-2\epsilon_2.$
\end{lemmarep}
\begin{proof}[Proof of \Cref{lem:lip_omni}]
    First, because $|x_1-x_2|\le 1-\epsilon_1$, we can bound $x_1x_2+(1-x_1)(1-x_2)$ 
    Let $U:= x_1x_2+(1-x_1)(1-x_2)$ and $V:= x_1x_2-(1-x_1)(1-x_2)$.  
    \begin{align*}
        V^2-2U+1=&x_1^2x_2^2+(1-x_1)^2(1-x_2)^2-2x_1(1-x_1)x_2(1-x_2)-2x_1x_2-2(1-x_1)(1-x_2)+1\\
        =& x_1^2x_2^2+[1-2x_1+x_1^2-2x_2+4x_1x_2-2x_1^2x_2+x_2^2-2x_1x_2^2+x_1^2x_2^2]\\
        &+[2x_1x_2-2x_1^2x_2-2x_1x_2^2+2x_1^2x_2^2]-2x_1x_2-2[1-x_1-x_2+x_1x_2]+1\\
        =& x_1^2+x_2^2-2x_1x_2 = (x_1-x_2)^2
    \end{align*}
    Therefore, if $|x_1-x_2|\le 1-\epsilon_1$, we have
    \begin{equation}\label{eq:lip_omni1}
        x_1x_2+(1-x_1)(1-x_2) = U = \frac{V^2+1-(x_1-x_2)^2}{2}\ge \epsilon_1-\frac{1}{2}\epsilon_1^2\ge \frac{1}{2}\epsilon_1
    \end{equation}
    On the other hand, by symmetry, we consider $\mu\le 1/2$.  Then we can use the second condition $x_1, x_2\le \frac{\mu}{\epsilon_2}$ and have
    \begin{equation}\label{eq:lip_omni2}
        x_1(1-x_1), x_2(1-x_2)\le \frac{\mu}{\epsilon_2}
    \end{equation}
The partial derivative of $g_\mu(\vx)$ is 
$\nabla_\vx g_\mu(\vx)(x_1,x_2) =\left(\frac{\mu(1-\mu)x_2(1-x_2)}{(\mu(1-x_1)(1-x_2)+(1-\mu)x_1x_2)^2}, \frac{\mu(1-\mu)x_1(1-x_1)}{(\mu(1-x_1)(1-x_2)+(1-\mu)x_1x_2)^2}\right)$.  We can bound each coordinate as follows:
\begin{align*}
    \frac{\mu(1-\mu)x_2(1-x_2)}{(\mu(1-x_1)(1-x_2)+(1-\mu)x_1x_2)^2}\le& \frac{\mu^2(1-\mu)}{\epsilon_2(\mu(1-x_1)(1-x_2)+(1-\mu)x_1x_2)^2}\tag{by \cref{eq:lip_omni2}}\\
    \le& \frac{\mu^2}{\epsilon_2(\mu(x_1x_2+(1-x_1)(1-x_2)))^2}\tag{$\mu\le 1/2$ and $(1-\mu)\le 1$}\\
    \le& \frac{4}{\epsilon_1^2\epsilon_2}\tag{by \cref{eq:lip_omni1}}
\end{align*}

Finally, by symmetry, we can assume $\mu\le 1/2$.  By \Cref{lem:ext1,lem:ext2},  $\Pr_\theta[x_1\ge \frac{\mu}{\epsilon_2}]\le \epsilon_2$ and $\Pr_\theta[|x_1-x_2|>1-\epsilon_1]\le \frac{2\epsilon_1}{1-\epsilon_1}\le 4\epsilon_1$.  Thus, by union bound, $\Pr_\theta[A_\mu(\epsilon_1, \epsilon_2)]\ge 1-4\epsilon_1-2\epsilon_2$
\end{proof}

\begin{lemmarep}[Bounded Lipschitz extension]\label{lem:lip_extension}
    Given $\epsilon_1, \epsilon_2\in (0,1/2)$ and $\mu\in (0,1)$, there exists a $\frac{8}{\epsilon_1^2\epsilon_2}$-Lipschitz function $\tilde{g}_\mu: [0,1]^2\to [0,1]$ and $\tilde{g}_{\mu}(\vx) = g_\mu(\vx)$ for all $\vx\in {A_{\mu}(\epsilon_1, \epsilon_2)}$.
\end{lemmarep}
Note that we need extend $g_\mu$ to an Lipschitz and bounded function, and we can not directly apply Kirszbraun theorem~\citep{kirszbraun1934zusammenziehende} whose extension is not necessarily bounded.  
\begin{proof}
By symmetry, we consider $\mu\le 1/2$ so that $$A_\mu(\epsilon_1, \epsilon_2) = \left\{(x_1, x_2): 0\le x_1< \frac{\mu}{\epsilon_2},0\le x_2< \frac{\mu}{\epsilon_2},|x_1-x_2|< 1-\epsilon_1\right\}$$.  In the rest of the proof we write $A_\mu = A_\mu(\epsilon_1, \epsilon_2)$, and $\bar{x}:=\inf(1, \mu/\epsilon_2)$ which is the maximum value of $x_1$ or $x_2$ in $A_\mu$. We want to define $\tilde{g}_\mu$ on set  $[0,1]^2\setminus A_\mu$ which can be written as the union of the following three sets 
\begin{align*}
    D_1 =& \{\vx: x_2\le \bar{x}\wedge(x_1> \bar{x} \vee x_1-x_2> 1-\epsilon_1)\},\\
    D_2 =& \{\vx: x_1\le\bar{x}\wedge(x_2> \bar{x} \vee x_2-x_1> 1-\epsilon_1)\},\\
    E =& \{\vx: x_1> \bar{x}\text{, and }x_2>\bar{x}\}.
\end{align*}
Intuitively, $D_1$ is the collection of points that is on the right-hand side of $A_\mu$, $D_2$ is on the top of $A_\mu$, and $E$ is not connected to $A_\mu$.  Let $\hat{x}_1$ and $\hat{x}_2$ map each points to the rightmost or topmost in $A_\mu$ respectively: Specifically, we define
\begin{equation}\hat{x}_1(x_2) = \argmax_{x_1'}\{(x_1', x_2)\in A_\mu\} = \begin{cases}
    \bar{x}&\text{ if $x_2\ge \bar{x}-(1-\epsilon_1)$}\\
    1-\epsilon_1+x_2&\text{ if $x_2< \bar{x}-(1-\epsilon_1)$}
\end{cases},\label{eq:lip_extension_x1}
\end{equation}
for all $x_2\in [0,1]$, and 
\begin{equation}
    \hat{x}_2(x_1) = \argmax_{x_2'}\{(x_1, x_2')\in A_\mu\} = \begin{cases}
    \bar{x}&\text{ if $x_1\ge \bar{x}-(1-\epsilon_1)$}\\
    1-\epsilon_1+x_1&\text{ if $x_1< \bar{x}-(1-\epsilon_1)$}
\end{cases}\label{eq:lip_extension_x2}
\end{equation}
for all $x_1\in [0,1]$. \Cref{fig:area} shows an example of these areas and symbols we define.

\begin{figure}
\centering
\includegraphics[width=0.5\textwidth]{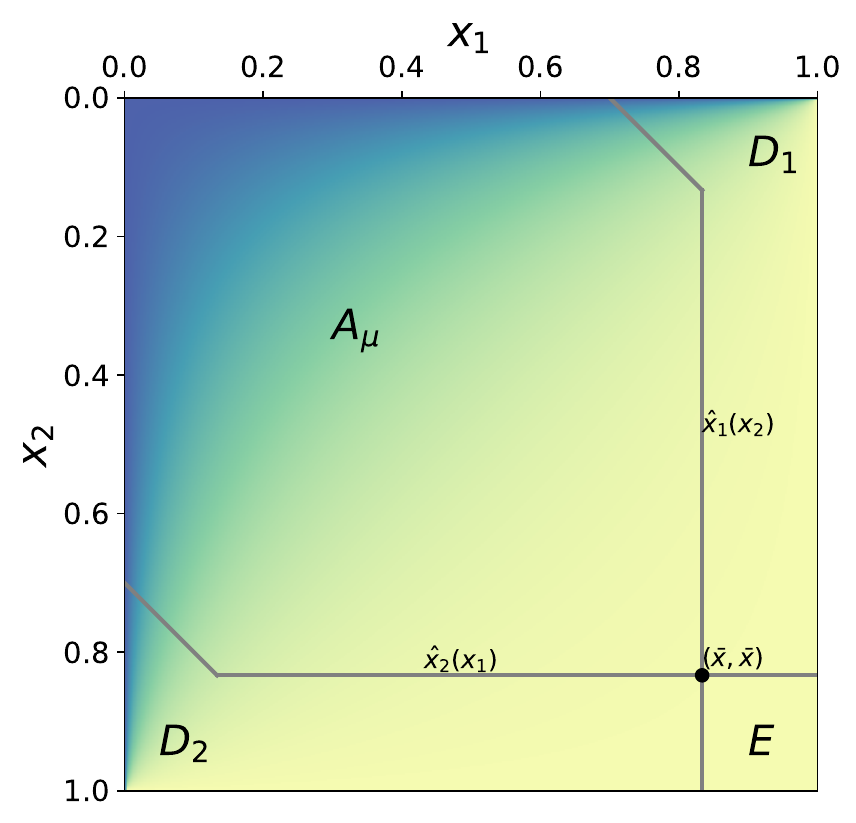}
\caption{An example of area division.}\label{fig:area}
\end{figure}

Finally, we can define
$$\tilde{g}_\mu(x_1, x_2) = \begin{cases}g_\mu(\hat{x}_1(x_2), x_2)&\text{if $(x_1, x_2)\in D_1$}\\
g_\mu(x_1, \hat{x}_2(x_1))&\text{if $(x_1, x_2)\in D_2$}\\
g_\mu(\bar{x}, \bar{x})=g_\mu(\hat{x}_1(x_2), \hat{x}_2(x_1))&\text{if $(x_1, x_2)\in E$}\\
g_\mu(x_1, x_2)&\text{otherwise}\end{cases}$$
It is easy to see that $\tilde{g}_\mu(\vx)\in [0,1]$ for all $\vx$.  By \Cref{lem:lip_omni}, $g_\mu(\vx)$ is $L$-Lipschitz on $A_\mu$ with $L\le \frac{4}{\epsilon_1^2\epsilon_2}$.  We will show $\tilde{g}_\mu$ is $2L$-Lipschitz.  First, for all $x_1< x_1'$ and $x_2$,  let $\vx = (x_1, x_2)$ and $\vx' = (x_1', x_2)$.   we can show the following statement is true. 
\begin{equation}\label{eq:lip_extension1}
    |\tilde{g}_\mu(\vx')-\tilde{g}_\mu(\vx)|\le 2L\|\vx'-\vx\|_1 = 2L(x_1'-x_1).
\end{equation}
\begin{enumerate}
    \item If $\vx, \vx'\in A_\mu$, \cref{eq:lip_extension1} holds because $\tilde{g}_\mu(\vx) = g_\mu(\vx)$, $\tilde{g}_\mu(\vx') = g_\mu(\vx')$ and  $\|g_\mu(\vx)-g_\mu(\vx')\|\le L\|\vx'-\vx\|_1$.
    \item If $\vx, \vx'\in D_1$ or $\vx, \vx'\in E$, \cref{eq:lip_extension1} holds, because $\tilde{g}_\mu(\vx) = \tilde{g}_\mu(\vx')$.
    \item If $\vx\in D_2$ and $\vx'\in D_2$ or $E$.  
    We observe that $|\hat{x}_2(x_1)-\hat{x}_2(x_1')|\le |x_1-x_1'|$ since the gradient of $\hat{x}_2(x_1)$ is less than 1. So $|\tilde{g}_\mu(\vx')-\tilde{g}_\mu(\vx)| \le L(|x_1'-x_1|+|\hat{x}_2(x_1)-\hat{x}_2(x_1')|)\le 2L\|\vx'-\vx\|_1$, which proves \cref{eq:lip_extension1}.
    \item If $\vx\in D_2$, $\vx'\in A_\mu$ or $D_1$.  Because $\vx\in D_2$, $x_2\ge \inf(\bar{x}, x_1+(1-\epsilon_1)) = \hat{x}_2(x_1)\ge x_1+(1-\epsilon_1)$.  Because $\vx'\in A_\mu$ or $D_1$, $x_2-x_1'\le(1-\epsilon_1)$.  Combining above two inequalities, we have $|x_2-\hat{x}_2(x_1)| = x_2-\hat{x}_2(x_1)\le x_1'+(1-\epsilon_1)-(x_1+(1-\epsilon_1)) = x_1'-x_1$, and
    \begin{align*}
        |\tilde{g}_\mu(\vx')-\tilde{g}_\mu(\vx)| =& |g_\mu(x_1', x_2)-g_\mu(x_1, \hat{x}_2(x_1))|\\
        \le& L(|x_1'-x_1|+|x_2-\hat{x}_2(x_1)|)\\
        \le& 2L(x_1'-x_1)
    \end{align*}
    that proves \cref{eq:lip_extension1}. 
    \item If $\vx\in A_\mu$, $\vx'\in D_1$, by the definition of $D_1$, we have $x_1'\ge \hat{x}_1(x_2)$. Thus, $|\tilde{g}_\mu(\vx')-\tilde{g}_\mu(\vx)| = |g_\mu(\hat{x}_1(x_2), x_2)-g_\mu(x_1, x_2)|\le L|\hat{x}(x_2)-x_1|\le L|x_1'-x_1|$.
    \item Finally, because $\vx$ and $\vx'$ have the same second coordinate, it is impossible to have $\vx\in A_\mu$ and $\vx'\in E$, or $\vx\in D_1$ and $\vx'\in E$.
\end{enumerate}
The above argument exhaust all possible combination of a pair of points in four sets, and proves \cref{eq:lip_extension1}.  By symmetry, for all $x_1, x_2< x_2'$ in $[0,1]$,
$$|\tilde{g}_\mu(x_1, x_2')-\tilde{g}_\mu(x_1, x_2')|\le 2L(x_2'-x_2).$$
Combining these two, $\|\tilde{g}_\mu\|_{Lip}\le 2L\le \frac{8}{\epsilon_1^2\epsilon_2}$.
\end{proof}

\begin{claimrep}\label{lem:prior_shift_lip}
     Given for all $0< \beta, \gamma\le 2/3$ and $\theta, \theta'$ with $|\mu-\mu'| = \epsilon<1/6$ and $|1/2-\mu|>|1/2-\mu'|$, 
     if $\epsilon_1 = \epsilon^{\beta}$, $\epsilon_2 = \epsilon^{\gamma}$ satisfying $\epsilon_1, \epsilon_2<1/2$ and $\epsilon<2^{-32/9}$, 
     $$\E_{\theta}[|g_\mu(\vx)-\tilde{g}_{\mu'}(\vx)|] \le 4\epsilon^\beta+52\epsilon^{\frac{\gamma}{2}(1-\frac{\gamma}{2})} \le 4\epsilon_1+52\epsilon_2^{1/4}$$
     where $\tilde{g}_{\mu'}$ is the Lipschitz extension of $g_{\mu'}(\vx)$ from $A_{\mu'}(\epsilon_1, \epsilon_2)$.
\end{claimrep}
\begin{proof}
 By symmetry, we can assume $\mu\le 1/2$.  Let $\sigma = \frac{\mu}{1-\mu}$, and $\sigma' = \frac{\mu'}{1-\mu'}$.   For any $\vx \in A_{\mu'}(\epsilon_1, \epsilon_2)$, let $\xi = \frac{x_1x_2}{\bar{x}_1\bar{x}_2}\in \R_{\ge 0}$, $$|g_\mu(\vx)-\tilde{g}_{\mu'}(\vx)| = |g_\mu(\vx)-g_{\mu'}(\vx)|=\frac{|\sigma'-\sigma|\xi}{\xi^2+(\sigma+\sigma')\xi+\sigma\sigma'}.$$
To prove the inequality, we consider two cases: $\mu\ge \epsilon^{1-\gamma/2}$ or $\mu<\epsilon^{1-\gamma/2}$.  

\paragraph{Case 1:} $\mu\ge \epsilon^{1-\gamma/2}$, by the argument identical to the first part of \Cref{lem:prior_shift}'s proof, we have for all $\vx\in A_{\mu'}(\epsilon_1, \epsilon_2)$, $|g_\mu(\vx)-\tilde{g}_{\mu'}(\vx)| = |g_\mu(\vx)-g_{\mu'}(\vx)| = 4\epsilon^{\gamma/2}$, so by \cref{eq:bad_area}
$$\E_\theta[|g_\mu(\vx)-\tilde{g}_{\mu'}(\vx)|] = 4\epsilon^{\gamma/2}+1-\Pr_\theta[A_{\mu'}(\epsilon_1, \epsilon_2)]\le 4\epsilon^{\gamma/2}+\Pr_\theta[B(\epsilon_1)]+\Pr_\theta[C_{\mu'}(\epsilon_2)].$$
By \Cref{lem:ext2}, $\Pr_\theta[B(\epsilon_1)]\le \frac{2\epsilon_1}{1-\epsilon_1}\le 4\epsilon_1$ when $\epsilon_1\le 1/2$. On the other hand, if $\mu'\le 1/2$, we have $\mu\le \mu'$ and by \Cref{lem:ext1} $\Pr_\theta[C_{\mu'}(\epsilon_2)] \le 2\Pr_\theta[x_1\ge \frac{\mu'}{\epsilon_2}]\le 2\frac{\mu}{\mu'}\epsilon_2\le 2\epsilon_2$.  If $\mu'>1/2$, $\Pr_\theta[C_{\mu'}(\epsilon_2)] \le 2\Pr_\theta[x_1\le 1-\frac{1-\mu'}{\epsilon_2}]\le 2\frac{1-\mu}{1-\mu'}\epsilon_2\le 2\frac{1/2}{1/2-\epsilon}\epsilon_2\le 3\epsilon_2$ since $\epsilon<2^{-32/9}<1/6$. Therefore,
\begin{equation}\label{eq:prior_shift_lip1}
    \E_\theta[|g_\mu(\vx)-\tilde{g}_{\mu'}(\vx)|]\le 4\epsilon^{\gamma/2}+4\epsilon_1+3\epsilon_2 = 4\epsilon^{\beta}+7\epsilon^{\gamma/2}.
\end{equation}
\paragraph{Case 2:} $\mu\le \epsilon^{1-\gamma/2}$, because $|1/2-\mu|>|1/2-\mu'|$, $\mu' = \mu+\epsilon$ is greater than $\mu$ and $\epsilon$.  Consider $\vx$ with $x_1, x_2\le \mu^{1-\gamma/2}$.  Note that 
$x_1, x_2 \le \mu^{1-\gamma/2}\le (\mu')^{1-\gamma/2}\le \frac{\mu'}{{\epsilon}_2}$, because $\mu\le \mu'$ and $(\mu')^{\gamma/2}\ge \epsilon^{\gamma/2}\ge \epsilon_2$. 
 Additionally, $|x_1-x_2|\le 2\mu^{1-\gamma/2}\le 2{\epsilon}^{(1-\gamma/2)^2}\le 2{\epsilon}^{(3/4)^2}\le 1/2\le 1-\epsilon_1$ since $\epsilon<2^{-32/9}$ and $\epsilon_1\le 1/2$.  Therefore, 
 $$\{\vx: x_1, x_2\le \mu^{1-\gamma/2}\}\subset A_{\mu'}(\epsilon_1, \epsilon_2),$$
 and $|g_\mu(\vx)-\tilde{g}_{\mu'}(\vx)| = |g_\mu(\vx)-g_{\mu'}(\vx)|$.  We can use the same argument as \Cref{lem:prior_shift}, and have $|g_\mu(\vx)-\tilde{g}_{\mu'}(\vx)| = |g_\mu(\vx)-g_{\mu'}(\vx)| \le 24\sigma^{1-\gamma}$, and by union bound with \Cref{lem:ext1}, we have
\begin{equation}\label{eq:prior_shift_lip2}
    \E_\theta |g_\mu(\vx)-\tilde{g}_{\mu'}(\vx)|\le 2\sigma^{\gamma/2}\cdot \|g_\mu(\vx)-\tilde{g}_{\mu'}\|_\infty+24\sigma^{1-\gamma}\le26\sigma^{\gamma/2}\le 52\epsilon^{(1-\gamma/2)\gamma/2}
\end{equation}
Combining \cref{eq:prior_shift_lip1,eq:prior_shift_lip2} completes the proof.
\end{proof}

\subsection{$\ThetabciNM$ is a Small EMD Covering of $\Thetabci$}\label{sec:dem_smallcover}

In this section, we show $\ThetabciNM$ is a good EMD covering of $\Thetabci$. 
\begin{proposition}[Small Cover]\label{l:smallcoverlipschitz}
For all $M, N\in\mathbb{N}$ with $M>N$, $\ThetabciNM$ is a $\left(\frac{12N}{M}+\frac{8}{\sqrt{N}}+\frac{4}{N}, d_{EM}\right)$-covering of $\Thetabci$.
\end{proposition}

To prove \Cref{l:smallcoverlipschitz}, we first note that because the diameter of $[0,1]^2$ is bounded, the earth mover's distance of any pair of distributions on $[0,1]^2$ can be upper bounded by their total variation distance. With \Cref{lem:dtv_covernum}, $\ThetabciNM$ is also a good $d_{EM}$-covering of $\ThetabciN$.  Therefore, it is sufficient to prove $\ThetabciN$ is a good $d_{EM}$-covering of $\Thetabci$ as proved in \Cref{lem:dem_covernum}, which is proved in \Cref{lem:dem_covernum}. 

\begin{proof}[Proof of \Cref{l:smallcoverlipschitz}]
    By \Cref{lem:dtv_covernum}, $\ThetabciNM$ is a $(\frac{6N}{M}, d_{TV})$-covering of $\ThetabciN$.  Because the diameter of $[0,1]^2$ is $2$ so that $\|\vx-\vy\|_1\le 2\cdot\mathbf{1}[\vx\neq \vy]$ for all $\vx, \vy\in [0,1]^2$, by \cref{eq:dual} we have $d_{EM}(\Xdis{\theta}, \Xdis{\theta'})\le 2d_{TV}(\Xdis{\theta}, \Xdis{\theta'})$ for all $\theta, \theta'$.  Therefore, $\ThetabciNM$ is a $(\frac{12 N}{M}, d_{EM})$-covering of $\bar{\Theta}_{N'}$.
    
    On the other hand, by \Cref{lem:dem_covernum} $\ThetabciN$ is a $\left(\frac{8}{\sqrt{N}}+\frac{4}{N}, d_{EM}\right)$-covering of $\Thetabci$.  Because $d_{EM}$ is a metric~\citep{dudley2018real}, by triangle inequality, $\ThetabciNM$ is a $\left(\frac{12N}{M}+\frac{8}{\sqrt{N}}+\frac{4}{N}, d_{EM}\right)$
    -covering of $\Thetabci$
\end{proof}

\begin{lemmarep}\label{lem:dem_covernum}
For all $N\in \mathbb{N}$, $\ThetabciN$ is a $\left(\frac{8}{\sqrt{N}}+\frac{4}{N}, d_{EM}\right)$-covering of $\Thetabci$.
\end{lemmarep}

To prove \Cref{lem:dem_covernum}, for each $\theta\in \Thetabci$, it is sufficient to find a $\theta'\in \ThetabciN$ with small $d_{EM}(\theta,\theta')$. Recall that the earth mover's distance is defined as the supreme of $\E_P[h]-\E_Q[h]$ over all $h$ which is 1-Lipschitz. However, a specific $\E_P[h]-\E_Q[h]$ only provides a lower bound of EMD, while we want an upper bound. Thus, we will consider the dual form, which requires coupling techniques.

\paragraph{Coupling and Dual Form of EMD/TVD} We introduce coupling here. Readers familiar with these concepts can jump to \Cref{lem:dem_covernum}.
For any two distributions $P$ and $Q$ on $[0,1]^2$ with a metric $d:[0,1]^2\times [0,1]^2 \to \R$, let $\Pi(P,Q)$ denote the set of all joint distributions on $S\times S$\fang{$[0,1]^2\times[0,1]^2$?} with marginals $P$ and $Q$. The joint distribution $\pi\in \Pi(P,Q)$ is also known as \emph{coupling} or \emph{transportation plan} between $P$ and $Q$. 

The \emph{Wasserstein distance} between $P$ and $Q$ is $\inf_{\pi \in \Pi(P,Q)}\E_{(\vx, \vy)\sim \pi}[d(\vx, \vy)]$.  Given a real-valued function $f$ on $S$, $\|f\|_{L, d}:= \sup_{\vx\neq \vy\in S}|f(\vx)-f(\vy)|/d(\vx, \vy)$.  Note that when $d = \|\cdot\|_1$ is $1$-norm, $\|f\|_{L,d} = \|f\|_{Lip}$ , and when $d$ is the discrete metric, if $\|f\|_\infty = 1$ then $\|f\|_{L,d} = 2$.
Now we are ready to state a special case of Kantorovich-Rubinstein Theorems (Theorem 11.8.2~\citep{dudley2018real}) which shows the duality between Wasserstein distance and earth mover's distance/total variation distance: 
    For any metric $d$ and two distributions $P, Q$ on $[0,1]^2$, 
    \begin{equation}\label{eq:dual}
        \sup_{f: \|f\|_{L,d} = 1} \E_{\vx\sim P} f(\vx)-\E_{\vy\sim Q}f(\vy) = \inf_{\pi\in \Pi(P, Q)} \E_{(\vx, \vy)\sim \pi}[d(\vx,\vy)].
    \end{equation}
In particular, by taking $d$ as $1$-norm, 
$$d_{EM}(P, Q) = \inf_{\pi\in \Pi(P, Q)} \E_{(\vx, \vy)\sim \pi}[\|\vx-\vy\|_1].$$
By taking $d$ as the discrete metric where $d(x,y) = 1$ for all $x\neq y$, the dual form becomes $d_{TV}(P,Q) = \inf_{\pi\in \Pi(P, Q)} \E_{(\vx, \vy)\sim \pi}[\mathbf{1}[\vx\neq\vy]]$.  

With the above dual form, to prove \cref{lem:dem_covernum} we construct $\theta'\in \ThetabciN$  for each $\theta\in \Thetabci$ and device proper coupling between $\Xdis{\theta}$ and $\Xdis{\theta'}$. Specifically, we let $\rsupp(\theta)$ be close to the $\rsupp(\theta')$ and set the prior $\mu$ of $\theta$ identical to the prior $\mu'$ of $\theta'$.\fang{Is $\theta$ and $\theta'$ swapped?} Recall that $\rsupp(\theta)$ is the support of report space. Then we show that $\theta$ and $\theta'$ are close in earth mover's distance by constructing a coupling $\pi\in \Pi(\Xdis{\theta}, \Xdis{\theta'})$ with small \emph{transportation cost} $\E_{(\vx, \vx')\sim \pi}[\|\vx-\vx'\|_1]$ which upper bounds $d_{EM}(\Xdis{\theta}, \Xdis{\theta'})$ by \cref{eq:dual}.  

We note that $\pi$ can be specified as a function on $\rsupp(\theta)\times \rsupp(\theta')$, and the transportation cost can be written as 
$$\E_{(\vx, \vx')\sim \pi}[\|\vx-\vx'\|_1] = \sum_{\vx\in \rsupp(\theta), \vx'\in \rsupp(\theta')} \|\vx-\vx'\|_1\pi(\vx, \vx').$$
Then, intuitively we bound the cost by considering two cases. Recall two parameterizations of information structures \cref{eq:para1,eq:para2} $(p_0, p_1, q_0, q_1)$ and $(p_0', p_1', q_0', q_1')$ (or $(a_0, a_1, b_0, b_1)$ and $(a_0', a_1', b_0', b_1')$).
When $p_i\approx p_i'$ and $q_i\approx q_i'$ for all $i$, we can transport most of the probability locally $(a_i, b_j)\to (a_i', b_j')$ for all $i, j$ that has small cost $\|(a_i, b_j)-(a_i', b_j')\|_1$.  While if we cannot transport locally, we can show $|a_1-a_0|$ or $|b_1-b_0|$ are small and any coupling will incur little costs. We defer the formal proof to \Cref{prf:dem_covernum}.

\begin{toappendix}
    \label{prf:dem_covernum}
\end{toappendix}
\begin{proof}[Proof for \Cref{lem:dem_covernum}]

Given $\theta' = (\mu', p_0', p_1', q_0', q_1')$ the predictions are 
\begin{align*}
    &{a_0}' = \frac{\mu' p_1'}{\mu' p_1'+(1-\mu')p_0'}, a_1' = \frac{\mu' (1-p_1')}{\mu' (1-p_1')+(1-\mu')(1-p_0')},\\
    &{b_0}' = \frac{\mu' q_1'}{\mu' q_1'+(1-\mu')q_0'}\text{ and }{b_1}' = \frac{\mu' (1-q_1')}{\mu' (1-q_1')+(1-\mu')(1-q_0')}
\end{align*}
Without loss of generality, we can assume $a_0' < a_1'$, $b_0' < b_1'$.
We construct $\theta\in \bar{\Theta}_N$ as the following: Set $\mu = \mu'$.  Then pick $a_0<a_1, b_0<b_1\in \{0,\frac{1}{N},\cdots,1\}$ so that $$a_0'-1/N\le a_0\le a_0', a_1'\le a_1\le a_1'+1/N, b_0'-1/N\le b_0\le b_0'\text{, and }b_1'\le b_1\le b_1'+1/N.$$ 
Because $a_0\le \mu\le a_1$ and $b_0\le \mu\le b_1$, by \Cref{lem:construct_info} there exists $p_0, p_1, q_0$ and $q_1$ so that $\rsupp(\theta) = \{(a_0, b_0), (a_0, b_1), (a_1, b_0), (a_1, b_1)\}$.  By symmetry, We can additionally assume $|a_1-a_0|\le |b_1-b_0|$.

Recall that any coupling $\mu: \rsupp(\theta)\times\rsupp(\theta')\to \mathbb{R}_{\ge 0}$ between $\Xdis{\theta}$ and $\Xdis{\theta'}$, 
 \begin{equation}\label{eq:em_covernum0}
     d_{EM}(P_\theta(x_1,x_2), P_{\theta'}(x_1,x_2))\le \sum_{\vx\in \rsupp(\theta), \vx'\in \rsupp(\theta')} \|\vx-\vx'\|_1\mu(\vx, \vx')
 \end{equation}
 \begin{figure}
     \centering
     \includegraphics[width = .8 \textwidth]{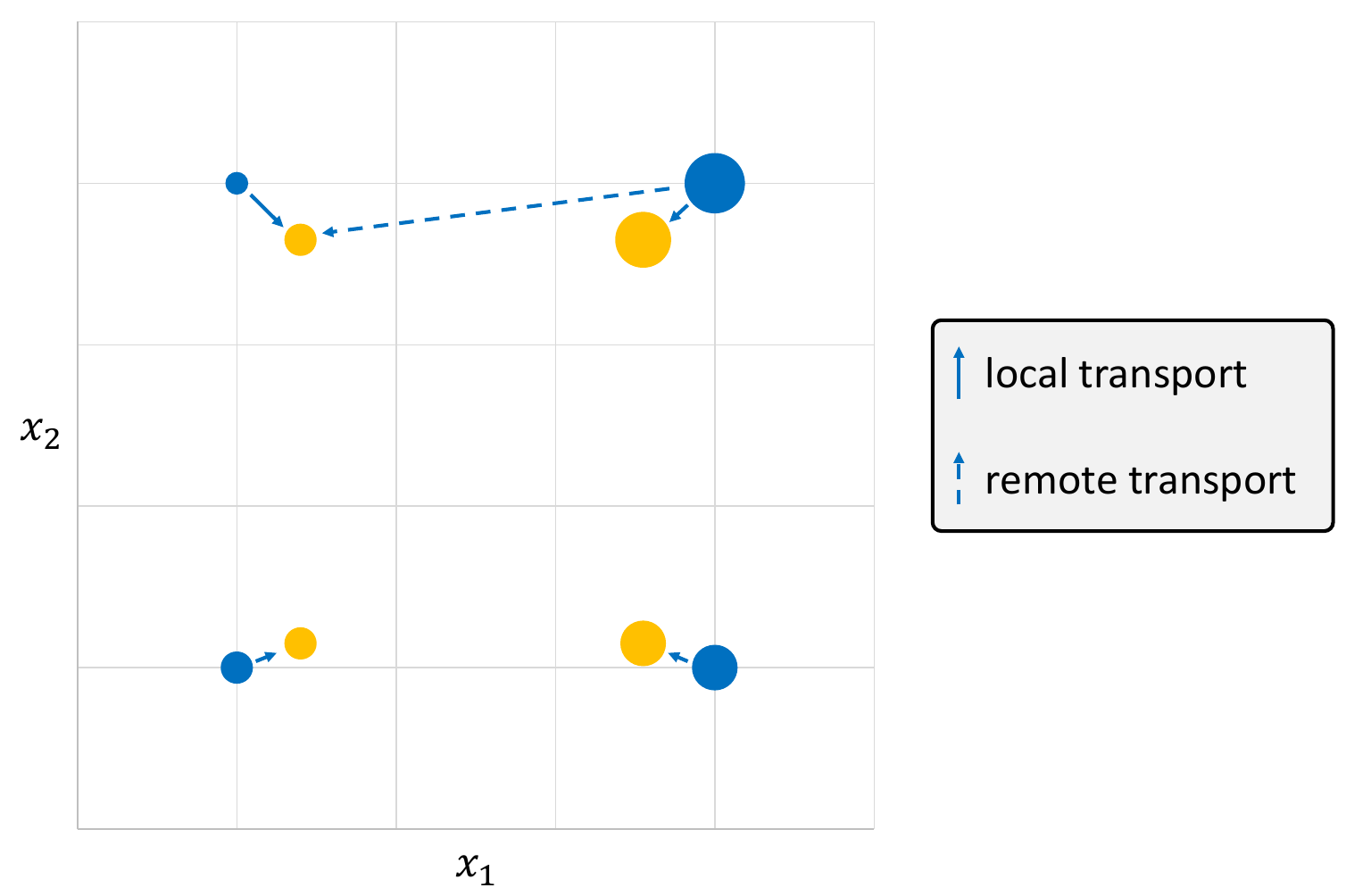}
     \caption{The sketch of the coupling in \Cref{lem:dem_covernum}.  Given $|a_1-a_0|\le |b_1-b_0|$, we first match locally $(a_i, b_j), (a_i', b_j')$ in \cref{eq:em_covernum1}, remotely $(a_i, b_j), (a_{\neg i}', b_j')$ in \cref{eq:em_covernum2}, and the rest.}
     \label{fig:coupling}
 \end{figure}
 The definition of our coupling $\pi(\vx, \vx')$ has three parts.  The first is \emph{local transportation} $\vx = (a_i, b_j), \vx' = (a_i', b_j')$ for all $i,j$ that has the smallest cost $\|\vx-\vx'\|_1 \approx 1/N$.  The second is \emph{remote transportation} $\vx = (a_i, b_j), \vx' = (a_{\neg i}', b_j')$ for all $i,j$ and $\neg i\neq i$ that has cost $\|\vx-\vx'\|_1 \approx |a_0-a_1|$.  The rest has the largest cost $\|\vx-\vx'\|_1 \approx |b_0-b_1|$.  Therefore, to minimize the cost, $\pi$ will prioritize the local transportation, remote transportation, then the rest.  First, we set 
 \begin{equation}\label{eq:em_covernum1}
     \pi((a_i, b_j), (a_i', b_j')) = \inf \{\Pr_{\theta}[a_i, b_j], \Pr_{\theta'}[a_i', b_j']\}\text{ for all }i, j = 0,1.
 \end{equation}  
 Second, for all $i, {\neg i},j = 0,1, i\neq {\neg i}$
 \begin{equation}\label{eq:em_covernum2}
     \pi((a_i, b_j), (a_{\neg i}', b_j')) = \inf\{(\Pr_{\theta'}[a_i', b_j']-\Pr_{\theta}[a_i, b_j])_+, (\Pr_{\theta}[a_{\neg i}, b_j]-\Pr_{\theta'}[a_{\neg i}', b_j'])_+ \}
 \end{equation}
 where $(z)_+ = \sup \{z, 0\}$ for all $z\in \R$. 
Finally, we extend \cref{eq:em_covernum1,eq:em_covernum2} to a valid coupling 
We will call $\pi(\vx, \vx')$ the flow/transportation between $\vx$ and $\vx'$.

To upper bound the right hand side of \cref{eq:em_covernum0}, we set a $\delta_N$ that will be specified later, and consider three cases:  small distance case $\delta_N\ge |b_1-b_0|$, large distance case $|a_1-a_0|\ge \delta_N$, and mixed distance case $|b_1-b_0|\ge \delta_N\ge |a_1-a_0|$. 

\paragraph{Small distance case} Because for all $i,j,k,l = 0,1$, the cost are small $|a_i-a_j'|+|b_k-b_l'|\le |a_1-a_0|+|b_1-b_0|\le 2\delta_N$, 
$$\sum_{\vx\in \rsupp(\theta), \vx'\in \rsupp(\theta')} \|\vx-\vx'\|_1\pi(\vx, \vx')\le \sum_{\vx\in \rsupp(\theta), \vx'\in \rsupp(\theta')} 2\delta_N\pi(\vx, \vx') = 2\delta_N$$

\paragraph{Large distance case} When $|a_1-a_0|\ge \delta_N$, though the cost of between $(a_i,b_j)$ and $(a_k',b_l')$ with $(i,j)\neq (k,l)$ is large, we can show most transportation in $\pi$ happens between $(a_i,b_j)$ and $(a_i',b_j')$.  
Because $\|(a_i, b_j)-(a_i', b_j')\|_1\le 2/N$ and $\|(a_i, b_j)-(a_k', b_l')\|_1\le 2$ for all $i,j,k$, and $l$ with $(i,j)\neq (k,l)$, 

\begin{align*}
    \sum_{\vx\in \rsupp(\theta), \vx'\in \rsupp(\theta')} \|\vx-\vx'\|_1\pi(\vx, \vx')\le& \sum_{i,j = 0,1} \pi((a_i, b_i), (a_i', b_i'))\frac{2}{N}+\sum_{i,j,k,l: (i,j)\neq (k,l)} \pi((a_i, b_i), (a_k', b_l'))\\
    =& \sum_{i,j = 0,1} \pi((a_i, b_i), (a_i', b_i'))\frac{2}{N}+2(1-\sum_{i,j = 0,1} \pi((a_i, b_i), (a_i', b_i'))).
\end{align*}
Additionally, by \cref{eq:em_covernum1}, $$\sum_{i,j = 0,1}\pi((a_i, b_j), (a_i', b_j')) = \sum_{i,j = 0,1}\inf (\Pr_\theta[a_i, b_j], \Pr_{\theta'}[a_i', b_j']) = 1-\frac{1}{2}\sum_{i,j = 0,1}|\Pr_\theta[a_i, b_j]- \Pr_{\theta'}[a_i', b_j']|.$$ So 
\begin{equation}\label{eq:em_covernum4}
    \sum_{\vx\in \rsupp(\theta), \vx'\in \rsupp(\theta')} \|\vx-\vx'\|_1\pi(\vx, \vx')\le \frac{2}{N}+(1-\frac{1}{N})\sum_{i,j = 0,1}|\Pr_\theta[a_i, b_j]- \Pr_{\theta'}[a_i', b_j']|
\end{equation}
Thus, it is sufficient to upper bound $|\Pr_\theta[a_i, b_j]- \Pr_{\theta'}[a_i', b_j']|$.  
Let $F_{00}(\alpha_0, \alpha_1, \beta_0, \beta_1)=\frac{\alpha_0\beta_0(\alpha_1-\mu)(\beta_1-\mu)}{\mu(\alpha_1-\alpha_0)(\beta_1-\beta_0)}$ and $G_{00}(\alpha_0, \alpha_1, \beta_0, \beta_1) = \frac{(1-\alpha_0)(1-\beta_0)(\alpha_1-\mu)(\beta_1-\mu)}{(1-\mu)(\alpha_1-\alpha_0)(\beta_1-\beta_0)}$ for $\alpha_0, \alpha_1, \beta_0, \beta_0\in [0,1]$  Because $\mu = \mu'$ and \Cref{lem:construct_info}, we have 
\begin{align*}
    &\Pr_{\theta}[a_0, b_0]-\Pr_{\theta'}[a_0', b_0'] = F_{00}(a_0, a_1, b_0, b_1)-F_{00}(a_0', a_1', b_0', b_1')+G_{00}(a_0, a_1, b_0, b_1)-G_{00}(a_0', a_1', b_0', b_1')
\end{align*}
By Taylor's approximation, there exist $\alpha_0\in [a_0, a_0']$, $\alpha_1\in [a_1', a_1]$, $\beta_0\in [b_0, b_0']$, and $\beta_1\in [b_1', b_1]$ so that 
\begin{align*}
    &|F_{00}(a_0, a_1, b_0, b_1)-F_{00}(a_0', a_1', b_0', b_1')+G_{00}(a_0, a_1, b_0, b_1)-G_{00}(a_0', a_1', b_0', b_1')| \\
    =& |\nabla (F_{00}+G_{00})(\alpha_0, \alpha_1, \beta_0, \beta_1)\cdot (a_0-a_0', a_1-a_1', b_0-b_0', b_1-b_1')|\\
    \le&\|\nabla (F_{00}+G_{00})(\alpha_0, \alpha_1, \beta_0, \beta_1)\|_1 \sup(|a_0-a_0'|+|a_1-a_1'|+|b_0-b_0'|+|b_1-b_1')|)\\
    \le&\frac{1}{N}\|\nabla (F_{00}+G_{00})(\alpha_0, \alpha_1, \beta_0, \beta_1)\|_1
\end{align*}
Now we bound $\|\nabla (F_{00}+G_{00})(\alpha_0, \alpha_1, \beta_0, \beta_1)\|_1$.  
\begin{align*}
    |\frac{\partial (F_{00}+G_{00})}{\partial \alpha_0}| =& \left|\frac{(\alpha_1-\mu)(\beta_1-\mu)(\alpha_1\beta_0-\mu \alpha_1-\mu\beta_0+\mu)}{\mu(1-\mu)(\alpha_1-\alpha_0)^2(\beta_1-\beta_0)}\right|\\
    =& \left|\frac{(\alpha_1-\mu)(\beta_1-\mu)[(\alpha_1-\mu)(\beta_0-\mu)+\mu(1-\mu)]}{\mu(1-\mu)(\alpha_1-\alpha_0)^2(\beta_1-\beta_0)}\right|\\
    \le& \left|\frac{(\alpha_1-\mu)(\beta_1-\mu)(\alpha_1-\mu)(\beta_0-\mu)}{\mu(1-\mu)(\alpha_1-\alpha_0)^2(\beta_1-\beta_0)}\right|+ \left|\frac{(\alpha_1-\mu)(\beta_1-\mu)\mu(1-\mu)}{\mu(1-\mu)(\alpha_1-\alpha_0)^2(\beta_1-\beta_0)}\right|\\    
    =& \left|\frac{(\alpha_1-\mu)^2(\mu-\beta_0)(\beta_1-\mu)}{\mu(1-\mu)(\alpha_1-\alpha_0)^2(\beta_1-\beta_0)}\right|+ \left|\frac{(\alpha_1-\mu)(\beta_1-\mu)}{(\alpha_1-\alpha_0)^2(\beta_1-\beta_0)}\right|
\end{align*}
We upper bound the above two terms separately.  Because $0\le (\mu-\beta_0)(\beta_1-\mu)\le \mu(\beta_1-\beta_0)$ and $0\le (\alpha_1-\mu)^2\le (1-\mu)(\alpha_1-\alpha_0)$, $\left|\frac{(\alpha_1-\mu)^2(\mu-\beta_0)(\beta_1-\mu)}{\mu(1-\mu)(\alpha_1-\alpha_0)^2(\beta_1-\beta_0)}\right|\le \frac{1}{(\alpha_1-\alpha_0)}$.  Additionally, because $0\le (\alpha_1-\mu)(\beta_1-\mu)\le (\alpha_1-\alpha_0)(\beta_1-\beta_0)$, $\left|\frac{(\alpha_1-\mu)(\beta_1-\mu)}{(\alpha_1-\alpha_0)^2(\beta_1-\beta_0)}\right|\le \frac{1}{(\alpha_1-\alpha_0)}$.  Therefore, 
$$|\frac{\partial (F_{00}+G_{00})}{\partial \alpha_0}|\le \frac{2}{(\alpha_1-\alpha_0)}.$$

Second,
\begin{align*}
    &|\frac{\partial (F_{00}+G_{00})}{\partial \alpha_1}|\\
    =& \left|\frac{(\mu-\alpha_0)(\beta_1-\mu)(\alpha_0\beta_0-\alpha_0\mu-\beta_0\mu+\mu))}{(1-\mu)\mu(\alpha_1-\alpha_0)^2(\beta_1-\beta_0)}\right|\\
    =& \left|\frac{(\mu-\alpha_0)(\beta_1-\mu)((\alpha_0-\mu)(\beta_0-\mu)+\mu(1-\mu))}{(1-\mu)\mu(\alpha_1-\alpha_0)^2(\beta_1-\beta_0)}\right|\\
    \le& \left|\frac{(\mu-\alpha_0)(\beta_1-\mu)(\alpha_0-\mu)(\beta_0-\mu)}{(1-\mu)\mu(\alpha_1-\alpha_0)^2(\beta_1-\beta_0)}\right|+\left|\frac{(\mu-\alpha_0)(\beta_1-\mu)\mu(1-\mu)}{(1-\mu)\mu(\alpha_1-\alpha_0)^2(\beta_1-\beta_0)}\right|\\
    =& \left|\frac{(\mu-\alpha_0)^2(\mu-\beta_0)(\beta_1-\mu)}{(1-\mu)\mu(\alpha_1-\alpha_0)^2(\beta_1-\beta_0)}\right|+\left|\frac{(\mu-\alpha_0)(\beta_1-\mu)}{(\alpha_1-\alpha_0)^2(\beta_1-\beta_0)}\right|
\end{align*}
Because $0\le (\mu-\beta_0)(\beta_1-\mu)\le (1-\mu)(\beta_1-\beta_0)$ and $0\le (\mu-\alpha_0)^2\le \mu(\alpha_1-\alpha_0)$, 
$\left|\frac{(\mu-\alpha_0)^2(\mu-\beta_0)(\beta_1-\mu)}{(1-\mu)\mu(\alpha_1-\alpha_0)^2(\beta_1-\beta_0)}\right|\le \left|\frac{(\mu-\alpha_0)^2}{\mu(\alpha_1-\alpha_0)^2}\right|\le \left|\frac{1}{(\alpha_1-\alpha_0)}\right|$.
Additionally, because $0\le (\mu-\alpha_0)(\beta_1-\mu)\le (\alpha_1-\alpha_0)(\beta_1-\beta_0)$, 
$\left|\frac{(\mu-\alpha_0)(\beta_1-\mu)}{(\alpha_1-\alpha_0)^2(\beta_1-\beta_0)}\right|\le \left|\frac{1}{(\alpha_1-\alpha_0)}\right|$.  
Therefore,
$$|\frac{\partial (F_{00}+G_{00})}{\partial \alpha_1}|\le \frac{2}{\alpha_1-\alpha_0}.$$
Finally, by symmetry, we have
$$|\frac{\partial (F_{00}+G_{00})}{\partial \beta_0}|, |\frac{\partial (F_{00}+G_{00})}{\partial \beta_1}|\le \frac{2}{\beta_1-\beta_0}.$$
Since $\alpha_1-\alpha_0\ge a_1-a_0-2/N\ge \delta_N-2/N$ and $\beta_1-\beta_0\ge \delta_N-2/N$
$$|\Pr_{\theta}[a_0, b_0]-\Pr_{\theta'}[a_0', b_0']|\le \frac{4}{N}\left(\frac{1}{\alpha_1-\alpha_0}+\frac{1}{\beta_1-\beta_0}\right)\le \frac{8}{N\delta_N-2}.$$
Using similar argument, we can have
$$|\Pr_{\theta}[a_i, b_j]-\Pr_{\theta'}[a_i', b_j']|\le \frac{8}{N\delta_N-2}$$
for all $i$ and $j = 0,1$.  Apply above inequality to \cref{eq:em_covernum4} and we have
$$\sum_{\vx\in \rsupp(\theta), \vx'\in \rsupp(\theta')} \|\vx-\vx'\|_1\pi(\vx, \vx')\le \frac{2}{N}+(1-\frac{1}{N})4\cdot\frac{8}{N\delta_N-2}\le \frac{2}{N}+\frac{32}{N\delta_N-2}$$

\textbf{Mixed distance case: } When $|a_1-a_0|\le \delta_N\le |b_1-b_0|$, this can be seen as a mixture of the first and second cases.  The remote transportation incur a small cost for all $i,j$, so we can use the argument similar to the first case's to bound the cost of all remote transportation.  On the other hand, though the rest of transportation has a large cost, we will use an argument similar to the large distance case's to show most of the transportation happens remotely.  Formally, let $S = \{((a_i, b_j), (a_i', b_j')): \forall, i,j = 0,1\}$ and $H = \{((a_i, b_j), (a_{\neg i}', b_j')): \forall, i,j = 0,1\}$ be the set contains all remote transportation.  Because $\|(a_i, b_j)-(a'_{i}, b_j')\|_1\le 2/N$ and $\|(a_i, b_j)-(a'_{\neg i}, b_j')\|_1\le \delta_N+1/N$, 
$$\sum_{(\vx, \vx')\in S\cup H} \|\vx-\vx'\|_1\pi(\vx, \vx')\le (\delta_N+\frac{1}{N})\sum_{(\vx, \vx')\in S\cup H}\pi(\vx, \vx').$$
Additionally, 
\begin{align*}
   &\sum_{(\vx, \vx')\in S\cup H}\pi(\vx, \vx')\\
   =& \sum_{j = 0,1}\sum_{i, k = 0,1}\pi((a_i, b_j), (a_k', b_j')) \\
   =& \sum_{j = 0,1}\pi((a_0, b_j), (a_0', b_j'))+\pi((a_1, b_j), (a_1', b_j'))+\pi((a_0, b_j), (a_1', b_j'))+\pi((a_1, b_j), (a_0', b_j'))\\
   =& \sum_{j = 0,1} \inf \{\Pr_{\theta}[a_0, b_j]+\Pr_{\theta}[a_1, b_j], \Pr_{\theta'}[a_0', b_j']+\Pr_{\theta'}[a_1', b_j']\}\\
   =& \sum_{j = 0,1} \inf \{\Pr_{\theta}[ b_j], \Pr_{\theta'}[ b_j']\}\\
   =& 1-\frac{1}{2}\sum_{j = 0,1} |\{\Pr_{\theta}[ b_j]-\Pr_{\theta'}[ b_j']|
\end{align*}
Thus, using argument similar to \cref{eq:em_covernum4} we have 
$$\sum_{\vx\in \rsupp(\theta), \vx'\in \rsupp(\theta')} \|\vx-\vx'\|_1\pi(\vx, \vx')\le (\delta_N+\frac{1}{N})+\sum_{j = 0,1} |\Pr_{\theta}[ b_j]-\Pr_{\theta'}[ b_j']|.$$
Because $\Pr_{\theta}[ b_0] = \frac{{b_1}-\mu}{{b_1}-{b_0}}$ and $\Pr_{\theta}[ b_1] = \frac{\mu-{b_0}}{b_1-b_0}$, by Taylor's approximation, there exist $\beta_0\in [b_0, b_0']$ and $\beta_1\in [b_1', b_1]$ so that 
$$|\Pr_{\theta}[ b_0]-\Pr_{\theta'}[b_0']| =  \left|\frac{\beta_1-\mu}{(\beta_1-\beta_0)^2}(b_0-b_0')+\frac{\mu-\beta_0}{(\beta_1-\beta_0)^2}(b_1-b_1')\right| \le \frac{2}{N(\beta_1-\beta_0)}$$
and
$$|\Pr_{\theta}[ b_1]-\Pr_{\theta'}[b_1']| = \left|\frac{-(\beta_1-\mu)}{(\beta_1-\beta_0)^2}(b_0-b_0')+\frac{-(\mu-\beta_0)}{(\beta_1-\beta_0)^2}(b_1-b_1')\right|\le \frac{2}{N(\beta_1-\beta_0)}$$
As a result,
$$\sum_{\vx\in \rsupp(\theta), \vx'\in \rsupp(\theta')} \|\vx-\vx'\|_1\pi(\vx, \vx')\le (\delta_N+\frac{1}{N})+\frac{2}{N\delta_N-2}.$$

Therefore,
$$\sum_{\vx\in \rsupp(\theta), \vx'\in \rsupp(\theta')} \|\vx-\vx'\|_1\pi(\vx, \vx')\le \sup(2\delta_N, \frac{2}{N}+\frac{32}{N\delta_N-2}, (\delta_N+\frac{1}{N})+\frac{4}{N\delta_N-2})\le \frac{8}{\sqrt{N}}+\frac{4}{N}$$
and we complete the proof by taking $\delta_N = \frac{4}{\sqrt{N}}+\frac{2}{N}$.
\end{proof}

We now show that computing the optimal $L$-Lipschitz operator is efficient.

\subsection{Calculating Optimal $L$-Lipschitz Aggregator is Efficient}
At each round of the algorithm we need to compute a best response over Lipschitz aggregators $\mathcal{F}_L$ efficiently. Consider a distribution over information structures $\vw_{\ThetabciNM} \in \Delta_{\ThetabciNM}$. Let us denote the size of the support of $\vw_{\ThetabciNM}$ as $| \text{supp}(\vw_{\ThetabciNM})|$. We need to solve the optimization problem 
\begin{equation}
\label{e:lipopt}
     \begin{aligned}
       & \inf_{f} & \quad & \mathbf{E}_{\theta \sim \vw_{\ThetabciNM}}[R(f,\theta)] \\
       & \text{subject to} & & |f(\vx) - f(\mathbf{y})| \leq L\|\vx - \mathbf{y}\|_1 \quad \forall \vx,\mathbf{y} \in [0,1]^{2}.
     \end{aligned}
\end{equation}
in time polynomial in $| \text{supp}(\vw_{\ThetabciNM})|$. We show the following lemma. The proof is deferred to \Cref{prf:lipbestresponsecomp}.

\begin{lemmarep}(Efficient Best Response)
\label{lem:lipbestresponsecomp}
    There exists an algorithm that finds the best $L$-Lipschitz aggregator for distribution over information structures $\vw_{\ThetabciNM}$ in time $O\left(\left(|\text{supp}(\vw_{\ThetabciNM})|N^{4} + N^{8}\right)\log\left(\frac{N}{L}\right)\right)$, when $N>L$, and in time $O(|\text{supp}(\vw_{\ThetabciNM})|)$\fang{Should there be $|\cdot |$?}, when $N < L$, both of which are polynomial in the size of the support of $\vw_{\ThetabciNM}$.
\end{lemmarep}

\begin{toappendix}
    \label{prf:lipbestresponsecomp}
\end{toappendix}
\begin{proof}[Proof of \Cref{lem:lipbestresponsecomp}]
Consider any distribution over $\ThetabciNM$,  represented by $\vw_{\ThetabciNM} \in \Delta_{\ThetabciNM}$. By definition of $\ThetabciNM$, the only reports that an aggregator observes are of the form $\{x_1,x_2: x_1,x_2 \in [1/N]\}$. Thus, to solve the optimization problem, 
    \begin{equation}
     \begin{aligned}
       & \inf_{f} & \quad & \mathbf{E}_{\theta \sim \vw_{\ThetabciNM}}[R(f,\theta)] \\
       & \text{subject to} & & |f(\vx) - f(\mathbf{y})| \leq L\|\vx - \mathbf{y}\|_1 \quad \forall \vx,\mathbf{y} \in [0,1]^{2}.
     \end{aligned}
\end{equation}
in time polynomial in $| \text{supp}(\vw_{\ThetabciNM})|$,
we need an assignment of the variables $\{f\left(x_1,x_2\right) : x_1,x_2\in[1/N]\}$, while maintaining that the function is $L$-Lipschitz. We first consider a discretized optimization problem and show that an interpolation technique gives a solution to a best $L$-Lipschitz aggregator. Consider the optimization problem which only constrains the adjacent points.
\begin{equation}
     \begin{aligned}
       & \inf_{f} & \quad & \mathbf{E}_{\theta \sim \vw_{\ThetabciNM}}[R(f,\theta)] \\
       & \text{subject to} & & |f(\vx) - f(\mathbf{y})| \leq \frac{L}N \quad \forall \vx,\mathbf{y} \in \left(\frac{k_{1}}{N},\frac{k_{2}}{N}\right), k_{1},k_{2} \in \{0,1,\dots,N\},\|\vx - \mathbf{y}\|_1=\frac1N.
     \end{aligned}
\end{equation}

This is a convex optimization problem with $(N+1)^{2}$ variables and $O(N^{2})$ constraints. This can be solved using the Ellipsoid method using the following separation oracle: given an assignment of the variables $\{f\left(\frac{k_{1}}{N},\frac{k_{2}}{N}\right): k_{1},k_{2} \in \{0,1,\dots,N\}\}$, we verify each constraint in a brute force approach, which will take $O(N^{2})$ time. Note that since we are working with aggregators trying to predict a binary state, it follows that each variable can be restricted to be in the range $[0,1]$. Thus the variable space lies in a ball of radius $O(N)$,\fang{2 norm ball or infity norm?} as there are $O(N^{2})$ variables. By the Lipschitz condition, there also exists a ball of radius $\frac{L}{N}$ within the feasible region. By rounding the solution of the ellipsoid to the optimal solution(observe that the coefficients are rational\fang{Why is rational enough.  multiples of $1/N$?})  We can then obtain an optimal solution to \Cref{e:lipopt} by interpolating the values at the grid points. The time taken to solve this optimization problem then is $\Tilde{O}((\text{supp}(\vw_{\ThetabciNM})N^{4} + N^{8})\log\left(\frac{N}{\inf\{1,L\}}\right))$. Consider the following function definition at the non-grid points :
\begin{multline*}
    f(x_{1},x_{2}) := \alpha_{1}\alpha_{2}f\left(\frac{\floor{x_{1}N}}{N},\frac{\floor{x_{2}N}}{N}\right) + \alpha_{1}(1-\alpha_{2})f\left(\frac{\floor{x_{1}N}+1}{N},\frac{\floor{x_{2}N}}{N}\right) \\ + (1-\alpha_{1})\alpha_{2}f\left(\frac{\floor{x_{1}N}}{N},\frac{\floor{x_{2}N}+1}{N}\right) + (1-\alpha_{1})(1-\alpha_{2})f\left(\frac{\floor{x_{1}N}+1}{N},\frac{\floor{x_{2}N}+1}{N}\right), \quad \forall x_{1},x_{2} \in [0,1],
\end{multline*}
where $$\alpha_{1} = \floor{x_{1}N}+1-Nx_{1}, x_{2} = \floor{x_{2}N}+1-Nx_{2}.$$ 

Now we check other points. For any $\vx,\mathbf{y}\in [0,1]^2$, since the non-grid points are interpolated by grid points, we only need to ensure that the grid points satisfy the Lipschitz constraint. Suppose the $\|\vx - \mathbf{y}\|_1=\frac kN$. We denote $\vx$ as $\vx_0$ and $\mathbf{y}$ as $\vx_k$. Then we can find a path $\vx_i,i=1,\cdots,k$, such that $\|\vx_{i-1} - \vx_i\|_1=\frac 1N, i=1,\cdots,k$. Then
\begin{align*}
    |f(\vx)-f(\mathbf{y})|\le &\sum_{i=1}^k |f(\vx_{i-1})-f(\vx_i)|\\
    \le & \sum_{i=1}^k \frac LN\\
    \le & \frac{Lk}{N}= L\|\vx - \mathbf{y}\|_1\\
\end{align*}
Thus we find a optimal Lipschitz aggregator.

In the case $N < L$, then the Lipschitz constraint holds vacuously and the problem becomes an unconstrained optimization problem for which the solution is obtained by taking the FOC with respect to each variable and setting it to $0$, thus giving us an explicit formula for the optimum solution.
\end{proof}

\begin{lemma}(Efficient Utility Computation)
\label{lem:liplosscomp}
    Given any function $f \in \mathcal{F}_L$, the loss vector for $\ThetabciNM$, at each iteration, can be computed in time $O\left(|\text{supp}\left(\vw_{\ThetabciNM}\right)|\right)$.
\end{lemma}

\subsection{Proof of \Cref{thm:lipschitz}}
\label{sec:proof}

\begin{toappendix}
    \begin{claim}\label{clm:compactness2}
        $\mathcal{F}_{L}$ is compact.
    \end{claim}
    \begin{proof}
        We first show that $\left(\mathcal{F}_{L},\lvert|\cdot\rvert|_{\infty}\right)$ is complete.\fang{similar to Theorem 5.1. The completeness is standard and can be deferred to appendix.} Consider a Cauchy sequence $\left(f_{n}\right)$ where $f_{n} \in \mathcal{F}_{L}$. For any $\mathbf{x} \in [0,1]^{2}$ and by the sup-norm, it follows that $(f_{n}(\mathbf{x}))$ is a Cauchy sequence on $\mathbb{R}$ and thus converges in $\mathbb{R}$. Define $f(\mathbf{x}) := \lim_{n \to \infty}f_{n}(\mathbf{x})$. We first show that the function $f(\cdot)$ is $L$-Lipschitz. For any $\mathbf{x} \in [0,1]^{2}$ and $\mathbf{y} \in [0,1]^{2}$, consider the following chain of inequalities
\begin{equation*}
    \begin{split}
        |f(\mathbf{x})-f(\mathbf{y})| &= |\lim_{n \to \infty}f_{n}(\mathbf{x}) - \lim_{n \to \infty}f_{n}(\mathbf{y})|, \\
        &= |\lim_{n \to \infty}\left(f_{n}(\mathbf{x})-f_{n}(\mathbf{y})\right)|, \\
        &= \lim_{n \to \infty}|f_{n}(\mathbf{x}-f_{n}(\mathbf{y})|, \\
        & \leq L\lvert|\mathbf{x}-\mathbf{y}\rvert|.
    \end{split}
\end{equation*}
where the third equality follows from the continuity of the absolute value function and the first inequality follows from the fact that $f_{n}$ is $L$-Lipschitz. To show that $(f_{n})$ converges to $f$, for any $\epsilon > 0$, there exists $N$ such that $\sup_{\mathbf{x} \in [0,1]^2}|f_{n}(\mathbf{x})-f_{m}(\mathbf{x})| \leq \epsilon$, for all $n,m \geq N$. Thus $\sup_{\mathbf{x}}|f_{n}(\mathbf{x})-f(\mathbf{x})| \leq \epsilon$. But this holds for all $\epsilon > 0$ and thus $(f_{n})$ converges to $f$. It follows that $\left(\mathcal{F}_{L},\lvert|\cdot\rvert|_{\infty}\right)$ is complete. From observing that $\left(\mathcal{F}_{L},\lvert|\cdot\rvert|_{\infty}\right)$ is also totally-bounded, we have that the space is compact. 
    \end{proof}
\end{toappendix}

\begin{proof}[Proof of \Cref{thm:lipschitz}]
We claim that $\mathcal{F}_{L}$ is compact (the proof is deferred to \Cref{clm:compactness2}) and follow from Glicksberg's theorem an equilibrium exists when considering the set of information structures $\ThetabciNM$~\citep{glicksberg1952further}. We formally state that $$\inf_{\vw_{\mathcal{F}} \in \Delta_{\mathcal{F}_{L}}} \sup_{\theta \in \ThetabciNM} \mathbf{E}_{f \sim \vw_{\mathcal{F}}}[R(f,\theta)] = \sup_{\vw_{\Theta}\in \Delta_{\ThetabciNM}} \inf_{f \in \mathcal{F}_{L}} \mathbf{E}_{\theta \sim \vw_{\Theta}}[R(f,\theta)].$$ 
It is easy to verify that the Lipschitz setting satisfies the conditions in \Cref{thm:finite}.

Choosing $M=N^{\frac{3}{2}}$, \Cref{l:smallcoverlipschitz} tells us that $\ThetabciNM$ is a $\left(\frac{24}{\sqrt{N}},d_{EM}\right)$-covering of $\Thetabci$. By \Cref{lem:lipbestresponsecomp}, \Cref{lem:liplosscomp} and \Cref{thm:finite}, \Cref{alg:2} can compute an $\frac{\epsilon}{2}$-optimal $L$-Lipschitz aggregator $f^{*}$ in time $O\left(N^{\frac{11}{2}}\log\left(\frac{N}{L}\right)\frac{1}{\epsilon^{2}}\right)$. We thus have that $$R(f^{*},\ThetabciNM) \leq \inf_{f\in\mathcal{F}_L}R(f,\ThetabciNM) + \frac{\epsilon}{2}.$$
\Cref{lem:dem_smooth} and \Cref{l:smallcoverlipschitz} tell us that for any Lipschitz aggregator $f$, we have $$R(f,\ThetabciNM) \leq R(f,\Thetabci) + 134\left(\frac{24}{\sqrt{N}}\right)^{\frac{1}{7}} + 2L\left(\frac{24}{\sqrt{N}}\right).$$
From both of the equations above, and since $\ThetabciNM \subset \Thetabci$, it follows that 
$$R\left(f^{*},\ThetabciNM\right) \leq \inf_{f\in\mathcal{F}_L}R(f,\Thetabci) + \frac{\epsilon}{2} +  134\left(\frac{24}{\sqrt{N}}\right)^{\frac{1}{7}} + 2L\left(\frac{24}{\sqrt{N}}\right) \leq \inf_{f\in\mathcal{F}_L}R(f,\Thetabci) + \epsilon,$$ 
for large enough $N$, precisely for $N$ chosen as $O\left(\frac{L^{14}}{\epsilon^{14}}\right)$. Thus the overall running time is $O\left(\frac{L^{77}}{\epsilon^{79}}\log \left(\frac{L^{13}}{\epsilon^{14}}\right)\right)$.

\end{proof}

\section{Numerical Results}
In this section, we will show the performance of our algorithm and compare different robustness paradigms numerically.



\subsection{Aggregators}


\paragraph{Regret Estimation}

Following \citet{doi:10.1073/pnas.1813934115}, we numerically estimate and compare the regrets of our and other common aggregators.\footnote{The global optimum software of Matlab R2022a.}  Our algorithm obtains an aggregator with $N=20, M=400, L=\infty$\footnote{We use linear interpolation to obtain a continuous aggregator from the discrete output of our algorithm.} whose regret is $0.0226$ that outperforms all previous aggregators: The simple averaging  $\frac{x_1+x_2}{2}$ has regret $0.0625$, the average prior aggregator (see formula in \Cref{table:matlab}), proposed by \citet{doi:10.1073/pnas.1813934115}, has regret $0.0260$, and the previous state-of-the-art aggregator in \citet{doi:10.1073/pnas.1813934115} has $0.0250$.  Moreover, our aggregator's regret $0.0226$ almost match previous theoretical lower-bound $\frac18(5\sqrt{5}-11)\approx 0.0225$~\citep{doi:10.1073/pnas.1813934115}. Notice that the convergence time of our FPTAS depends on the discretization parameters and the Lipschitz constant. To have a reasonable convergence time, we pick relatively small discretization parameters. This is why we only obtain a near-tight aggregator. The results are listed in \Cref{table:matlab}. We discuss the efficiency and implementation details in \Cref{sec:implement}.

\begin{table}[!ht]
\renewcommand{\arraystretch}{2}
  \centering
  
  \begin{tabular}{lcc}
    \toprule
    Aggregator & Formula & Regret\\
    \midrule
    \midrule
    Simple averaging & $\frac{x_1+x_2}2$ & 0.0625\\
    \midrule 
    Average prior  & $\frac{x_1x_2(1-\frac{x_1+x_2}2)}{x_1x_2(1-\frac{x_1+x_2}2)+(1-x_1)(1-x_2)\frac{x_1+x_2}2}$ & 0.0260\\
    \midrule
    State-of-the-art  & $\frac{x_1x_2(1-ep(x_1,x_2))}{x_1x_2(1-ep(x_1,x_2))+(1-x_1)(1-x_2)ep(x_1,x_2)}$ & 0.0250\\
    \midrule
    Our aggregator & - & \pmb{0.0226}\\
    \bottomrule
  \end{tabular}
  \caption{\textbf{Regret of different aggregators.}\\ Here $ep(x_1,x_2)=\begin{cases}0.49x_1+0.49x_2, &if\ x_1+x_2\leq 1\\ 0.49x_1+0.49x_2+0.02, & otherwise\end{cases}$
  }
  \label{table:matlab}
\end{table}

\paragraph{Visual Comparison to Previous Aggregators}

\begin{figure}[!ht]
    \centering
    \begin{subfigure}[b]{0.225\textwidth}
        \centering
        \includegraphics[height=.96\textwidth]{figures/simple_average.pdf}
        \caption{\textbf{Simple averaging}}
    \end{subfigure}
    \begin{subfigure}[b]{0.225\textwidth}
        \centering
        \includegraphics[height=.96\textwidth]{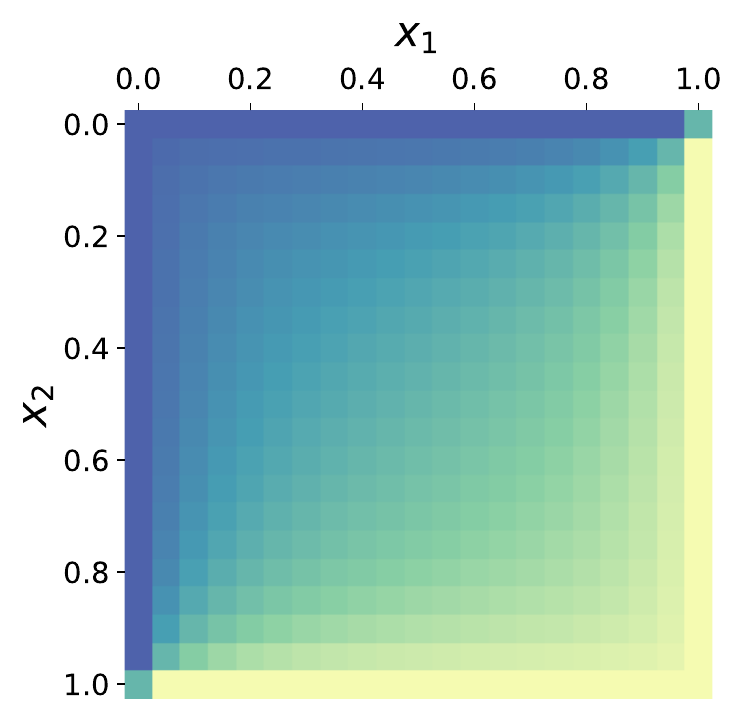}
        \caption{\textbf{Average prior}}
    \end{subfigure}
    \begin{subfigure}[b]{0.225\textwidth}
        \centering
        \includegraphics[height=.96\textwidth]{figures/sota_pnas.pdf}
        \caption{\textbf{State-of-the-art}}
    \end{subfigure}
    \begin{subfigure}[b]{0.27\textwidth}
        \centering
        \includegraphics[height=.8\textwidth]{figures/function.pdf}
        \includegraphics[height=.7\textwidth]{figures/cbar1.pdf}
        \caption{\textbf{Our aggregator}}
    \end{subfigure}
    
    \caption{\textbf{Heatmaps of different aggregators $f(x_1,x_2)$.} Darker to lighter shades represent the range of $f(x_1,x_2)$ from 0 to 1.}
    \label{fig:allaggregators}
\end{figure}

\begin{figure}[!ht]
    \centering
    \begin{subfigure}[b]{0.225\textwidth}
        \centering
        \includegraphics[height=.96\textwidth]{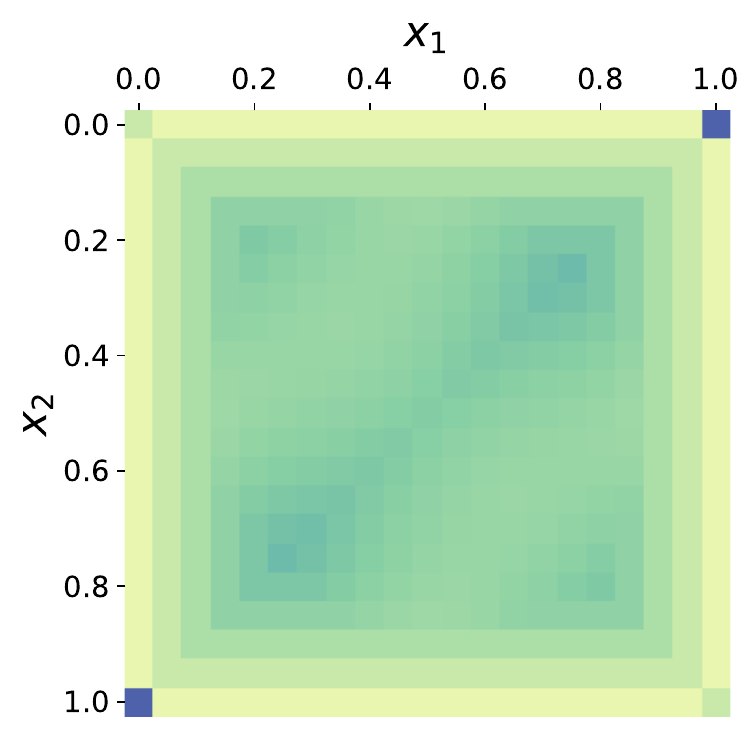}
        \caption{\textbf{Simple averaging}}
        \label{fig:image1}
    \end{subfigure}
    \begin{subfigure}[b]{0.225\textwidth}
        \centering
        \includegraphics[height=.96\textwidth]{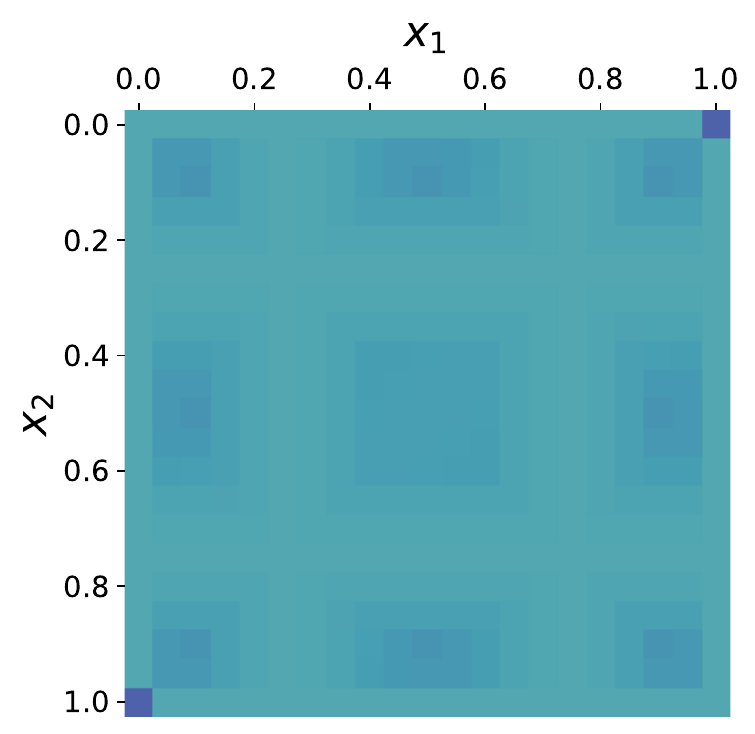}
        \caption{\textbf{Average prior}}
        \label{fig:image2}
    \end{subfigure}
    \begin{subfigure}[b]{0.225\textwidth}
        \centering
        \includegraphics[height=.96\textwidth]{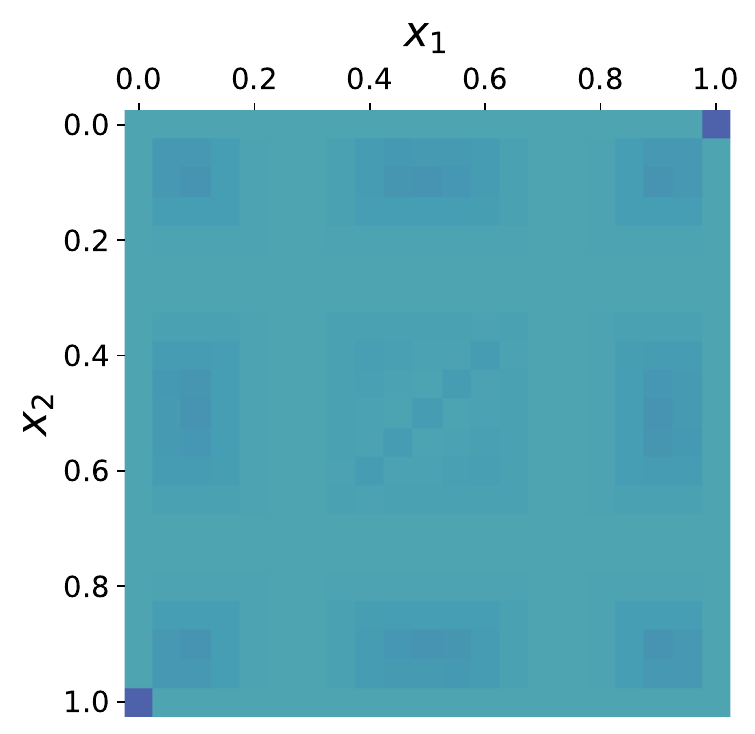}
        \caption{\textbf{State-of-the-art}}
        \label{fig:image3}
    \end{subfigure}
    \begin{subfigure}[b]{0.27\textwidth}
        \centering
        \includegraphics[height=.8\textwidth]{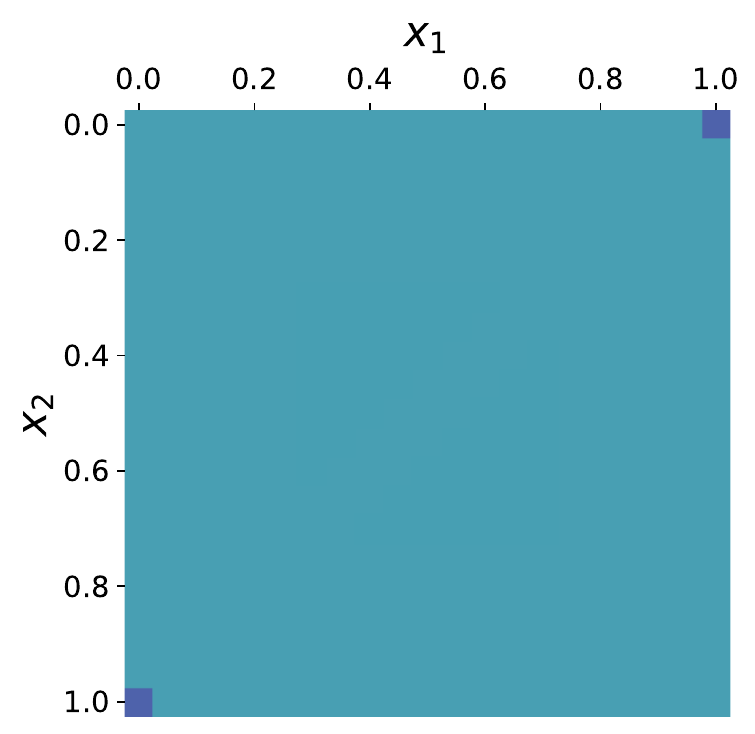}
        \includegraphics[height=.7\textwidth]{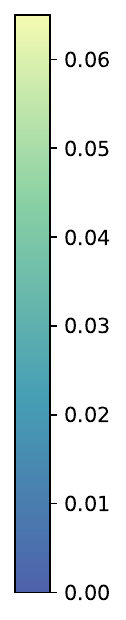}
        \caption{\textbf{Our aggregator}}
        \label{fig:image4}
    \end{subfigure}
 \caption{\textbf{Heatmaps of $R(x_1,x_2)=\max_{\theta,\Pr_{\theta}[x_1,x_2]>0}R(f,\theta)$.} Visualize the regret associated with reports $(x_1,x_2)$ from different aggregators. Darker to lighter shades represent the range of $R(x_1,x_2)$ from $0$ to $0.065$.}
    \label{fig:allregrets}
\end{figure}

\Cref{fig:allaggregators} shows the heatmaps
 of the aforementioned aggregators' values. When one report is almost certain ($\le 0.05$) and the other is not, except the simple averaging, all other aggregators will follow the almost certain report. This makes sense because the agent who provides more extreme forecasts usually has more information.  In particular, if one knows the ground state exactly and reports either $0$ or $1$, a good aggregator should follow the report regardless of the other report. This is one reason why simple averaging performs badly. Though the Average prior, State-of-the-art and our aggregator are similar, our aggregator aggregates forecast to a more extreme value when $x_1,x_2\approx0.1$ or $x_1, x_2\approx0.9$. 

To further compare our aggregator with the others, we aim to visualize the ``weakness'' of different aggregators. Specifically, we calculate the maximal regret over the information structures associated with each report $\vx$,\footnote{Note that it is impossible to have predictions $(0,1)$ and $(1,0)$.  We set the value of $R(x_1,x_2) = 0$ when $(x_1, x_2) = (0,1), (1,0)$.} denoted as $\max_{\theta,\Pr_{\theta}[\vx]>0}R(f,\theta)$ in \Cref{fig:allregrets}. In this context, a report $\vx$ with a high maximal regret indicates a potential ``weakness'' of the aggregator.  We observe that the simple averaging aggregator exhibits substantial regret along the borders, suggesting its vulnerability in those areas, which aligns with our previous discussion.  Similarly, both the average prior aggregator and the State-of-the-art aggregator display high regret at the borders, as well as when either $x_1$ or $x_2$ is around $0.25$ or $0.75$. 
In contrast, the regret associated with reports from our aggregator is almost uniformly distributed across the entire heatmap. This indicates that our aggregator is not vulnerable to specific areas but rather pays attention to all regions. 


\begin{figure}[!ht]
    \centering
    \begin{subfigure}[b]{0.225\textwidth}
        \includegraphics[height=.96\textwidth]{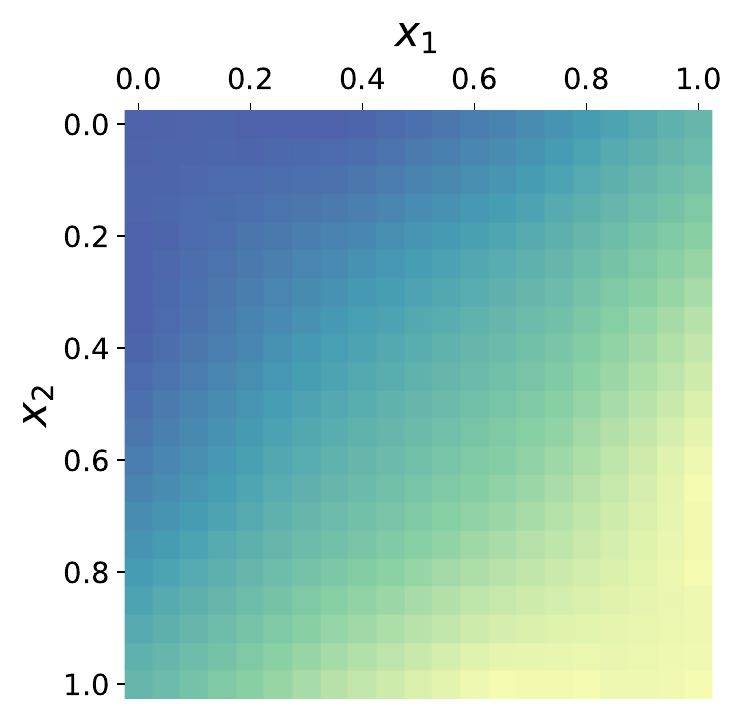}
        \caption{\textbf{\textit{$1$-Lipschitz}}}
        \label{fig:image1}
    \end{subfigure}
    \begin{subfigure}[b]{0.225\textwidth}
        \includegraphics[height=.96\textwidth]{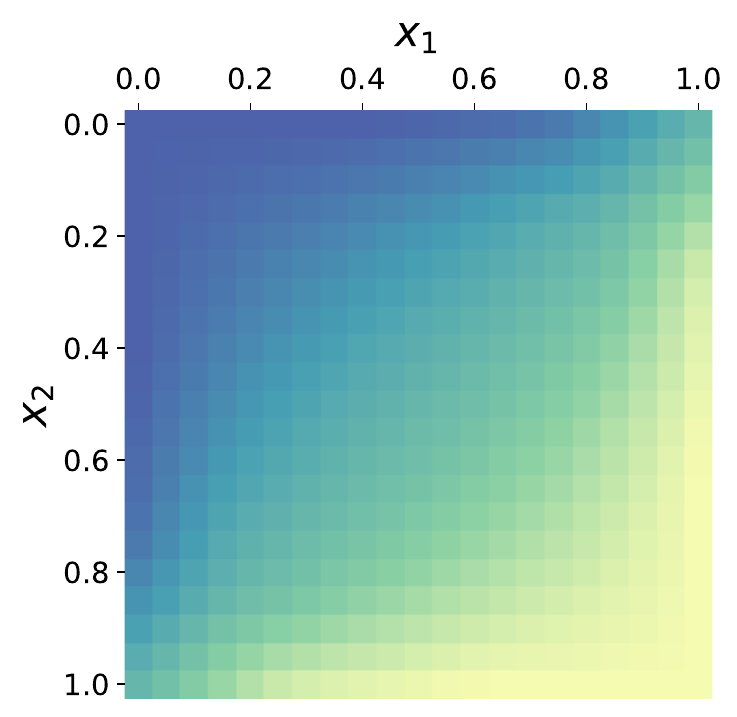}
        \caption{\textbf{\textit{$2$-Lipschitz}}}
        \label{fig:image2}
    \end{subfigure}
    \begin{subfigure}[b]{0.225\textwidth}
        \includegraphics[height=.96\textwidth]{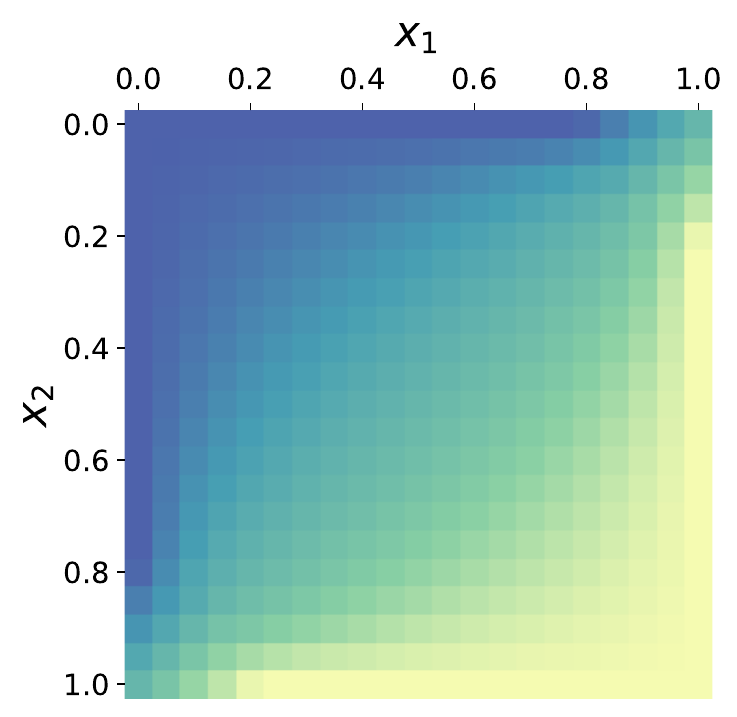}
        \caption{\textbf{\textit{$4$-Lipschitz}}}
        \label{fig:image2}
    \end{subfigure}
    \begin{subfigure}[b]{0.27\textwidth}
        \includegraphics[height=.8\textwidth]{figures/function.pdf}
        \includegraphics[height=.7\textwidth]{figures/cbar1.pdf}
        \caption{\textbf{\textit{$\infty$-Lipschitz}}}
        \label{fig:image2}
    \end{subfigure}
    
    \caption{\textbf{Heatmaps of approximate optimal aggregators $f(x_1,x_2)$ under different Lipschitz constant ($L$).} Resolution parameters are $N=20, M=400$. Darker to lighter shades represent the range of $f(x_1,x_2)$ from 0 to 1.}
    \label{fig:allresults2}
\end{figure}

To study the effect of Lipschitz constant $L$, \Cref{fig:allresults2} shows the results of the online learning algorithm with different $L$. As $L$ increases, the optimal aggregator aggregates the forecasts to a more extreme value in the corner case ($x_1\approx 0,1$ or $x_2\approx 0,1$), while the aggregators are similar in the central parts. On the other hand, the regret of aggregators under small $L$ ($L\le 4$) is relatively high (around $0.024$). When $L$ is around 20, the aggregator becomes near-optimal.

\begin{toappendix}
\section{Efficiency and Implementation Details}
\label{sec:implement}
\paragraph{Pruning by symmetry} When we run the algorithm, to improve efficiency, we prune some information structures which is symmetric with some other information structures. In detail, for each $\theta=(\mu, a_0,a_1,b_0,b_1)$, we delete the centrosymmetric information $(1-\mu,1-a_0,1-a_1,1-b_0,1-b_1)$, the axisymmetric $(\mu,b_0,b_1,a_0,a_1)$ and $(\mu,a_1,a_0,b_0,b_1)$. This is supported by the following lemma. By this method, we improve the efficiency 16-fold.
\begin{lemmarep}[Symmetry of weights and functions]
    If the set of aggregators $\mathcal{F}$ are convex, compact, and symmetric, there exists an equilibrium $(\vw_{\ThetabciNM}, f)$ for $\vw_{\ThetabciNM}\in \Delta_{\ThetabciNM}$ and $f\in \mathcal{F}$ where the weight and aggregator are both symmetric:
    \begin{align*}
    \forall \theta(\mu, a_0,a_1,b_0,b_1)\in\ThetabciNM,&\\ w_{\theta(\mu,a_0,a_1,b_0,b_1)}&=w_{\theta(\mu,a_1,a_0,b_0,b_1)}\\
    &=w_{\theta(\mu,b_0,b_1,a_0,a_1)}\\
    &=w_{\theta(1-\mu,1-a_0,1-a_1,1-b_0,1-b_1)}\\
    \forall x_1,x_2\in\{0,\frac1N,\frac2N,\cdots,1\},&\\
    f(x_1,x_2)&=f(x_2,x_1)\\
        &=1-f(1-x_1,1-x_2)
    \end{align*}
    \label{lem:symmetry}
\end{lemmarep}

\begin{proof}
As we already proved the existence of equilibrium in \cref{sec:proof}, let $V$ be the value of the zero-sum game and strategy profile $(\vw_{\ThetabciNM}^{(1)},f^{(1)})$ be an equilibrium, and we have
\begin{align}
V&=\E_{\theta\sim \vw_{\ThetabciNM}^{(1)}}\left[R(f^{(1)},\theta)\right]\notag\\
\forall \hat{\vw}_{\ThetabciNM}\in\Delta_{\ThetabciNM}, \E_{\theta\sim \hat{\vw}_{\ThetabciNM}}\left[R(f^{(1)}, \theta)\right]&\leq \E_{\theta\sim \vw_{\ThetabciNM}^{(1)}}\left[R(f^{(1)},\theta)\right]\label{eq:anyw}\\
\forall \hat{f}\in\mathcal{F}, \E_{\theta\sim \vw_{\ThetabciNM}^{(1)}}\left[R(\hat{f}, \theta)\right]&\geq \E_{\theta\sim \vw_{\ThetabciNM}^{(1)}}\left[R(f^{(1)},\theta)\right]\label{eq:anyf}
\end{align}

Then we define strategy profiles $(\vw_{\ThetabciNM}^{(2)},f^{(2)}), (\vw_{\ThetabciNM}^{(3)},f^{(3)}), (\vw_{\ThetabciNM}^{(4)},f^{(4)})$
\begin{align*}
w^{(2)}_{\theta(\mu,a_0,a_1,b_0,b_1)}&=w^{(1)}_{\theta(\mu,a_1,a_0,b_1,b_0)}\\
w^{(3)}_{\theta(\mu,a_0,a_1,b_0,b_1)}&=w^{(1)}_{\theta(1-\mu,1-a_0,1-a_1,1-b_0,1-b_1)}\\
w^{(4)}_{\theta(\mu,a_0,a_1,b_0,b_1)}&=w^{(1)}_{\theta(1-\mu,1-a_1,1-a_0,1-b_1,1-b_0)}\\
f^{(2)}(x_1,x_2)&=f^{(1)}(x_2,x_1)\\
f^{(3)}(x_1,x_2)&=f^{(1)}(1-x_1,1-x_2)\\
f^{(4)}(x_1,x_2)&=f^{(1)}(1-x_2,1-x_1)\\
\end{align*}

First, we prove the strategy profiles $(\vw^{(2)}_{\ThetabciNM},f^{(2)}), (\vw^{(3)}_{\ThetabciNM},f^{(3)}), (\vw^{(4)}_{\ThetabciNM},f^{(4)})$ are all equilibrium. 

Consider $(\vw^{(2)}_{\ThetabciNM},f^{(2)})$ as an example. We prove this by contradiction.

\begin{itemize}
\item Suppose that there exists a policy  $\hat{\vw}^{(2)}_{\ThetabciNM}\in\Delta_{\ThetabciNM}$ that 
\(
\E_{\theta\sim \hat{\vw}^{(2)}_{\ThetabciNM}}\left[R(f^{(2)}, \theta)\right]>\E_{\theta\sim \vw^{(2)}_{\ThetabciNM}}\left[R(f^{(2)}, \theta)\right].
\) Then we can construct $\hat{\vw}_{\ThetabciNM}$, i.e., 
\(
\hat{w}_{\theta(\mu,a_0,a_1,b_0,b_1)}=\hat{w}^{(2)}_{\theta(\mu,a_1,a_0,b_1,b_0)},
\)
and we will have
\(
\E_{\theta\sim \hat{\vw}_{\ThetabciNM}}\left[R(f^{(1)}, \theta)\right]>\E_{\theta\sim \vw^{(1)}_{\ThetabciNM}}\left[R(f^{(1)}, \theta)\right]
\)
which contradicts the inequality~(\ref{eq:anyw}).

\item Suppose that there exists a function $\hat{f}^{(2)}\in \mathcal{F}$ that 
\(
\E_{\theta\sim \vw^{(2)}_{\ThetabciNM}}\left[R(\hat{f}^{(2)}, \theta)\right]<\E_{\theta\sim \vw^{(2)}_{\ThetabciNM}}\left[R(f^{(2)}, \theta)\right].
\) Then we can construct $\hat{f}$, i.e., 
\(
\hat{f}(x_1,x_2)=\hat{f}^{(2)}(x_2,x_1),
\)
and we will have
\(
\E_{\theta\sim \vw^{(1)}_{\ThetabciNM}}\left[R(\hat{f}, \theta)\right]<\E_{\theta\sim \vw^{(1)}_{\ThetabciNM}}\left[R(f^{(1)}, \theta)\right]
\)
which contradicts the inequality~(\ref{eq:anyf}).
\end{itemize}

Then we proved $(\vw^{(2)}_{\ThetabciNM},f^{(2)})$ is an equilibrium. Similarly we can prove $(\vw^{(3)}_{\ThetabciNM},f^{(3)}), (\vw^{(4)}_{\ThetabciNM},f^{(4)})$ are both equilibrium.

Second, we prove $\forall i,j\in\{1,2,3,4\}$, the strategy profile $(\vw^{(i)}_{\ThetabciNM},f^{(j)})$ is an equilibrium.

Since $(\vw^{(i)}_{\ThetabciNM},f^{(i)})$ and $(\vw^{(j)}_{\ThetabciNM},f^{(j)})$ are both equilibrium, we have

\begin{align}
\forall \hat{\vw}_{\ThetabciNM}\in\Delta_{\ThetabciNM}, \E_{\theta\sim \hat{\vw}_{\ThetabciNM}}\left[R(f^{(j)}, \theta)\right]&\leq \E_{\theta\sim \vw_{\ThetabciNM}^{(j)}}\left[R(f^{(j)}, \theta)\right]\label{eq:anyfj}\\
\forall \hat{f}\in\mathcal{F}, \E_{\theta\sim \vw_{\ThetabciNM}^{(i)}}\left[R(\hat{f}, \theta)\right]&\geq \E_{\theta\sim \vw_{\ThetabciNM}^{(i)}}\left[R(f^{(i)}, \theta)\right]\label{eq:anyfi}
\end{align}

Then we have
\begin{align*}
\E_{\theta\sim \vw^{(i)}_{\ThetabciNM}}\left[R(f^{(j)}, \theta)\right]&\leq \E_{\theta\sim \vw_{\ThetabciNM}^{(j)}}\left[R(f^{(j)}, \theta)\right]\\
\E_{\theta\sim \vw_{\ThetabciNM}^{(i)}}\left[R(f^{(j)}, \theta)\right]&\geq \E_{\theta\sim \vw_{\ThetabciNM}^{(i)}}\left[R(f^{(i)}, \theta)\right]
\end{align*}

Since the game is a zero-sum game, $\E_{\theta\sim \vw_{\ThetabciNM}^{(j)}}\left[R(f^{(j)}, \theta)\right]=\E_{\theta\sim \vw_{\ThetabciNM}^{(i)}}\left[R(f^{(i)}, \theta)\right]$, and we have 
\begin{align*}
\forall \hat{\vw}_{\ThetabciNM}\in\Delta_{\ThetabciNM}, \E_{\theta\sim \hat{\vw}_{\ThetabciNM}}\left[R(f^{(j)}, \theta)\right]&\leq \E_{\theta\sim \vw_{\ThetabciNM}^{(j)}}\left[R(f^{(j)},\theta)\right]=\E_{\theta\sim \vw_{\ThetabciNM}^{(i)}}\left[R(f^{(j)}, \theta)\right]\tag{inequality~(\ref{eq:anyfj})}\\
\forall \hat{f}\in\mathcal{F}, \E_{\theta\sim \vw_{\ThetabciNM}^{(i)}}\left[R(\hat{f}, \theta)\right]&\geq \E_{\theta\sim \vw_{\ThetabciNM}^{(i)}}\left[R(f^{(i)},\theta)\right]=\E_{\theta\sim \vw_{\ThetabciNM}^{(i)}}\left[R(f^{(j)}, \theta)\right]\tag{inequality~(\ref{eq:anyfi})}
\end{align*}
Then we proved $\forall i,j\in\{1,2,3,4\}$, the strategy profile $(\vw^{(i)}_{\ThetabciNM},f^{(j)})$ is an equilibrium.

Third, let $f=\frac14\sum_{j=1}^4f^{(j)}$ be a symmetric function, then we prove $\forall i\in\{1,2,3,4\}$, the strategy profile $(\vw^{(i)}_{\ThetabciNM},f)$ is an equilibrium.
\begin{enumerate}
    \item Prove $\E_{\theta\sim \vw_{\ThetabciNM}^{(i)}}\left[R(f, \theta)\right]$ equals to the value of the zero-sum game.
    \begin{align}
    \E_{\theta\sim \vw_{\ThetabciNM}^{(i)}}\left[R(f, \theta)\right]&=\E_{\theta\sim \vw_{\ThetabciNM}^{(i)}}\left[R(\frac14\sum_{j=1}^4f^{(j)}, \theta)\right]\notag\\
    &\leq \E_{\theta\sim \vw_{\ThetabciNM}^{(i)}}\left[\frac14\sum_{j=1}^4R(f^{(j)}, \theta)\right]\tag{$R(f,\theta)$ is convex to $f$}\\
    &=\frac14\sum_{j=1}^4\E_{\theta\sim \vw_{\ThetabciNM}^{(i)}}\left[R(f^{(j)}, \theta)\right]\notag\\
    &=\E_{\theta\sim \vw_{\ThetabciNM}^{(1)}}\left[R(f^{(1)}, \theta)\right]\notag\\
    \E_{\theta\sim \vw_{\ThetabciNM}^{(i)}}\left[R(f, \theta)\right]&\geq \E_{\theta\sim \vw_{\ThetabciNM}^{(1)}}\left[R(f^{(1)}, \theta)\right]\tag{inequality~(\ref{eq:anyf})}\\
    \Rightarrow \E_{\theta\sim \vw_{\ThetabciNM}^{(i)}}\left[R(f, \theta)\right]&=\E_{\theta\sim \vw_{\ThetabciNM}^{(1)}}\left[R(f^{(1)}, \theta)\right]=V\label{eq:equalV}
    \end{align}
    \item Prove there is no better $\hat{f}$.
    \begin{align*}
        \forall \hat{f}\in\mathcal{F}, \E_{\theta\sim \vw_{\ThetabciNM}^{(1)}}\left[R(\hat{f}, \theta)\right]&\geq \E_{\theta\sim \vw_{\ThetabciNM}^{(1)}}\left[R(f^{(1)},\theta)\right]=\E_{\theta\sim \vw_{\ThetabciNM}^{(i)}}\left[R(f, \theta)\right]\tag{inequality~(\ref{eq:anyf})}
    \end{align*}
    \item Prove there is no better $\hat{\vw}_{\ThetabciNM}$. 
    \begin{align*}
        \forall\hat{\vw}_{\ThetabciNM}\in\Delta_{\ThetabciNM},&\\
        \E_{\theta\sim \hat{\vw}_{\ThetabciNM}}\left[R(f, \theta)\right]&\leq \E_{\theta\sim \hat{\vw}_{\ThetabciNM}}\left[\frac14\sum_{j=1}^4R(f^{(j)},\theta)\right]\tag{$R(f,\theta)$ is convex to $f$}\\
        &= \frac14\sum_{j=1}^4\E_{\theta\sim \hat{\vw}_{\ThetabciNM}}\left[R(f^{(j)},\theta)\right]\\
        &\leq \frac14\sum_{j=1}^4\E_{\theta\sim \vw^{(j)}_{\ThetabciNM}}\left[R(f^{(j)},\theta)\right]\tag{inequality~(\ref{eq:anyfj})}\\
        &=\E_{\theta\sim \vw_{\ThetabciNM}^{(i)}}\left[R(f, \theta)\right]\tag{inequality~(\ref{eq:equalV})}
    \end{align*}
\end{enumerate}

Then we proved the strategy profile $(\vw^{(i)}_{\ThetabciNM},f)$ is an equilibrium.

Forth we define $\vw^{(5)}_{\ThetabciNM}\ldots \vw^{(16)}_{\ThetabciNM}$
\begin{align*}
w^{(5)}_{\theta(\mu,a_0,a_1,b_0,b_1)}&=w^{(1)}_{\theta(\mu,a_1,a_0,b_0,b_1)}\\
w^{(6)}_{\theta(\mu,a_0,a_1,b_0,b_1)}&=w^{(1)}_{\theta(\mu,a_0,a_1,b_1,b_0)}\\
w^{(7)}_{\theta(\mu,a_0,a_1,b_0,b_1)}&=w^{(1)}_{\theta(1-\mu,1-a_1,1-a_0,1-b_0,1-b_1)}\\
w^{(8)}_{\theta(\mu,a_0,a_1,b_0,b_1)}&=w^{(1)}_{\theta(1-\mu,1-a_0,1-a_1,1-b_1,1-b_0)}\\
w^{(9)}_{\theta(\mu,a_0,a_1,b_0,b_1)}&=w^{(1)}_{\theta(\mu,b_0,b_1,a_0,a_1)}\\
w^{(10)}_{\theta(\mu,a_0,a_1,b_0,b_1)}&=w^{(1)}_{\theta(\mu,b_1,b_0,a_1,a_0)}\\
w^{(11)}_{\theta(\mu,a_0,a_1,b_0,b_1)}&=w^{(1)}_{\theta(1-\mu,1-b_0,1-b_1,1-a_0,1-a_1)}\\
w^{(12)}_{\theta(\mu,a_0,a_1,b_0,b_1)}&=w^{(1)}_{\theta(1-\mu,1-b_1,1-b_0,1-a_1,1-a_0)}\\
w^{(13)}_{\theta(\mu,a_0,a_1,b_0,b_1)}&=w^{(1)}_{\theta(\mu,b_0,b_1,a_1,a_0)}\\
w^{(14)}_{\theta(\mu,a_0,a_1,b_0,b_1)}&=w^{(1)}_{\theta(\mu,b_1,b_0,a_0,a_1)}\\
w^{(15)}_{\theta(\mu,a_0,a_1,b_0,b_1)}&=w^{(1)}_{\theta(1-\mu,1-b_0,1-b_1,1-a_1,1-a_0)}\\
w^{(16)}_{\theta(\mu,a_0,a_1,b_0,b_1)}&=w^{(1)}_{\theta(1-\mu,1-b_1,1-b_0,1-a_0,1-a_1)}\\
\end{align*}

Since $f$ is symmetric, we can prove $\forall{i}\in\{1,2,\ldots,16\}$, strategy profile $(\vw^{(i)}_{\ThetabciNM},f)$ is an equilibrium.

Next let $\vw_{\ThetabciNM}=\frac1{16}\sum_{i=1}^{16}\vw^{(i)}_{\ThetabciNM}$ be a symmetric mixed strategy, then we prove the symmetric strategy profile $(\vw_{\ThetabciNM},f)$ is an equilibrium.

\begin{enumerate}
    \item Prove $\E_{\theta\sim \vw_{\ThetabciNM}}\left[R(f, \theta)\right]$ equals to the value of the zero-sum game.
    \begin{align}
    \E_{\theta\sim \vw_{\ThetabciNM}}\left[R(f, \theta)\right]&=\frac1{16}\sum_{i=1}^{16}\E_{\theta\sim \vw_{\ThetabciNM}^{(i)}}\left[R(f, \theta)\right]\notag\\
    &=\E_{\theta\sim \vw_{\ThetabciNM}^{(1)}}\left[R(f, \theta)\right]=V
    \end{align}
    \item Prove there is no better $\hat{f}$.
    \begin{align*}
        \forall \hat{f}\in\mathcal{F}, \E_{\theta\sim \vw_{\ThetabciNM}}\left[R(\hat{f}, \theta)\right]&=\frac1{16}\sum_{i=1}^{16} \E_{\theta\sim \vw_{\ThetabciNM}^{(i)}}\left[R(\hat{f},\theta)\right]\\
        &\geq \frac1{16}\sum_{i=1}^{16} \E_{\theta\sim \vw_{\ThetabciNM}^{(i)}}\left[R(f,\theta)\right]\\
        &=\E_{\theta\sim \vw_{\ThetabciNM}}\left[R(f,\theta)\right]
    \end{align*}
    \item Prove there is no better $\hat{\vw}_{\ThetabciNM}$. 
    \begin{align*}
        \forall\hat{\vw}_{\ThetabciNM}\in\Delta_{\ThetabciNM},\E_{\theta\sim \hat{\vw}_{\ThetabciNM}}\left[R(f, \theta)\right]&\leq \E_{\theta\sim \vw_{\ThetabciNM}^{(i)}}\left[R(f, \theta)\right]=\E_{\theta\sim \vw_{\ThetabciNM}}\left[R(f, \theta)\right]
    \end{align*}
\end{enumerate}

Thus we proved $(\vw_{\ThetabciNM},f)$ is a symmetric equilibrium.
\end{proof}

To explore the efficiency of our online learning algorithm, we introduce two metrics that provide a lower bound and upper bound of minimax value, and constrain the magnitude of $\epsilon$ in the $\epsilon$-equilibrium.

Let $\overline{\vw}^t_{\ThetabciNM}=\frac1t \sum_{i=1}^t\vw^i_{\ThetabciNM}$ and $\overline{f}^t=\frac1t \sum_{i=1}^t f^i$ denote the averaged strategies of the nature and aggregator over the first $t$ rounds. With these, we can define the lower and upper bounds for each round $t$ as follows:

\begin{itemize}
    \item \textbf{Lower bound} for round $t$ represents the expected regret when the aggregator responds optimally to nature's mixed strategy $\overline{\vw}^t_{\ThetabciNM}$, given by:
    \[
    \inf_{f\in\mathcal{F}}\E_{\theta\sim \overline{\vw}^t_{\ThetabciNM}}[R(f,\theta)]
    \]
    \item \textbf{upper bound} for round $t$ signifies the regret when nature responds optimally to the aggregator's strategy $\overline{f}^t$, expressed as:
    \[
    R(\overline{f}^t,\ThetabciNM)=\sup_{\theta\in \ThetabciNM}R(\overline{f}^t,\theta)
    \]
\end{itemize}

It can be shown that these two bounds encapsulate the minimax value.
\[
\inf_{f\in\mathcal{F}}\E_{\theta\sim \overline{\vw}^t_{\ThetabciNM}}[R(f,\theta)]\leq \inf_{f\in\mathcal{F}}\sup_{\theta\in \ThetabciNM}R(f,\theta) \leq R(\overline{f}^t,\ThetabciNM)
\]

\begin{lemma}[Constrain the $\epsilon$ in $\epsilon$-equilibrium]
    Let $\epsilon^t$ be the difference of upper bound and lower bound of round $t$, i.e., $\epsilon^t=R(\overline{f}^t,\ThetabciNM)-\inf_{f\in\mathcal{F}}\E_{\theta\sim \overline{\vw}^t_{\ThetabciNM}}[R(f,\theta)]$. Then the strategy profile $(\overline{\vw}^t_{\ThetabciNM},\overline{f}^t)$ is an $\epsilon^t$-equilibrium.
    \label{lem:constrain_eps}
\end{lemma}

\begin{proof}[Proof of Lemma~\ref{lem:constrain_eps}]
First we have
\[
\inf_{f\in\mathcal{F}}\E_{\theta\sim \overline{\vw}^t_{\ThetabciNM}}[R(f,\theta)]\leq \E_{\theta\sim \overline{\vw}^t_{\ThetabciNM}}R(\overline{f},\theta) \leq R(\overline{f}^t,\ThetabciNM)
\]

Then we have
\begin{align*}
    \E_{\theta\sim \overline{\vw}^t_{\ThetabciNM}}R(\overline{f},\theta) &\leq R(\overline{f}^t,\ThetabciNM)\\
    &= \inf_{f\in\mathcal{F}}\E_{\theta\sim \overline{\vw}^t_{\ThetabciNM}}[R(f,\theta)]+\epsilon^t\\
    \E_{\theta\sim \overline{\vw}^t_{\ThetabciNM}}R(\overline{f},\theta) &\geq \inf_{f\in\mathcal{F}}\E_{\theta\sim \overline{\vw}^t_{\ThetabciNM}}[R(f,\theta)]\\
    &= R(\overline{f}^t,\ThetabciNM)-\epsilon^t
\end{align*}

Thus we proved the strategy profile $(\overline{\vw}^t_{\ThetabciNM},\overline{f}^t)$ is an $\epsilon^t$-equilibrium.
\end{proof}

When the lower bound equals the upper bound, the algorithm is converged. \Cref{fig:perform} shows the performance of our algorithm. We set the learning rate $\eta=1$ during our experiments, and normalize the loss of online learning by the expected regret $\E_{\theta\sim \vw^t_{\ThetabciNM}}\left[R(f^t,\theta)\right]$. For different resolution $N$ and different Lipschitz constant $L$, the convergence speed is close, which costs around $T=10^3$ rounds. The maximum regret converges quicker, which means finding the optimal aggregator is much easier than finding the optimal mixed-strategy of nature.


\begin{figure}[!ht]
    
    \centering
    \includegraphics[width=.45\textwidth]{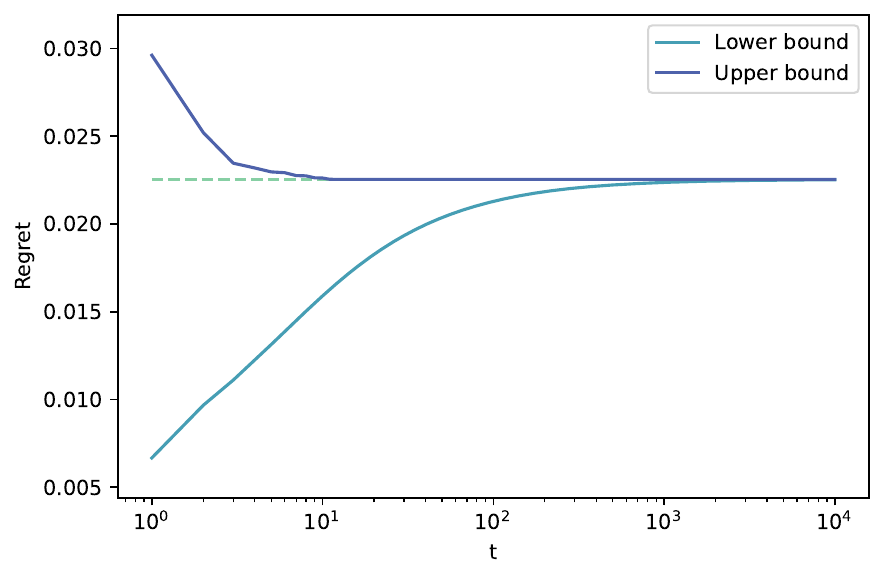}
        
    
    \caption{\textbf{Convergence Rate.} The figure shows the convergence rate of online learning  when $N=20,M=400,L=\infty$. The horizontal axis is the number of rounds and the vertical axis is the regret value.}
    \label{fig:perform}
\end{figure}

\end{toappendix}

\subsection{Different Robustness Paradigms}\label{sec:paradigm}

There are three different robustness paradigms, the additive, the absolute, and the ratio. The regret formulas for these robustness paradigms are listed below:

\begin{itemize}
    \item\textbf{Additive} $Additive(f,\theta)=\E_\theta[\ell(f(\vx),\omega)]-\E_{\theta}[\ell(opt_\theta(\mathbf{s}),\omega)]$
    \item\textbf{Absolute} $Absolute(f,\theta)=\E_\theta[\ell(f(\vx),\omega)]$
    \item\textbf{Ratio} $Ratio(f,\theta)=\frac{\E_\theta[\ell(f(\vx),\omega)]}{\E_\theta[\ell(opt_\theta(\mathbf{s}),\omega)]}$
    
\end{itemize}

The paradigms aim to solve $\inf_{f\in\mathcal{F}}\sup_{\theta\in\Theta} Additive(f,\theta), \inf_{f\in\mathcal{F}}\sup_{\theta\in\Theta} Absolute(f,\theta)$, and\\ $\inf_{f\in\mathcal{F}}\sup_{\theta\in\Theta} Ratio(f,\theta)$ correspondingly. 

Prior-independent mechanism's robustness paradigms are typically scale-invariant~\citep{devanur2011prior, chawla2013prior}. In such cases, it may be more proper to use ratio-based regret. In our setting, since both reports and aggregators are constrained in $[0,1]$, we do not need to pay special attention to the scale. Besides, in some cases, the benchmark can be $0$ or very close to $0$. Therefore the ratio robustness paradigm becomes meaningless in our setting. Traditional machine learning often considers absolute robustness paradigms. However, with the absolute robustness paradigm, nature can always pick uninformative information structures to maximize the loss where no aggregator can help. Thus we select the additive robustness paradigm. We provide a visual comparison of different robustness paradigms and show that the additive robustness paradigm works best in our setting. The visual comparison can also be applied to other settings of prior-independent design.

\begin{figure}[!ht]
    \centering
    \begin{subfigure}[b]{0.32\textwidth}
        \includegraphics[width=\textwidth]{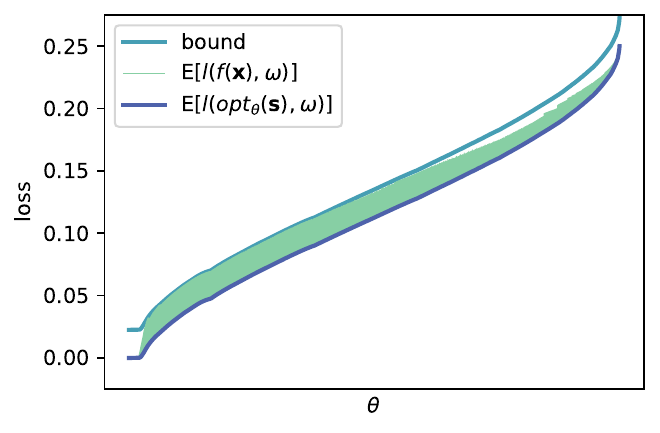}
        \caption{\textbf{Additive}}
        \label{fig:image2}
    \end{subfigure}
    \begin{subfigure}[b]{0.32\textwidth}
        \includegraphics[width=\textwidth]{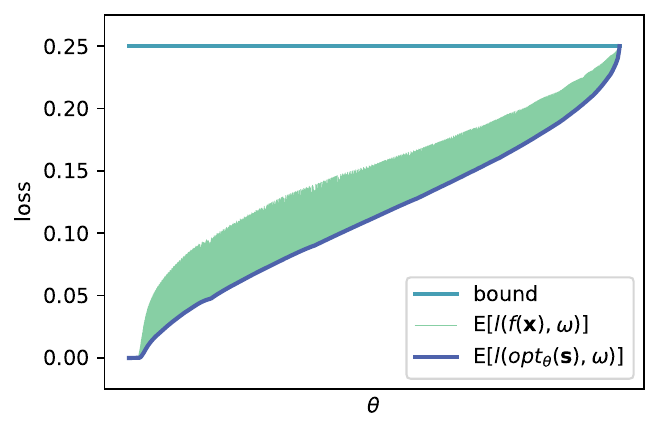}
        \caption{\textbf{Absolute}}
        \label{fig:image1}
    \end{subfigure}
    \begin{subfigure}[b]{0.32\textwidth}
        \includegraphics[width=\textwidth]{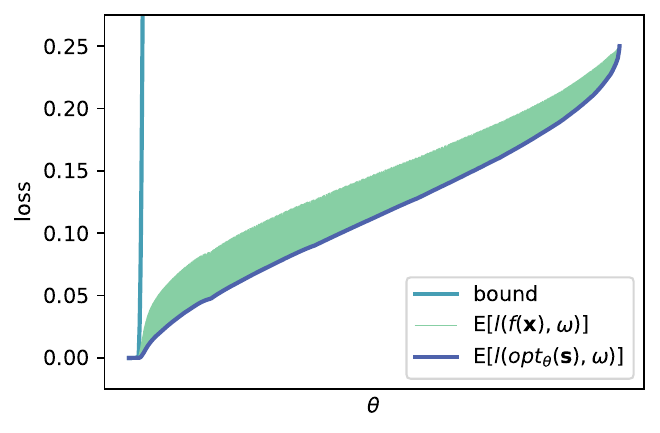}
        \caption{\textbf{Ratio}}
        \label{fig:image2}
    \end{subfigure}
    \caption{\textbf{Losses of the optimal aggregator under additive, ratio, and absolute robustness paradigms.} The horizontal axis represents the information structures in $\ThetabciNM$ sorted by their losses under the omniscient aggregator. The vertical axis represents the loss. The bottom curve  (navy blue) represents the optimal loss (lower bound), i.e., the loss of the omniscient aggregator. The middle curve (green), which consists of a shaded region, is the loss of the optimal aggregator $f$ obtained by our algorithm for each paradigm. The top curve (cyan) represents the bound on the highest loss that can be afforded for each information structure. The worst case occurs when the top curve touches the middle curve.\fang{Is this identical to Figure 2?}}
    \label{fig:allimages}
\end{figure}

\begin{figure}[!ht]
    \centering
    \begin{subfigure}[b]{0.32\textwidth}
        \includegraphics[width=\textwidth]{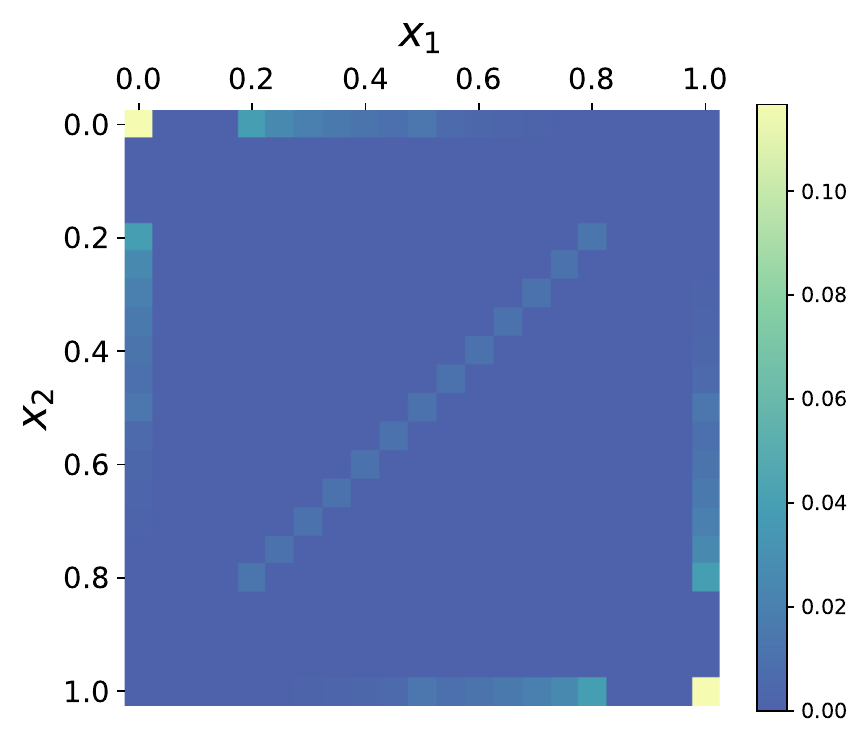}
        \caption{\textbf{Additive}}
        \label{fig:image2}
    \end{subfigure}
    \begin{subfigure}[b]{0.32\textwidth}
        \includegraphics[width=\textwidth]{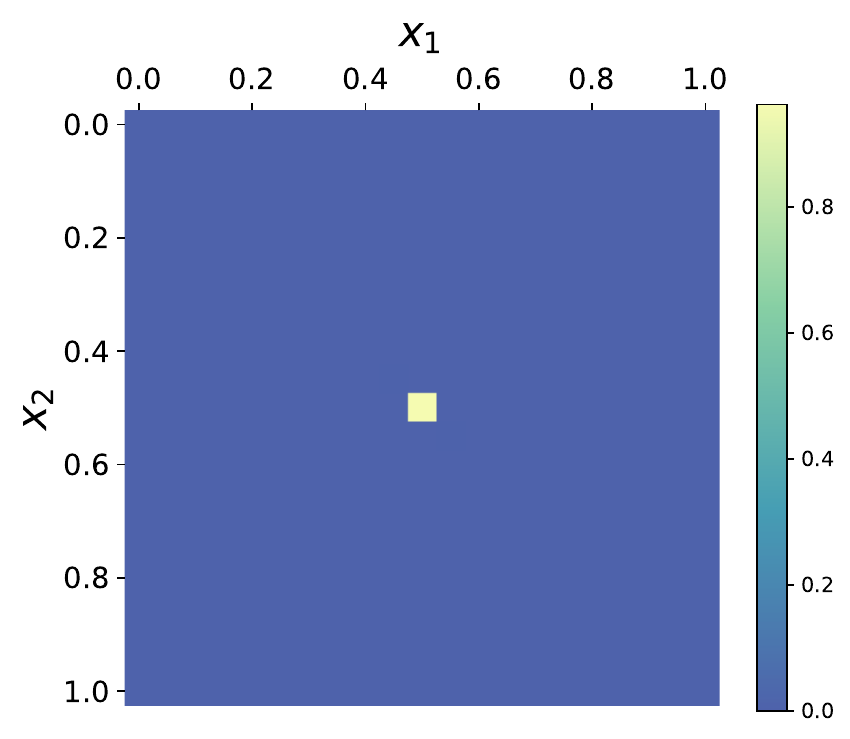}
        \caption{\textbf{Absolute}}
        \label{fig:image1}
    \end{subfigure}
    \begin{subfigure}[b]{0.32\textwidth}
        \includegraphics[width=\textwidth]{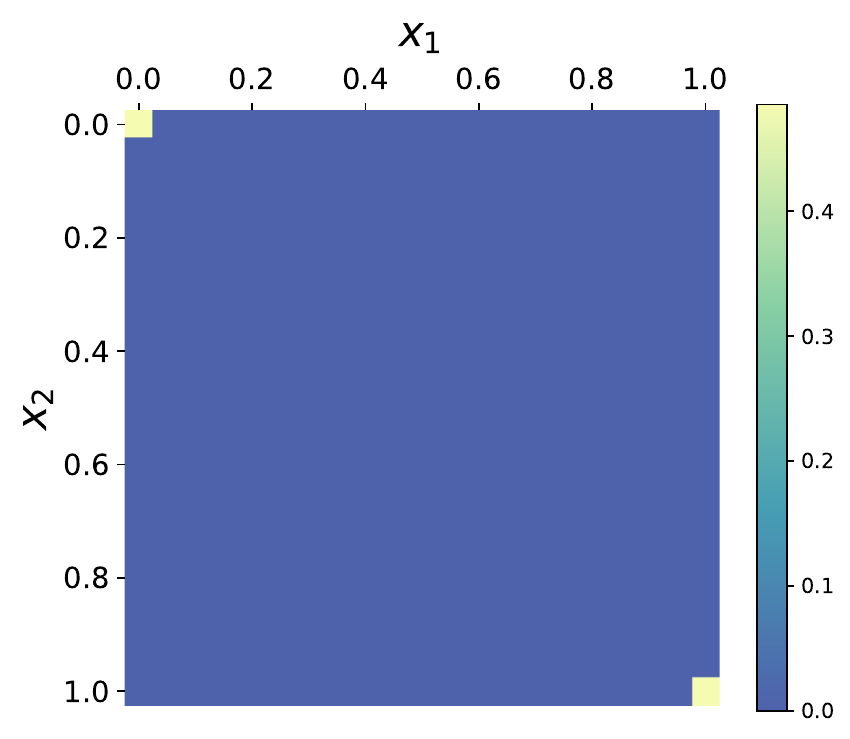}
        \caption{\textbf{Ratio}}
        \label{fig:image2}
    \end{subfigure}
    
    \caption{\textbf{Heatmaps of the distribution over $(x_1,x_2)$ at the $\epsilon$-equilibrium under different robustness paradigms} Darker to lighter shades represent the probabilities increasing, and we use a logarithmic scale.}
    \label{fig:weight}
\end{figure}
\Cref{fig:allimages} shows the learned aggregators under different robustness paradigms. At the equilibrium with the ratio-based robustness paradigm, nature will focus on the information structures where the omniscient aggregator has $\approx 0$ loss\footnote{Weights on the other region are negligible but not zero. This is why the learned aggregator does not have a high loss in the other region.}. At the equilibrium with the absolute robustness paradigm, nature will focus on the information structures where agents have no information. At the equilibrium with the additive robustness paradigm, nature will focus on a more diverse region of the information structures, which leads to an aggregator which behaves quite well everywhere. Thus, the additive robustness paradigm is the most suitable robustness paradigm in our setting. 

In addition to information structures, we also show that with additive regret, a wider range of reports is paid attention to, whereas with absolute loss and ratio-based regret, the focus is narrowed down to one or two specific reports at the $\epsilon$-equilibrium. \Cref{fig:weight} shows the probability of seeing report $\vx$ at the $\epsilon$-equilibrium. Formally, the probability is $\Pr[\vx]=\sum_{\theta\in\ThetabciNM}\Pr_{\theta}[\vx]\cdot\overline{w}_{\theta}$, where $\overline{\vw}$ is the mixed strategy of nature at the $\epsilon$-equilibrium. The larger the probability, the more important the aggregator's response at this point is. Notice in the additive case, we pay more attention to the borders and secondary diagonals. In the absolute case, we only pay attention to the central point $(0.5,0.5)$. In the ratio case, we only pay attention to the upper left $(0,0)$ and lower right $(1,1)$. This matches the fact that with the absolute loss, nature will focus on the uninformative information structure, and with the ratio-based regret, near certain information structures where the omniscient aggregator has near zero loss. 

\section{Conclusion and Discussion}
Our algorithmic framework for robust aggregation addresses the challenges of prior-independent optimal aggregator design by providing a systematic approach. There are several future directions for further exploration. Firstly, although the continuous results focus on the case of two agents with independent signals conditioned on a binary state, extending the framework to multiple symmetric agents would be an interesting avenue for future research. Secondly, our framework can be extended to consider more general information structures, allowing for a richer representation of the forecasters' knowledge. Exploring the impact of higher-order reports and considering the ensemble learning scenario, such as aggregating classifiers, would broaden the scope of our framework and enable its application in diverse domains. Furthermore, incorporating the existence of irrational agents, who may exhibit biases or deviate from rational behavior, would enhance the robustness of the framework in real-world settings. To validate the effectiveness of our framework in real-world scenarios, conducting extensive experiments using real-world data would be also important.

In summary, our algorithmic framework for robust aggregation opens up several avenues for future research. By exploring the extensions and applications discussed, we can further enhance the robustness, scalability, and real-world applicability of the framework, ultimately advancing the field of information aggregation and decision-making.

\section{Acknowledgement}
This research was undertaken in part while the authors were participating in the 2022 IDEAL Special Quarter on Data Economics which was supported by NSF grant CCF 1934931. Yongkang Guo, Zhihuan Huang, and Yuqing Kong were funded by National Natural Science Foundation of China award number 62002001.

\bibliography{reference}
\bibliographystyle{apalike}
\clearpage

\appendix
\end{document}